\documentclass[11pt]{article}

\usepackage[margin=1in]{geometry}
\usepackage{amsmath,amssymb,amsfonts,amsthm,mathtools}
\usepackage{enumitem}
\usepackage{xcolor}
\usepackage{graphicx}
\usepackage{float}
\usepackage{hyperref}
\usepackage[round]{natbib}
\allowdisplaybreaks

\hypersetup{
  pdftitle={Bellman-sufficient Information Complexity},
  pdfauthor={Yunbei Xu},
  colorlinks=true,
  linkcolor=blue!70!white,
  citecolor=blue!70!white,
  urlcolor=blue!70!white
}

\newtheorem{theorem}{Theorem}[section]
\newtheorem{proposition}[theorem]{Proposition}
\newtheorem{lemma}[theorem]{Lemma}
\newtheorem{corollary}[theorem]{Corollary}
\theoremstyle{definition}
\newtheorem{definition}[theorem]{Definition}
\newtheorem{assumption}[theorem]{Assumption}
\newtheorem{remark}[theorem]{Remark}
\newtheorem{example}[theorem]{Example}

\newcommand{\E}{\mathbb E}
\newcommand{\Pp}{\mathbb P}

\newcommand{\R}{\mathbb R}

\newcommand{\1}{\mathbf 1}
\newcommand{\calA}{\mathcal A}

\newcommand{\calF}{\mathcal F}
\newcommand{\calH}{\mathcal H}

\newcommand{\calX}{\mathcal X}
\newcommand{\calY}{\mathcal Y}

\newcommand{\calI}{\mathcal I}
\newcommand{\calE}{\mathcal E}

\newcommand{\KL}{D_{\mathrm{KL}}}

\newcommand{\kl}{\operatorname{kl}}
\newcommand{\Reg}{\operatorname{Reg}}
\newcommand{\Alg}{\mathsf{Alg}}
\newcommand{\AIR}{\mathrm{AIR}}
\newcommand{\MAIR}{\mathrm{MAIR}}

\newcommand{\iid}{\mathrm{i.i.d.}}

\newcommand{\unif}{\operatorname{Unif}}

\newcommand{\polylog}{\operatorname{polylog}}
\DeclareMathOperator*{\argmax}{arg\,max}
\DeclareMathOperator*{\argmin}{arg\,min}

\newcommand{\proofinappendix}[1]{\par\noindent\emph{Proof.} See Appendix~\ref{#1}.\par}

\title{Bellman-sufficient Information Complexity}
\author{Yunbei Xu\\National University of Singapore\\\texttt{yunbei@nus.edu.sg}}
\date{}

\begin{document}
\maketitle

\begin{abstract}

We introduce Bellman-sufficient information complexity for minimax analysis of
sequential decision problems. A Bellman-sufficient state retains the history
needed for the controlled recursion; an index $Y=\chi(\Omega)$ specifies
which information the bounds charge. A log-penalized Bellman program gives
the upper bound, and a Bellman--Fano comparison with an algorithm-dependent
reference trajectory gives the lower bound. When admissibility, calibration,
and growth conditions hold and the bounds match at a common localization
scale, they yield an information-risk sandwich. UCB bounds the immediate
decision loss through calibration and optimism, while AMS/EBO gives a
robust relaxation of the AIR Bellman program.

Our main application gives a negative answer to a widely studied form of the
GP--UCB minimax-optimality question. For every $0<\alpha<1/4$, we construct a
fixed bounded continuous kernel with minimax regret $\Theta(T^{1-\alpha})$
along an infinite sequence of horizons. At those horizons, two globally
calibrated GP--UCB rules incur linear regret under one fixed truth. A policy
that applies action-index AMS to a predetermined finite set in each epoch
attains the minimax order; EBO gives a dual interpretation. Many low-value
directions inflate the global exploration multiplier even though ignoring
them is inexpensive. Through the canonical RKHS feature map, the same
construction gives a finite-horizon polynomial minimax gap for the specified
maximal-information-calibrated LinUCB rule. A numerical experiment
illustrates the mechanism.

\end{abstract}

\setcounter{tocdepth}{2}
\tableofcontents

\section{Introduction}
\label{sec:introduction}

Information-theoretic minimax theory asks how much risk remains when an
experiment can reveal only limited information. In classical noninteractive
estimation, the experiment is fixed before any data are observed. Sharp
rates can be obtained by matching an information-risk upper bound with an
entropy lower bound: KL divergence measures how well the experiment can
distinguish alternatives, while local prior mass or packing entropy measures
how many alternatives must be distinguished
\citep{zhang2006information,yang1999information,yu1997assouad}. In sequential
decision making, the learner's actions change the experiment. The policy
chooses each action based on the past, so decisions, observations, and the
information they reveal are all part of the same controlled process.

Information-theoretic upper and lower bounds provide complementary ways to
study this adaptivity. The algorithmic information ratio (AIR) of
\citet{xu2025bayesian} turns a mutual-information calculation into a
frequentist guarantee for an algorithm, building on earlier
information-ratio and mirror-descent approaches
\citep{russo2014learning,lattimore2021mirror}. For lower bounds, the
interactive Fano method of \citet{chen2024unified} compares the real history
with an algorithm-dependent reference, or ``ghost,'' history rather than with
a fixed nonadaptive sampling scheme. This approach is closely connected to
the decision-estimation coefficient (DEC) of
\citet{foster2021statistical} and its constrained or localized refinements
\citep{foster2023tight}. Our comparison retains the reference process over
all \(T\) rounds, so the bound can track changes in the posterior,
confidence set, or reference state. This complements analyses based on
complexity measures that summarize the adaptive experiment without
explicitly tracking the algorithm's full trajectory.

The algorithmic information ratio and the interactive Fano method belong to a broader tradition of pointwise and
algorithm-dependent analysis in statistics, PAC-Bayesian learning, and
information theory
\citep{audibert2007combining,cai2011testing,tishby2000information,zhang2006information,lugosi2023online}.
Related pointwise, algorithmic, and field-theoretic approaches have been developed for
deep neural networks and Gaussian fields, with connections to computational
mathematics and statistical physics
\citep{li-xu-pointwise-generalization-dnn,xu2026separation}.

Kernel bandits provide the main application.  GP--UCB chooses an action by
adding a confidence bonus to a Gaussian-process posterior mean \citep{srinivas2010gaussian}.  Whether the
gap between its known guarantees and kernel-bandit minimax rates reflects the
analysis or the acquisition rule has remained a prominent open question
\citep{vakili2021open,whitehouse2023sublinear,wang2026suboptimality}.  We show
that maximal information gain can alter the learner's trajectory: many
individually negligible directions enlarge the global confidence multiplier
and draw GP--UCB into a low-value cloud.

The construction starts from the weighted-basis hub--cloud geometry of
\citet{xu2026separation} and reverses its amplitude ordering, so that ignoring
the cloud costs no more than the sublinear minimax order, while the number of
cloud actions still drives the exploration multiplier. Under the canonical
RKHS feature map, GP posterior means and standard deviations coincide with
ridge predictions and ellipsoidal widths. Thus each finite block also gives
a LinUCB instance; gluing the blocks gives a single kernel for the GP--UCB
separation. Section~\ref{sec:applications} distinguishes these results from
the Mat\'ern bounds and states what remains open.

The analysis uses an indexed Bellman-sufficient representation. For each
fixed environment, the state $S_t$ retains what the Bellman recursion needs:
feasible actions, predictive laws, losses, and the state update. An index
$Y=\chi(\Omega)$ identifies the information worth learning---for example, an
optimal action or policy, an active finite marginal, or the full environment.

The coordinate potential $\gamma\log(1/q_t(y^\star))$ makes the fixed-truth
AIR bracket telescope and yields a Bellman upper program. UCB uses calibrated
optimism \citep{srinivas2010gaussian,chu2011contextual}, E2D balances decision
loss against estimation error \citep{foster2021statistical}, and AMS uses
robust optimization over joint environment--index beliefs, with an
equivalent EBO formulation under KL duality
\citep{lattimore2021mirror,xu2025bayesian}. These are different algorithms for which we give a common interpretation by relaxing the dynamic program.

On the lower side, Bellman--Fano compares the true trajectory with an
algorithm-dependent reference trajectory and counts reference histories that
are already good for a randomly sampled index.  Subject to the stated
admissibility, calibration, and growth conditions, the upper and lower values
match when the coordinate code length and the entropy of the ghost good
event have the same order at a common localization scale. This is the
sequential counterpart of matching information to local entropy in a fixed
experiment: the learner now controls how information accumulates.

The framework also covers some episodic and continuing reinforcement
learning formulations (see Section~1.2 of \citet{littman1996algorithms}).
We discuss this extension at the level of the model.

\subsection{Main contributions}
\begin{enumerate}[leftmargin=2em]
\item A controlled-process model of sequential decision making, where indexed Bellman-sufficient states separate what the recursion needs
from what the information bound charges.

\item An indexed chain rule and a fixed-truth AIR identity yield a
log-penalized Bellman upper program. The optimal decision and the
environment give the AIR and MAIR indices, respectively.

\item An algorithm-uniform version of the interactive Fano method gives the lower bound. Combining it with the upper bound yields a succinct proof of the conditional information-complexity sandwich.

\item The fixed-kernel separation compares two globally calibrated
GP--UCB rules with a rate-optimal AIR/AMS policy. At each sufficiently large
horizon \(T\), the finite construction gives \(T+1\) actions in
\(\mathbb R^T\) on which the specified LinUCB rule has a regret ratio of
\(\Omega(T^\alpha)\) relative to unrestricted minimax regret. Gaussian
multi-armed bandits and hypercube linear bandits give further matching
examples.
\end{enumerate}

\subsection{Organization}
Section~\ref{sec:indexed} defines the model and
Section~\ref{sec:identity} proves the indexed information identities.
Sections~\ref{sec:upper} and~\ref{sec:lower} develop the upper and lower
bounds; Section~\ref{sec:info-sandwich} compares them at a common entropy
scale. Section~\ref{sec:applications} presents the kernel-bandit results,
the LinUCB consequence, and the simulation. The appendices contain the
technical comparisons, further examples, and proofs.

\section{Sequential decision making with Bellman-sufficient representations}
\label{sec:indexed}

\subsection{Sequential decision making and Bellman state compression}\label{sec:bellman-motivation}

We use a history-based controlled-process model in which actions, observations, and losses may depend on the full history. Such models are standard in sequential decision making and reinforcement learning \citep{barto1989learning,littman1996algorithms,sutton2018reinforcement,lattimore2020bandit} and provide a natural setting for studying continuing reinforcement learning, planning, and reasoning from experience \citep{abel2023definition,silver2025era}. We develop both our upper and lower bounds in this setting.
\paragraph{Sequential model.}
We evaluate performance under each fixed environment
\(\omega^\star\in\Omega\). A prior \(\mu\) enters only for a Bayes value, a
Yao lower bound, a posterior-reference trajectory, or an algorithmic belief.
Let \(O_0\) be the initial observation packet, if any. The full history
before the round-\(t\) action is \(H_{t-1}\), with
\[
    H_0=O_0,\qquad H_t=(H_{t-1},A_t,O_t).
\]
The packet \(O_t\) collects what is publicly revealed after the round-\(t\)
action and before the next decision: feedback, the next context or physical
state, the next stage marker, public randomization, or other protocol
information. The policy, including its randomization, is specified by the
decision kernels rather than included in the history.

After observing \(h\in\calH_{t-1}\), the learner chooses from the measurable
set \(\calA_t(h)\). A nonanticipating algorithm is a sequence of kernels
\[
    \pi_t^\Alg(\cdot\mid h)\in\Delta(\calA_t(h)),
    \qquad h\in\calH_{t-1},
\]
with deterministic algorithms included as degenerate kernels.  Let \(\mathfrak A_T^{\rm na}\) denote the class of all such algorithms, and write
\[
    p_t=\pi_t^\Alg(\cdot\mid H_{t-1}),\qquad A_t\sim p_t .
\]
The environment or an adaptive adversary may respond to the announced
mixed action \(p_t\), in addition to the realized action \(A_t\):
\begin{align}\label{eq:general}
    O_t\sim P_{\omega^\star,t}(\cdot\mid H_{t-1},p_t,A_t),
    \qquad
    \ell_t(\omega^\star,H_{t-1},p_t,A_t).
\end{align}
When the observation law and loss do not depend on \(p_t\), as in ordinary
stochastic bandits and Markov decision processes, we omit that argument.
Loss increments may be signed; the lower-bound theorems specify the
integrability and lower-bound conditions on the cumulative loss. In
episodic reinforcement learning, \(t\) may index whole episodes, with an
episode policy as the action, or individual stages, indexed by the episode
and the stage within it.

The frequentist objective studied in this paper is
\[
    \mathfrak R_T^\star(\Omega)
    :=
    \inf_{\Alg\in\mathfrak A_T^{\rm na}}
    \sup_{\omega\in\Omega}
    \E_\omega^\Alg
    \sum_{t=1}^T\ell_t(\omega,H_{t-1},p_t,A_t).
\]
This explicit stepwise model is compatible with the interactive statistical decision making (ISDM)
framework of \citet{chen2024unified}; its purpose here is to make the histories, controls, losses, and
Bellman state used by the subsequent identities unambiguous.

\paragraph{Bellman state and Bellman sufficiency.}

A Bellman state $S_t=\phi_t(H_{t-1})$ determines the feasible actions,
fixed-truth predictive law and loss, and the next-state update. It need
not be the physical state or a posterior. The full history always
qualifies; a smaller state is useful when the same recursions still
hold. Definition~\ref{def:index-bellman} gives the formal requirements
and the additional conditions for posterior-indexed information accounting.

Throughout, the history, action, observation, state, environment, and index
spaces are standard Borel. All displayed kernels, losses, and state maps are
measurable, and feasible-action correspondences have measurable graphs.
Every displayed Bellman optimization is assumed to admit universally
measurable \(\varepsilon\)-selectors; exact attainment is assumed only where
it is invoked.

\paragraph{Reference dynamic program and minimax lower bound.}
Suppose the state is sufficient for the fixed-truth recursion and the
reference posterior $b_s$, posterior predictive law, and Bayes loss all
depend on the history only through the state. For the update
\(S_{t+1}=\tau_t(S_t,p_t,A_t,O_t)\), the reference Bayes recursion is
\begin{equation}
\label{eq:exact-bayes-dp}
    V_{T+1}(s)=0,
    \qquad
    V_t(s)
    =
    \inf_{p\in\Delta(\calA_t(s))}
    \left\{
        \ell(s,p)+
        \E_{a\sim p,\,o\sim P_{s,p,a}}
        V_{t+1}(\tau_t(s,p,a,o))
    \right\},
\end{equation}
where
\[
    \ell(s,p)=\E_{a\sim p}\int \ell_\omega(s,p,a)b_s(d\omega),
    \qquad
    P_{s,p,a}=\int P_{\omega,s,p,a}b_s(d\omega).
\]
This recursion evaluates the Bayes problem for the chosen reference,
not the original frequentist problem. Yao's principle connects the two:
\[
    \mathfrak R_T^\star(\Omega)
    \ge
    \sup_{\mu\in\Delta(\Omega)}
    \inf_{\Alg\in\mathfrak A_T^{\rm na}}
    \E_{\Omega\sim\mu}\E_{\Pp_\Omega^\Alg}
    \sum_{t=1}^T\ell_t(\Omega,H_{t-1},p_t,A_t).
\]
\begin{remark}[Adversarial and estimated losses]\label{remark:adv}
The same control convention accommodates a nonanticipating adversary that
observes the announced mixed action, as well as state-measurable estimated-loss
surrogates.  The precise interpretation and the lower-bound qualification are
given in Appendix~\ref{app:states}.
\end{remark}

\subsection{Information index, conditional loss, and entropy accounting}\label{subsec:information-index}

The bounds need not charge information about the full environment. An
information index is defined by a measurable map \(\chi\) on the environment
space:
\[
   Y=\chi(\Omega)\in\calY.
\]
The analyst chooses the index to suit the decision problem: it might be
an optimal arm or policy, a value function, a finite active marginal, or
the full environment. The fixed-truth recursion still uses
\(\omega^\star\). The index enters the logarithmic potential, the
posterior-reference information accounting, and the ghost event.

The loss need not be a function of \(Y\) alone. For Bayes or Fano
calculations, we therefore use the conditional index loss, assuming the
relevant regular conditional laws exist:
\begin{equation}
\label{eq:index-loss}
    \ell_t^\chi(y,h,p,a)
    :=
    \E_\mu\!\left[
        \ell_t(\Omega,h,p,a)
        \mid
        \chi(\Omega)=y,
        H_{t-1}=h
    \right].
\end{equation}
We omit \(p\) when the primitives do not depend on it. With the exact
posterior-indexed lift of Definition~\ref{def:index-bellman}, the conditional
loss and predictive law depend only on the state and can be written as
\(\ell_{t,y}^\chi(s,p,a)\) and \(P_{t,s,p,a}^y\). They are used for
indexed posterior averaging and Bellman--Fano, not for ordinary
fixed-truth Bellman control.

We use the cumulative and average losses, with the same suppression convention,
\[
L_\omega(H_T)=\sum_{t=1}^T\ell_t(\omega,H_{t-1},p_t,A_t),
    \qquad
    \bar L_\omega(H_T)=T^{-1}L_\omega(H_T),
\]
and
\[
L_\chi(y,H_T)=\sum_{t=1}^T\ell_t^\chi(y,H_{t-1},p_t,A_t),
    \qquad
    \bar L_\chi(y,H_T)=T^{-1}L_\chi(y,H_T).
\]
Under the Bayesian mixture generated by \(\Omega\sim\mu\),
\begin{equation}
\label{eq:index-loss-equality}
\E_{\Omega\sim\mu,H_T\sim\Pp_\Omega^\Alg}L_\Omega(H_T)
    =
    \E_{Y,H_T}L_\chi(Y,H_T).
\end{equation}
The upper identities allow any integrable loss increments. The quantile
lower bounds assume \(L_\chi(Y,H_T)\ge0\) almost surely under the true
mixture, not that each increment is nonnegative. If \(L_\chi\ge-B\), apply
the bounds to \(L_\chi+B\) and subtract the shift at the end.

The upper bound pays a coordinate code length; the lower bound uses a
ghost entropy. For an initial reference marginal \(q_1\) on \(\calY\), the
worst-case coordinate code length over \(\Omega_0\) is
\[
    L_0(\Omega_0;q_1,\chi)
    :=
    \sup_{\omega\in\Omega_0}\log\frac1{q_1(\chi(\omega))}.
\]
If the upper comparison is over a finite index set of size \(M\) and \(q_1\) is uniform on that index set, then \(L_0=\log M\). At a regret radius \(r\), a hard prior \(\mu\) with index marginal \(q_{1,\mu}:=\chi_\#\mu\) gives the algorithm-uniform Bellman--Fano ghost mass and entropy
\[
   p_{\mu,r}^{\star}
   =\sup_{\Alg\in\mathfrak A_T^{\rm na}}
    \Pp_{Y\sim {q_{1,\mu}},\,H_T'\sim\bar\Pp_\mu^{\Alg}}
    \{\bar L_\chi(Y,H_T')\le r\},
    \qquad
    E_{\mu,r}^{\star}:=\log\frac1{p_{\mu,r}^{\star}}.
\]
The relevant comparison is the ratio \(L_0/E_{\mu,r}^{\star}\), not the
size of the full model class. Section~\ref{sec:info-sandwich} explains the
comparison, and Appendix~\ref{app:sandwich-details} gives the conditions under
which it yields a regret bound.

\subsection{Bellman-sufficient representation}

Priors, posteriors, confidence sets, exponential-weights distributions, and
algorithmic beliefs can guide the learner or enter a bound. They do not
replace the fixed truth \(\omega^\star\) unless we explicitly take a
Bayesian average.

\begin{definition}[Indexed Bellman-sufficient representation]
\label{def:index-bellman}

Fix an index map \(\chi:\Omega\to\calY\).  A retained state process
\[
    S_t=\phi_t(H_{t-1})
\]
is a \emph{fixed-truth Bellman-sufficient representation} if the following objects are determined by the current time-state pair \((t,s)\).
\begin{enumerate}[label=(\roman*),leftmargin=2em]
\item \emph{Admissible actions.}  There is a state-measurable action set \(\calA_t(s)\) such that \(\calA_t(h)=\calA_t(s)\) whenever \(S_t(h)=s\).
\item \emph{Predictive sufficiency.} For every environment \(\omega\), mixed action \(p\in\Delta(\calA_t(s))\), and realized action \(a\in\calA_t(s)\), there is a kernel \(P_{\omega,t,s,p,a}\) such that, for every history \(h\) with \(S_t(h)=s\),
\[
    P_{\omega,t}(\cdot\mid h,p,a)=P_{\omega,t,s,p,a}.
\]
When the environment does not react to the announced mixed action, the argument \(p\) is suppressed.
\item \emph{Loss sufficiency.}  There is an integrable state-measurable cost function \(\ell_{\omega,t}(s,p,a)\in\mathbb R\) such that
\[
\ell_t(\omega,h,p,a)=\ell_{\omega,t}(s,p,a)
    \qquad\text{whenever }S_t(h)=s.
\]
\item \emph{Update sufficiency.}  There is an update map or controlled kernel \(\tau_t\) such that the next retained state is generated from the retained state and the current interaction.  In deterministic-update notation,
\[
S_{t+1}=\tau_t(S_t,p_t,A_t,O_t),
\]
with \(p_t\) suppressed when irrelevant.  Random next contexts, physical next states, stage markers, and public randomization are part of the observation packet \(O_t\).
\end{enumerate}
These conditions let us evaluate the Bellman recursion on the state under
each fixed truth.

For exact posterior-reference information accounting on the state, fix a
prior or reference law \(\mu\) and first define the full-history quantities
\[
    b_{t,h}:=\mathcal L_\mu(\Omega\mid H_{t-1}=h),
    \qquad q_{t,h}:=\chi_\# b_{t,h},
\]
and, for \(q_{t,h}\)-almost every \(y\),
\[
    b_{t,h}^y:=\mathcal L_\mu(\Omega\mid H_{t-1}=h,\chi(\Omega)=y).
\]
The full-history index-conditional predictive law and loss are
\[
    P_{t,h,p,a}^{y}:=
    \int P_{\omega,t}(\cdot\mid h,p,a)\,b_{t,h}^y(d\omega),
\]
\[
    \ell_{t,h,p,a}^{\chi}(y):=
    \int \ell_t(\omega,h,p,a)\,b_{t,h}^y(d\omega).
\]
An \emph{exact posterior-indexed lift on the state} requires these objects
to depend on the history only through \(S_t\). Whenever \(S_t(h)=s\),
there must be state-measurable versions satisfying
\begin{align}\label{eq:lift1}
    q_{t,h}=q_{t,s},
\end{align}
and, for \(q_{t,s}\)-almost every \(y\),
\begin{align}\label{eq:lift2}
    P_{t,h,p,a}^{y}=P_{t,s,p,a}^{y},
    \qquad
    \ell_{t,h,p,a}^{\chi}(y)=\ell_{t,y}^{\chi}(s,p,a),
\end{align}
together with a state-measurable reference update for \(q_{t,s}\). This lift
is needed for the posterior-averaged chain rule, the Bayes
information-potential program, and the Bellman--Fano lower bound on the
compressed state, not for the fixed-truth recursion alone.

\end{definition}

When the state includes the stage, we omit the time subscript. When the
environment does not respond to the mixed action, we omit \(p\) and write
\(P_{\omega,s,a}\) and \(\ell_\omega(s,a)\). The state and the remaining
control arguments must still determine the Bellman primitives.

\paragraph{Bellman sufficiency versus classical sufficiency and Bellman rank.}

Classical model sufficiency retains the full posterior of $\Omega$ and
therefore every fixed index marginal. Index-marginal sufficiency asks only
for the posterior of $Y=\chi(\Omega)$. Bellman sufficiency instead asks
whether the state determines the feasible actions, predictive laws, losses,
and updates. It need not preserve a posterior, and an index-sufficient state
need not determine the dynamics. The lift in
\eqref{eq:lift1}--\eqref{eq:lift2} supplies what the indexed recursion needs
without requiring sufficiency for the full environment.

Bellman rank \citep{jiang2017bellman} concerns a factorization of expected
Bellman errors relative to a function class, a roll-in distribution, and
a Bellman error evaluation family. A low-rank factorization need not provide
a state that determines the next-observation law, fixed-truth loss,
reference update, or conditional laws for an index. To use such a
representation here, one must show that the Bellman primitives depend on
the history only through that representation.

Appendix~\ref{app:states} collects basic examples involving sufficient statistics, posterior states, and finite marginals, and shows how the general framework encompasses the standard Decision Making with Structured Observations (DMSO) setting \citep{foster2021statistical,foster2023tight}.
\subsection{Why sufficient states and compression matter}\label{subsec:why-sufficient-state}

Compression is optional: the full history is already a Bellman state.
A smaller state can simplify the recursion and reduce the information
charged about the chosen index, provided the relevant quantities still
depend only on that state. Such a state need not be unique or minimal.

\begin{proposition}[State compression and information monotonicity]
\label{prop:state-compression-info}
Let \(S_t=\phi_t(H_{t-1})\) be a fixed-truth Bellman-sufficient representation.
Let \(\mathfrak C_t\) be a state-measurable comparison correspondence,
constant on the fibers of \(S_t\). Suppose a one-step Bellman cost
\(c_\omega\) and a continuation value \(F_{t+1}\) are state-measurable, so
that
\[
    c_\omega(h,p)=c_\omega(s,p),
    \qquad
    F_{t+1}(H_t)=f_{t+1}(S_{t+1}),
    \qquad s=S_t(h).
\]
Then the full-history Bellman operator factors through \(S_t\): for every \(h\) with \(S_t(h)=s\),
\[
    \inf_{p\in\Delta(\calA_t(h))}
    \sup_{\omega\in{\mathfrak C_t(h)}}
    \left\{
        c_\omega(h,p)+
        \E_{\omega,h,p}F_{t+1}(H_t)
    \right\}
    =
    \inf_{p\in\Delta(\calA_t(s))}
    \sup_{\omega\in{\mathfrak C_t(s)}}
    \left\{
        c_\omega(s,p)+
        \E_{\omega,s,p}f_{t+1}(S_{t+1})
    \right\},
\]
Moreover, for any reference law and any fixed index \(Y=\chi(\Omega)\),
\[
    I(Y;S_t)\le I(Y;H_{t-1}).
\]
When $I(Y;H_{t-1})<\infty$, equality in the information display holds if and only if index-marginal sufficiency holds, $Y\perp H_{t-1}\mid S_t$.
\end{proposition}

\proofinappendix{app:proof-prop:state-compression-info}

A strict reduction in index information can help the sandwich comparison
only if both the upper and lower recursions remain valid on the compressed
state.

\section{Indexed information and exact AIR/MAIR identities}
\label{sec:identity}

\subsection{Posterior-reference histories and index information}

The mutual-information chain rule uses the exact posterior-indexed lift of
Definition~\ref{def:index-bellman}. The fixed-truth log gain below need not
come from that posterior and cannot in general replace its increments.

For a prior $\mu$ and algorithm $\Alg$, define the Bayesian posterior-reference trajectory law
\[
    \bar\Pp_\mu^\Alg:=\int \Pp_\omega^\Alg\,\mu(d\omega).
\]
A reference history is denoted
\[
    H_T'\sim \bar\Pp_\mu^\Alg,
    \qquad
   S_t'=\phi_t(H_{t-1}'),
    \qquad
    b_t'=\mathcal L_\mu(\Omega\mid H_{t-1}'),
\]
and the algorithm's reference decision kernel is
\[
    p_t'=p_t(\cdot\mid H_{t-1}').
\]

\begin{definition}[Indexed information gain]
\label{def:index-info}
At time $t$ and state $s$, let $q_s$ be the posterior law of
$Y=\chi(\Omega)$, and let $P_{s,p,a}^y$ and $P_{s,p,a}$ be the conditional
and marginal predictive laws from Definition~\ref{def:index-bellman}.
We suppress $t$ in these laws when the stage is part of $s$. For
$p\in\Delta(\calA_t(s))$, define
\begin{equation}
\label{eq:index-info}
    \calI_\chi(s,p)
    :=
    \E_{a\sim p}\int \KL(P_{s,p,a}^y\|P_{s,p,a})q_s(dy).
\end{equation}
We omit $p$ from the predictive laws if the environment does not respond
to the announced mixed action. This is model-index information for
$Y=\Omega$ and action-index information for $Y=A^\star(\Omega)$.
In the latter case the belief is still over environments, but the
information target is the optimal action.
\end{definition}

Define the cumulative and average indexed information of an algorithm by
\begin{equation}
\label{eq:average-info}
    C_\chi^\Alg(\mu)
    :=
    \E_{H_T'\sim\bar\Pp_\mu^\Alg}
    \sum_{t=1}^T \calI_\chi(S_t',p_t'),
    \qquad
    \bar C_\chi^\Alg(\mu):=\frac1T C_\chi^\Alg(\mu).
\end{equation}

\begin{proposition}[Reference-history indexed chain rule]
\label{prop:index-chain}
Assume that the retained state is an exact posterior-indexed lift in the sense of Definition~\ref{def:index-bellman}, that the displayed KL terms are well defined and integrable, and that the initial packet $O_0$ is deterministic or independent of $Y$. For every prior $\mu$, index map $\chi$, and algorithm $\Alg$,
\begin{equation}
\label{eq:index-chain}
    I_\mu(Y;H_T)
    =
    \E_{H_T'\sim\bar\Pp_\mu^\Alg}
    \sum_{t=1}^T\calI_\chi(S_t',p_t')
    =
    T\bar C_\chi^\Alg(\mu).
\end{equation}
If $O_0$ may carry information about $Y$, the right-hand side equals $I_\mu(Y;H_T\mid O_0)$ and the unconditional identity has the additional term $I_\mu(Y;O_0)$.
\end{proposition}

\proofinappendix{app:proof-prop:index-chain}

\subsection{Fixed-truth indexed AIR bracket}
\label{subsec:air-mair-gradient}

The Bellman--Fano lower bound uses the posterior-averaged information
in \eqref{eq:index-info}. The upper program instead tracks the change
in the log probability of a fixed index value after each action and
observation. For finite or countable indices, the formulas below use
positive point masses. For a dominated non-atomic index, \(q(y)\) and
\(q_\nu^+(y\mid\cdot)\) denote densities with respect to a fixed dominating
measure; the formulas require positive densities at the evaluated
coordinates and integrable log-density ratios. Only in the discrete case
can we drop the terminal log loss on the basis of its sign.

\paragraph{Coefficient convention.}
We use a multiplier notation $\gamma>0$ to scale the information terms: a log gain $G$ contributes $\gamma G$, and an initial code length $\log(1/q)$ contributes $\gamma\log(1/q)$. In the temperature notation of \citet{lattimore2021mirror,xu2025bayesian,liu2025decision,liu2026improved}, which is convenient for the dual mirror-descent interpretation, $\gamma=1/\eta$.

Fix a state $s$, an action distribution $p\in\Delta(\calA_t(s))$, and a pair belief
\[
    \nu\in\Delta(\Omega\times\calY)
\]
over environments and index values.  Let $q_\nu$ be the $\calY$-marginal of $\nu$, and let $q$ be a current reference marginal that is positive on every coordinate evaluated by the logarithmic score.  The pair belief induces, for each realized action $a$, the predictive mixture
\[
    P_{\nu,s,p,a}(\cdot)
    :=
    \int P_{\omega,t}(\cdot\mid s,p,a)\,\nu(d\omega,dy),
\]
and the posterior index marginal
\[
    q_\nu^+(B\mid s,p,a,o)
    :=
    \nu(Y\in B\mid s,p,a,O=o),
    \qquad B\subseteq\calY.
\]
When the observation laws are dominated by a common measure and have densities $p_{\omega,t,s,p,a}(o)$, this update is explicitly
\begin{equation}
\label{eq:qplus-pair-belief}
    q_\nu^+(B\mid s,p,a,o)
    =
    \frac{
        \int_{\Omega\times B}p_{\omega,t,s,p,a}(o)\,\nu(d\omega,dy)
    }{
        \int_{\Omega\times\calY}p_{\omega,t,s,p,a}(o)\,\nu(d\omega,dy)
    } .
\end{equation}
In ordinary stochastic bandits and MDPs, the predictive laws do not depend
on the announced distribution $p$.

Let $\ell_{\omega,y}(s,p)$ denote the comparison loss associated with pair $(\omega,y)$; for a valid fixed truth one takes $y=\chi(\omega)$ and $\ell_{\omega,\chi(\omega)}=\ell_\omega$.  Define the general indexed AIR functional
\begin{equation}
\label{eq:air-functional-general}
\begin{aligned}
    \mathfrak A_{q,\gamma}(s,p,\nu)
    &:={}
    \E_{(\Omega,Y)\sim\nu}\ell_{\Omega,Y}(s,p)
    -\gamma\E_{(\Omega,Y)\sim\nu,\,a\sim p,\,O\sim P_{\Omega,t}(\cdot\mid s,p,a)}
    \log\frac{q_\nu^+(Y\mid s,p,a,O)}{q(Y)} .
\end{aligned}
\end{equation}
Equivalently,
\begin{equation}
\label{eq:air-functional-info-form}
    \mathfrak A_{q,\gamma}(s,p,\nu)
    =
    \E_\nu\ell_{\Omega,Y}(s,p)
    -\gamma \calI_\nu(Y;O\mid s,p)
    -\gamma\KL(q_\nu\|q),
\end{equation}
where
\[
    \calI_\nu(Y;O\mid s,p)
    :=
    \E_{a\sim p}\int \KL(P_{\nu,s,p,a}^{y}\|P_{\nu,s,p,a})q_\nu(dy)
\]
is the one-step information under the candidate pair belief. Under exact
Bayes, $\nu$ is the current posterior over $(\Omega,\chi(\Omega))$ and
$q=q_\nu$. In AIR/AMS/EBO, $\nu$ can instead be chosen by optimization; the
last term of \eqref{eq:air-functional-info-form} penalizes a mismatch
between its marginal $q_\nu$ and the reference $q$.

\begin{lemma}[AIR gradient bracket]
\label{lem:air-gradient-bracket}
Assume a finite or countable index with positive reference masses and finite
or dominated observation laws, or use the dominated-index density convention
above. Assume the displayed G\^ateaux derivatives and log ratios exist and
are integrable. Up to an additive constant on the probability simplex,
\begin{equation}
\label{eq:air-gradient}
    \frac{\partial}{\partial\nu(\omega,y)}\mathfrak A_{q,\gamma}(s,p,\nu)
    =
    \ell_{\omega,y}(s,p)
    -\gamma
    \E_{a\sim p,\,O\sim P_{\omega,t}(\cdot\mid s,p,a)}
    \log\frac{q_\nu^+(y\mid s,p,a,O)}{q(y)} .
\end{equation}
Consequently, for every fixed pair $(\omega,y)$,
\begin{align}
\label{eq:air-bracket}
    B_{q,\gamma}^{\AIR}(s,p,\nu;\omega,y)
    &:={}
    \mathfrak A_{q,\gamma}(s,p,\nu)
    +
    \left\langle
        \nabla_\nu\mathfrak A_{q,\gamma}(s,p,\nu),
        \delta_{(\omega,y)}-\nu
    \right\rangle  \nonumber\\
    &=
    \ell_{\omega,y}(s,p)
    -\gamma
    \E_{a\sim p,\,O\sim P_{\omega,t}(\cdot\mid s,p,a)}
    \log\frac{q_\nu^+(y\mid s,p,a,O)}{q(y)} .
\end{align}
\end{lemma}

\proofinappendix{app:proof-lem:air-gradient-bracket}

\subsection{Exact AIR identity and model-index specialization}
\label{subsec:air-mair-specialization}

We give simplified statements and alternative proofs of several AIR/MAIR
regret identities. \citet[Section~5.2 and concluding discussion]{xu2025bayesian}
allow the model and decision kernels in these identities to depend on
the history. This dependence motivates the Bellman formulation.

The next index marginal must come from the same pair belief \(\nu_t\)
that appears in the AIR functional. With this update, the log ratios in
Lemma~\ref{lem:air-gradient-bracket} telescope. Other updates may still
give log-potential bounds, but require separate calibration or relaxation
arguments unless they admit this pair-belief representation.

\begin{theorem}[Exact indexed AIR regret identity]
\label{thm:air-regret-identity}
At each reached state $S_t$, let $q_t$ be the algorithm's reference
marginal, $\nu_t\in\Delta(\Omega\times\calY)$ its pair belief, and
$p_t\in\Delta(\calA_t(S_t))$ its decision law. After sampling
$A_t\sim p_t$ and observing $O_t$, update by
\begin{equation}
\label{eq:air-reference-update}
    q_{t+1}(\cdot)=q_{\nu_t}^+(\cdot\mid S_t,p_t,A_t,O_t).
\end{equation}
Fix a pair $(\omega^\star,y^\star)$ with \(q_t(y^\star)>0\) and
\(q_{t+1}(y^\star)>0\) almost surely along the trajectory. Assume the
displayed losses and log ratios are integrable. For every $\gamma>0$,
\begin{equation}
\label{eq:air-regret-identity}
\begin{aligned}
    \E_{\omega^\star}^{\Alg}\sum_{t=1}^T\ell_{\omega^\star,y^\star}(S_t,p_t)
    &={}
    \gamma\E_{\omega^\star}^{\Alg}
    \log\frac{q_{T+1}(y^\star)}{q_1(y^\star)}
    +
    \E_{\omega^\star}^{\Alg}
    \sum_{t=1}^T
    B_{q_t,\gamma}^{\AIR}(S_t,p_t,\nu_t;\omega^\star,y^\star).
\end{aligned}
\end{equation}
For the original fixed-truth regret, take $y^\star=\chi(\omega^\star)$ and $\ell_{\omega^\star,y^\star}=\ell_{\omega^\star}$.  The expectation is over action randomization, observations, and any exogenous algorithmic randomness.
\end{theorem}

\proofinappendix{app:proof-thm:air-regret-identity}

The variational interpretation through posterior scoring and Danskin's theorem is given in Appendix~\ref{app:air-mair}.

\paragraph{The terminal code length.}
To see what is lost by passing from the identity to an upper bound,
rewrite \eqref{eq:air-regret-identity} as
\[
    \E_{\omega^\star}^{\Alg}\sum_{t=1}^T\ell_{\omega^\star,y^\star}(S_t,p_t)
    =
    \gamma\log\frac1{q_1(y^\star)}
    -\gamma\E_{\omega^\star}^{\Alg}\log\frac1{q_{T+1}(y^\star)}
    +
    \E_{\omega^\star}^{\Alg}\sum_{t=1}^T
    B_{q_t,\gamma}^{\AIR}(S_t,p_t,\nu_t;\omega^\star,y^\star).
\]
For a finite or countable index, the terminal term is nonpositive.
Dropping it gives the usual upper bound and discards the remaining
coordinate code length. With densities, the term can be positive and
must be retained or bounded separately. The minimax lower bound is a
separate argument: it controls the probability of a low-loss ghost
history for a randomly sampled index.

\paragraph{Stationary-posterior specialization.}
In the exact lift of Definition~\ref{def:index-bellman}, $\nu_t$ is the
posterior of $(\Omega,\chi(\Omega))$, $q_t$ is its index marginal, and
$q_{\nu_t}^+$ is the next posterior marginal. For $y=\chi(\omega)$, the
fixed-truth log gain is
\begin{equation}
\label{eq:index-fixed-loggain}
    J_\chi(s,p;\omega)
    :=
    \E_{a\sim p,\,O\sim P_{\omega,t}(\cdot\mid s,p,a)}
    \log\frac{q_{\nu_s}^+(y\mid s,p,a,O)}{q_s(y)} .
\end{equation}
Equivalently, when the relevant laws are dominated,
\begin{equation}
\label{eq:index-kl-difference}
    J_\chi(s,p;\omega)
    =
    \E_{a\sim p}\left[
        \KL(P_{\omega,s,p,a}\|P_{s,p,a})
        -
        \KL(P_{\omega,s,p,a}\|P_{s,p,a}^{\chi(\omega)})
    \right].
\end{equation}
Averaging $J_\chi$ over the posterior of $\omega$ given $Y=y$, then over
$y\sim q_s$, gives $\calI_\chi(s,p)$.

For $\gamma>0$, the exact-posterior fixed-truth indexed bracket is
\begin{equation}
\label{eq:index-fixed-bracket}
    B_{\chi,\gamma}(s,p;\omega)
    :=
    \ell_\omega(s,p)-\gamma J_\chi(s,p;\omega).
\end{equation}
It is the specialization of \eqref{eq:air-bracket} obtained by taking the pair belief to be the current posterior over $(\Omega,\chi(\Omega))$ and the reference marginal to be $q_s$.

\paragraph{Environment-index specialization.}
If $Y=\Omega$, the pair belief is concentrated on the graph $y=\omega$.
Write $\mu$ for the current environment belief and $\rho$ for the
reference model marginal. Use masses for finite or countable model
classes and densities with respect to a fixed dominating measure
otherwise; the displayed density ratios must be positive at the
evaluated truth and integrable. Then \eqref{eq:air-functional-general}
becomes
\begin{equation}
\label{eq:mair-objective}
    \MAIR_{\rho,\gamma}(s,p,\mu)
    :=
    \E_{\omega\sim\mu}\ell_\omega(s,p)
    -\gamma\E_{a\sim p}I_\mu(\Omega;O\mid s,p,a)
    -\gamma\KL(\mu\|\rho).
\end{equation}
The fixed-truth bracket is
\begin{align}
\label{eq:mair-bracket}
    B_{\rho,\gamma}^{\MAIR}(s,p,\mu;\omega^\star)
    &=
    \ell_{\omega^\star}(s,p)
    -\gamma\E_{a\sim p}\KL(P_{\omega^\star,s,p,a}\|P_{\mu,s,p,a})
    -\gamma\log\frac{\mu(\omega^\star)}{\rho(\omega^\star)} .
\end{align}
If $\rho_{t+1}$ is the Bayes posterior update of $\mu_t$ after $(S_t,p_t,A_t,O_t)$, the required point masses or densities remain positive at $\omega^\star$, and the losses and log ratios are integrable, then
\begin{equation}
\label{eq:mair-regret-identity}
    \E_{\omega^\star}^{\Alg}L_{\omega^\star}(H_T)
    =
    \gamma\E_{\omega^\star}^{\Alg}
    \log\frac{\rho_{T+1}(\omega^\star)}{\rho_1(\omega^\star)}
    +
    \E_{\omega^\star}^{\Alg}\sum_{t=1}^T
    B_{\rho_t,\gamma}^{\MAIR}(S_t,p_t,\mu_t;\omega^\star).
\end{equation}
The choice $\mu_t=\rho_t$ removes the log-ratio correction in
\eqref{eq:mair-bracket}, leaving the KL-DEC offset. The log score is for
the model itself; in action-index AIR it is for the optimal action or
policy. When only this index matters, the latter formulation can charge
less information.

Lemma~\ref{lem:mair-gradient} in Appendix~\ref{app:air-mair} derives the
KL-DEC offset, the robust AMS/EBO control at an interior maximizer, and the
Hellinger offset from the square-root posterior.

\section{Upper bounds from logarithmic Bellman potentials}
\label{sec:upper}

The fixed-truth telescope tracks the log score at
$y^\star=\chi(\omega^\star)$. At state $s$, it uses the current reference
marginal $q_s$. A control $u$ specifies an action law
$p_u\in\Delta(\calA_t(s))$ and a state-measurable update
\[
   q_u^+(\cdot\mid s,a,o)\in\Delta(\calY).
\]
In the AIR/MAIR specialization, this update is generated by a pair belief \(\nu_u\in\Delta(\Omega\times\calY)\), namely
\[
    q_u^+(\cdot\mid s,a,o)=q_{\nu_u}^+(\cdot\mid s,p_u,a,o),
\]
where \(q_\nu^+\) is defined in \eqref{eq:qplus-pair-belief}.
Exact Bayes and stationary-posterior MAIR use updates of this form.
Confidence-to-belief and exponential-weights procedures can also
supply reference updates for the general telescope. To use the AIR
gradient identity, however, the update must be generated by the same
admissible pair belief that enters the AIR functional.

The fixed-truth reference log gain used by the Bellman program is
\begin{equation}
\label{eq:reference-loggain}
    \mathsf G_\chi(s,u;\omega^\star)
    :=
    \E_{a\sim p_u,\,O\sim P_{\omega^\star,t}(\cdot\mid s,p_u,a)}
    \log\frac{q_u^+(y^\star\mid s,a,O)}{q_s(y^\star)} .
\end{equation}
For an AIR control \(u=(p,\nu)\), write the gain as
\(\mathsf G_\chi(s,p,\nu;\omega^\star)\). When \(\nu\) is the exact
posterior pair law, this is \(J_\chi(s,p;\omega^\star)\) in
\eqref{eq:index-fixed-loggain}.

For a state-measurable continuation value $V_t$ and $\gamma>0$, define
the fixed-truth potential
\begin{equation}
\label{eq:specialized-log-potential}
    \Phi_t^{V,\gamma}(s;y^\star)
    :=
    V_t(s)+\gamma\log\frac1{q_s(y^\star)}.
\end{equation}
The corresponding log-potential Bellman bracket is
\begin{equation}
\label{eq:upper-bracket}
\begin{aligned}
    \mathfrak B_{\chi}^{V,\gamma}(s,u;\omega^\star)
    &:={}
    \ell_{\omega^\star}(s,p_u)
    +\E_{\omega^\star}[V_{t+1}(S^+)\mid s,u]
    -V_t(s)
    -\gamma\mathsf G_\chi(s,u;\omega^\star).
\end{aligned}
\end{equation}
Here \(A\sim p_u\) and
\(O\sim P_{\omega^\star,t}(\cdot\mid s,p_u,A)\). The reference marginal
is updated by \(q_u^+(\cdot\mid s,A,O)\); the rule specified by \(u\)
updates the other state variables. For \(u=(p,\nu)\), we write
\(\mathfrak B_\chi^{V,\gamma}(s,p,\nu;\omega^\star)\).
Summing over time gives
\begin{equation}
\label{eq:potential-identity}
\begin{aligned}
    \E_{\omega^\star} L_{\omega^\star}(H_T)
    &={}
    V_1(s_1)-\E_{\omega^\star}V_{T+1}(S_{T+1})
    +\gamma\log\frac1{q_1(y^\star)}
    -\gamma\E_{\omega^\star}\log\frac1{q_{T+1}(y^\star)}        \\
    &\quad+
    \E_{\omega^\star}\sum_{t=1}^T
    \mathfrak B_{\chi}^{V,\gamma}(S_t,u_t;\omega^\star).
\end{aligned}
\end{equation}
The continuation term is the expected change in \(V\). When \(V\) is a
Bellman value function or supersolution, it accounts for the effect of
the current action on later losses and decisions. For example, in a
controlled Markov process, \(S^+\) contains the next physical state,
and \(V_{t+1}(S^+)\) accounts for the cost from that state onward.
For some bandit and episodic reinforcement learning bounds, one can instead set \(V_t\equiv0\), with \(t\) indexing whole episodes in the latter case.

Keeping both terminal terms gives an identity. A supersolution bounds
the bracket and hence the cumulative loss. The sign of the terminal
log term depends on whether \(q\) denotes a mass or a density, as
discussed after Theorem~\ref{thm:air-regret-identity}.

\begin{theorem}[Generic fixed-truth indexed AIR/MAIR upper bound]
\label{thm:mair-upper}

Fix $\omega^\star$ and $y^\star=\chi(\omega^\star)$. Under the
point-mass or dominated-density convention, assume that $q_1(y^\star)>0$
and that the reference update stays positive at $y^\star$ along the
trajectory. Assume that the losses, log ratios, continuation values,
and displayed error sum are integrable. Suppose the algorithm chooses
a decision law $p_t$ and a state-measurable prediction belief $\nu_t$,
uses $q_{t+1}=q_{\nu_t}^+(\cdot\mid S_t,p_t,A_t,O_t)$, and at every
reached state satisfies
\[
    \mathfrak B_{\chi}^{V,\gamma}(S_t,p_t,\nu_t;\omega^\star)
    \le \varepsilon_t,
\]
then
\begin{equation}
\label{eq:fixed-truth-log-potential-telescope}
    \E_{\omega^\star}^{\Alg} L_{\omega^\star}(H_T)
    \le
    V_1(s_1)+\gamma\log\frac1{q_1(y^\star)}
    -\E_{\omega^\star}^{\Alg}V_{T+1}(S_{T+1})
    -\gamma\E_{\omega^\star}^{\Alg}\log\frac1{q_{T+1}(y^\star)}
    +\E_{\omega^\star}^{\Alg}\sum_{t=1}^T\varepsilon_t .
\end{equation}

\end{theorem}

\proofinappendix{app:proof-thm:mair-upper}

A fixed estimation procedure maps histories to reference updates. The
retained state includes the current reference marginal
\begin{equation}
\label{eq:state-posterior}
   q_t\in\Delta(\calY),
\end{equation}
and can also retain the information the procedure needs to update it.
In exact Bayes, the update comes from the posterior over
$(\Omega,\chi(\Omega))$; in a frequentist analysis, the state can
instead include an algorithmic belief, a confidence set, or an
exponential-weights density. Averaging the log score under an index
marginal $\nu_\chi$ gives
\begin{equation}
\label{eq:information-potential}
    \E_{Y\sim\nu_\chi}\gamma\log\frac1{q_s(Y)}
    =
    \gamma H(\nu_\chi)+\gamma\KL(\nu_\chi\|q_s),
\end{equation}
and for a finite index set the same logarithmic scoring rule has the KL-dual log-partition form
\begin{equation}
\label{eq:log-partition-dual}
    \Psi_\gamma(q,z)
    :=
\gamma\log\sum_{y\in\calY}q(y)e^{z_y/\gamma}
    =
    \sup_{r\in\Delta(\calY)}
    \left\{
        \langle r,z\rangle-
        \gamma\KL(r\|q)
    \right\}.
\end{equation}
These formulas play different roles: fixed-truth guarantees use the
coordinate score in \eqref{eq:specialized-log-potential}; Bayesian
identities use the average in \eqref{eq:information-potential}; and
AMS/EBO relaxations use the dual form \eqref{eq:log-partition-dual}.

In the model-index specialization, write the posterior/reference belief as $\mu_t$.  If the maintained reference update is exact Bayes, then $\mathsf G_\chi=J_\chi$ and the fixed-truth MAIR information increment is
\begin{equation}
\label{eq:fixed-truth-mair-info-upper}
    J_t^{\MAIR}(p;\omega^\star)
    :=
    \E_{a\sim p}
    \KL\!\left(
        P_{\omega^\star,t}(\cdot\mid S_t,a)
        \,\middle\|\,
        P_{\mu_t,t}(\cdot\mid S_t,a)
    \right).
\end{equation}
The action-index version instead scores the optimal-action
coordinate, using $Y=A^\star$, the reference marginal $q_t$, and
the posterior coordinate $q^+(a^\star\mid a,o)$; see
\eqref{eq:index-fixed-loggain} for the log-gain formula.

\subsection{Information-potential Bellman programming}\label{subsec:information-Bellman}

The frequentist Bellman program depends on the state, index, comparison
set, and control family. At state $s$, the supremum ranges over candidate
truths $\mathfrak C_t(s)\subseteq\Omega$, and the infimum over the
learner's controls $\mathfrak U_t(s)$. The resulting value is exact for
this program, not necessarily for the unrestricted minimax problem. A control
$u\in\mathfrak U_t(s)$ specifies an action law
$p_u\in\Delta(\calA_t(s))$ and a state-measurable reference-update rule.  When
this update is represented by a pair belief, write that belief as $\nu_u$ and
set $q^+=q_{\nu_u}^+$.  If the pair belief is fixed by the state, then $u$ is
simply the action law together with this fixed update.  The robust AMS/EBO
belief variable is not this algorithmic control; it is introduced later under
a supremum to control the bracket.  A control for which an evaluated
coordinate log gain is not well defined on $\mathfrak C_t(s)$ is assigned
bracket value $+\infty$.  For $\gamma>0$, define
\begin{equation}
\label{eq:fixed-truth-log-penalized-dp}
    W_{T+1}^\gamma(s)=0,
    \qquad
    W_t^\gamma(s)
    :=
    \inf_{u\in\mathfrak U_t(s)}
    \sup_{\omega\in\mathfrak C_t(s)}
    \left\{
        \ell_\omega(s,p_u)
        +\E_\omega[W_{t+1}^\gamma(S^+)\mid s,u]
        -\gamma\mathsf G_\chi(s,u;\omega)
    \right\}.
\end{equation}
A measurable selector attaining the infimum in
\eqref{eq:fixed-truth-log-penalized-dp} gives the Bellman policy.
For the AIR control $u=(p,\nu)$,
\[
    \mathsf G_\chi(s,u;\omega)=\mathsf G_\chi(s,p,\nu;\omega).
\]
In exact Bayes, $\nu$ is fixed by the posterior state. The program is
related to admissible relaxations \citep{rakhlin2012relax}, especially the
partial-information formulation of \citet{rakhlin2016bistro}, where
estimators and potentials encode uncertainty. Minimum-complexity,
information-risk, and Gibbs estimation give static counterparts
\citep{barron1991minimum,yang1999information,zhang2006information}.
Here the logarithmic penalty is updated at each step of the recursion.

\begin{theorem}[Frequentist log-penalized Bellman upper bound]
\label{thm:frequentist-info-risk-upper}
Fix $\Omega_0\subseteq\Omega$ and $\gamma>0$.  Suppose the comparison
sets $\mathfrak C_t$ are surely calibrated for $\Omega_0$: under every fixed
truth $\omega^\star\in\Omega_0$, the inclusion
$\omega^\star\in\mathfrak C_t(S_t)$ holds almost surely at every reached
round.  Suppose the algorithm chooses a local control $u_t\in\mathfrak U_t(S_t)$, writes $p_t=p_{u_t}$ for its action law, updates the index marginal by the reference-update rule contained in $u_t$, and satisfies
\[
    \sup_{\omega\in\mathfrak C_t(S_t)}
    \left\{
        \ell_\omega(S_t,p_{u_t})
        +\E_\omega[W_{t+1}^\gamma(S^+)\mid S_t,u_t]
        -\gamma\mathsf G_\chi(S_t,u_t;\omega)
    \right\}
    \le
    W_t^\gamma(S_t)+\varepsilon_t .
\]
Then, for every fixed truth $\omega^\star\in\Omega_0$ with $y^\star=\chi(\omega^\star)$ and $q_1(y^\star)>0$, provided the index is finite or countable, the reference update remains positive on $y^\star$ along the trajectory, and the losses, coordinate log ratios, continuation values, and displayed error sum are integrable,
\begin{equation}
\label{eq:frequentist-info-risk-upper}
\E_{\omega^\star}^{\Alg}L_{\omega^\star}(H_T)
    \le
    W_1^\gamma(s_1)
    +\gamma\log\frac1{q_1(y^\star)}
    +\E_{\omega^\star}^{\Alg}\sum_{t=1}^T\varepsilon_t.
\end{equation}
\end{theorem}

\proofinappendix{app:proof-thm:frequentist-info-risk-upper}

\begin{remark}[High-probability calibration]
The theorem assumes sure calibration. A future-dependent event
$\mathcal E$ cannot in general be inserted into the conditional Bellman
brackets and charged afterward by $\E[L^+\1\{\mathcal E^c\}]$:
restricting the continuation term to $\mathcal E$ can invalidate the
conditional inequality. A high-probability version requires separate
control, for example via stopping, a supermartingale, or predictable
failure brackets. For kernel UCB below, we first prove the regret bound
on the confidence event and then take expectations.
\end{remark}

For the Bayesian recursion, average over the current posterior instead
of taking a supremum over the comparison set. Writing $H_\chi(s):=H(q_s)$
for the posterior index entropy, the coordinate identity gives
\[
    \E[H_\chi(S_{t+1})\mid s,p]-H_\chi(s)=-\calI_\chi(s,p).
\]
For a fixed multiplier $\gamma>0$, define the Bayes information-regularized value
\begin{equation}
\label{eq:info-regularized-bellman}
    U_{T+1}^\gamma\equiv0,
    \qquad
    U_t^\gamma(s)
    :=
    \inf_{p\in\Delta(\calA_t(s))}
    \left\{
        \ell(s,p)+
        \E[U_{t+1}^\gamma(S_{t+1})\mid s,p]
        -\gamma\calI_\chi(s,p)
    \right\}.
\end{equation}
Equivalently,
\begin{equation}
\label{eq:entropy-bellman-potential}
    \Phi_t^\gamma(s):=U_t^\gamma(s)+\gamma H_\chi(s)
\end{equation}
solves the admissible-relaxation recursion
\begin{equation}
\label{eq:phi-bellman-relaxation}
    \Phi_t^\gamma(s)
    =
    \inf_{p\in\Delta(\calA_t(s))}
    \left\{
        \ell(s,p)+
        \E[\Phi_{t+1}^\gamma(S_{t+1})\mid s,p]
    \right\},
    \qquad
    \Phi_{T+1}^\gamma(s)=\gamma H_\chi(s).
\end{equation}

\begin{theorem}[Posterior-averaged information-potential upper bound]
\label{thm:info-potential-dp-upper}
Assume the indexed Bellman representation is exact under the Bayesian mixture law, $Y$ is finite or countable, and $H_\mu(Y)<\infty$.  If a policy chooses $p_t$ satisfying
\[
    \ell(S_t,p_t)
    +\E[U_{t+1}^\gamma(S_{t+1})\mid S_t,p_t]
    -\gamma\calI_\chi(S_t,p_t)
    \le
    U_t^\gamma(S_t)+\varepsilon_t,
\]
then
\begin{equation}
\label{eq:exact-dp-upper-info}
    \E L_\Omega(H_T)
    \le
    U_1^\gamma(s_1)
    +\gamma I_\mu(Y;H_T)
	    +\E\sum_{t=1}^T\varepsilon_t
    \le
    U_1^\gamma(s_1)
    +\gamma H_\mu(Y)
	    +\E\sum_{t=1}^T\varepsilon_t .
\end{equation}
This is the posterior average of the fixed-truth coordinate telescope, not a pointwise replacement for Theorem~\ref{thm:frequentist-info-risk-upper}.
\end{theorem}

\proofinappendix{app:proof-thm:info-potential-dp-upper}

The fixed-truth telescope and Bellman program supply the upper bound in
the minimax comparison of Section~\ref{sec:info-sandwich}.

\subsection{UCB, E2D, and AMS/EBO as tractable controls}
\label{subsec:algorithm-family-summary}

Solving the Bellman program may be computationally difficult. UCB,
E2D, and AMS/EBO give different ways to bound or relax its bracket.
UCB uses calibrated optimism and, in the examples below, an
elliptical-potential bound on the realized variances before they are
scaled by the confidence multiplier. E2D fixes an estimation procedure
and balances decision loss against estimation error in a 
minimax optimization problem. AMS couples the decision law with the estimation update:
it chooses an action distribution against a worst-case joint belief over
the environment and index, then uses the maximizing belief to update its
estimate. A saddle-point or first-order optimality condition bounds the
fixed-truth bracket. Under KL duality, EBO gives an equivalent formulation
of this optimization problem.

The precise bracket inequalities, the distinction between algorithmic update
beliefs and robust comparison beliefs, and the DEC connections
are collected in Appendices~\ref{app:algorithm-comparisons}
and~\ref{app:dec}.  The kernel application below uses the localized robust
AIR/AMS control only after a finite marginal has been fixed; it does not assume
that any of these methods discovers the localization scale automatically.

\section{Lower bounds: reference histories and quantile indices}
\label{sec:lower}

\subsection{Success under the ghost and true laws}

Fix an average-regret threshold $r\ge0$. Under the ghost law, draw an
index and a reference history independently, then test whether the
history has average indexed loss at most $r$. The success probability is
\begin{equation}
\label{eq:ghost-prob}
    p_r^\Alg(\mu,\chi)
    :=
    \Pp_{Y\sim\mu_\chi,\,H_T'\sim\bar\Pp_\mu^\Alg}
    \left(
        \bar L_\chi(Y,H_T')\le r
    \right),
\end{equation}
where $\mu_\chi$ is the marginal law of $Y$. Under the true law, the
success probability is
\begin{equation}
\label{eq:true-good}
    q_r^\Alg(\mu,\chi)
    :=
    \Pp_{Y,H_T}
    \left(
        \bar L_\chi(Y,H_T)\le r
    \right),
\end{equation}
where $(Y,H_T)$ are generated by $\Omega\sim\mu$ and $H_T\sim\Pp_\Omega^\Alg$.

The ghost probability still depends on the algorithm through the law
of $H_T'$. It does not require replacing this law by a static small-ball
bound.

\subsection{Reference-history quantile theorem}

Let
\[
    \kl(q,p)=q\log\frac qp+(1-q)\log\frac{1-q}{1-p}
\]
be binary KL divergence.  For $p\in[0,1]$ and $c\ge0$, define
\[
    \psi(p,c):=\sup\{q\in[0,1]:\kl(q,p)\le c\}.
\]

\begin{theorem}[Reference-history quantile indexed-information lower bound]
\label{thm:quantile-index}
Fix a prior $\mu$, index map $\chi$, threshold $r\ge0$, and algorithm
$\Alg$. Assume the exact posterior-indexed lift of
Definition~\ref{def:index-bellman}, that $O_0$ is deterministic or
independent of $Y$, and that $L_\chi(Y,H_T)\ge0$ almost surely under the
true mixture law. For the cumulative and average indexed information
in \eqref{eq:average-info},
\begin{equation}
\label{eq:binary-dp}
    \kl\bigl(q_r^\Alg(\mu,\chi),p_r^\Alg(\mu,\chi)\bigr)
    \le
    C_\chi^\Alg(\mu)
    =
    T\bar C_\chi^\Alg(\mu).
\end{equation}
Consequently,
\begin{equation}
\label{eq:quantile-exact}
    \E_{\Omega\sim\mu}\E_{\Pp_\Omega^\Alg}L_\Omega(H_T)
    \ge
    Tr\left[1-
    \psi\bigl(p_r^\Alg(\mu,\chi),T\bar C_\chi^\Alg(\mu)\bigr)
    \right].
\end{equation}
In particular, if $p_r^\Alg(\mu,\chi)\le1/2$ and
\begin{equation}
\label{eq:quantile-kl-condition}
    T\bar C_\chi^\Alg(\mu)
    \le
    \kl\!\left(\frac12,p_r^\Alg(\mu,\chi)\right),
\end{equation}
then
\begin{equation}
\label{eq:average-lower}
    \E_{\Omega\sim\mu}\E_{\Pp_\Omega^\Alg}\bar L_\Omega(H_T)
    \ge
    \frac r2,
    \qquad
    \E_{\Omega\sim\mu}\E_{\Pp_\Omega^\Alg}L_\Omega(H_T)
    \ge
    \frac{Tr}{2}.
\end{equation}
A convenient sufficient condition is
\begin{equation}
\label{eq:average-readable}
    T\bar C_\chi^\Alg(\mu)+\log2
    \le
    \frac12\log\frac1{p_r^\Alg(\mu,\chi)}.
\end{equation}
If $O_0$ may be informative about $Y$, the same conclusions hold after replacing every occurrence of $T\bar C_\chi^\Alg(\mu)$ in the information budget by
\[
    I_\mu(Y;O_0)+T\bar C_\chi^\Alg(\mu).
\]
\end{theorem}

\proofinappendix{app:proof-thm:quantile-index}

Appendix~\ref{app:lower-details} details the adaptation of interactive Fano and the posterior-reference chain rule.

\subsection{Bellman--Fano lower-bound method}
\label{subsec:bellman-lower}

Theorem~\ref{thm:quantile-index} concerns a fixed algorithm. To bound the
minimax loss, we bound the indexed information and the ghost success
probability uniformly over algorithms.

\paragraph{Exact information Bellman value and supersolutions.}
For a function \(F\) on states, define
\[
    (\mathsf C_tF)(s)
    :=
    \sup_{p\in\Delta(\calA_t(s))}
    \left\{
        \calI_\chi(s,p)
        +
        \E_{a\sim p,o\sim P_{s,p,a}}F(\tau(s,p,a,o))
    \right\}.
\]
Set
\[
    C_{T+1}^\star(s)=0,
    \qquad
    C_t^\star(s)=(\mathsf C_t C_{t+1}^\star)(s).
\]
The value \(C_1^\star(s_1)\) is the supremum of cumulative indexed
information over algorithms in the posterior-reference process.
A sequence \(\overline C_t\) is an information supersolution if
\[
    \overline C_{T+1}\ge0,
    \qquad
    \overline C_t(s)\ge(\mathsf C_t\overline C_{t+1})(s)
    \quad\forall t,s.
\]
For every algorithm starting from \(s_1\),
\begin{equation}
\label{eq:capacity-bound}
    C_\chi^\Alg(\mu)
    \le C_1^\star(s_1)
    \le \overline C_1(s_1).
\end{equation}
\paragraph{Exact ghost-good Bellman value and supersolutions.}
Let \(m\in\Delta(\calY)\) be the ghost index law, \(s\) the
posterior-reference state, and \(g:\calY\to\mathbb R\) the profile of
losses accumulated so far, divided by \(T\). For threshold \(r\), the
ghost success recursion is
\[
    \Gamma_{T+1}^{r,\star}(m,s,g)=m\{y:g(y)\le r\},
\]
\[
   \Gamma_t^{r,\star}(m,s,g)
    =
    \sup_{p\in\Delta(\calA_t(s))}
    \E_{a\sim p,o\sim P_{s,p,a}}
    \Gamma_{t+1}^{r,\star}\left(
        m,
        \tau(s,p,a,o),
        g+\frac1T\ell_\cdot^\chi(s,p,a)
    \right),
\]
where \(\ell_\cdot^\chi(s,p,a)\) denotes
\(y\mapsto\ell_y^\chi(s,p,a)\). A ghost-mass supersolution
\(\overline\Gamma_t^r\) dominates the terminal condition and, at each
earlier step, the Bellman right-hand side evaluated at
\(\overline\Gamma_{t+1}^r\). We take these supersolutions to be
$[0,1]$-valued: clipping a nonnegative supersolution at one preserves
the inequality by monotonicity and because the operator maps
$[0,1]$-valued functions into themselves. For every algorithm,
\begin{equation}
\label{eq:ghost-mass-bound}
    p_r^\Alg(\mu,\chi)
    \le
    \Gamma_1^{r,\star}(\mu_\chi,s_1,0)
    \le
    \overline\Gamma_1^r(\mu_\chi,s_1,0).
\end{equation}
The exact ghost entropy and its computable lower bound are therefore
\begin{equation}
\label{eq:ghost-entropy}
    \calE_{\star,1}^r(\mu,\chi)
    :=
    -\log \Gamma_1^{r,\star}(\mu_\chi,s_1,0),
    \qquad
    \underline{\calE}_1^r(\mu,\chi)
    :=
    -\log \overline\Gamma_1^r(\mu_\chi,s_1,0).
\end{equation}

\begin{theorem}[Bellman--Fano quantile-index lower bound]
\label{thm:bellman-lower}
Assume the retained state is an exact posterior-indexed lift under
$\mu$, the information terms are well defined and integrable, and the
ghost experiment uses a fixed jointly measurable version of
$\ell_y^\chi(s,p,a)$. Assume the cumulative indexed loss is nonnegative
as in Theorem~\ref{thm:quantile-index}. The following bounds are
conditional on a fixed initial packet and resulting state $s_1$; for
a random initial state, average the full conditional bound.

Let $\overline C_t$ and $\overline\Gamma_t^r$ be an information
supersolution and a $[0,1]$-valued ghost-mass supersolution. The exact
values $C^\star$ and $\Gamma^{r,\star}$ can also be used. The statement
is for all nonanticipating algorithms. For a Bellman-compatible
restricted class, restrict each local supremum in both the exact
recursions and the supersolution inequalities to the corresponding
feasible set. Then
\begin{equation}
\label{eq:bellman-lower-exact}
    \inf_{\Alg\in\mathfrak A_T^{\rm na}}
    \E_{\Omega\sim\mu}\E_{\Pp_\Omega^\Alg}L_\Omega(H_T)
    \ge
    Tr\left[1-
    \psi\bigl(e^{-\underline{\calE}_1^r(\mu,\chi)},\overline C_1(s_1)\bigr)
    \right].
\end{equation}
In particular, if
\begin{equation}
\label{eq:bellman-critical}
    \overline C_1(s_1)+\log2
    \le
    \frac12\underline{\calE}_1^r(\mu,\chi),
\end{equation}
then
\begin{equation}
\label{eq:bellman-lower-half}
    \inf_{\Alg\in\mathfrak A_T^{\rm na}}
    \E_{\Omega\sim\mu}\E_{\Pp_\Omega^\Alg}L_\Omega(H_T)
    \ge
    \frac{Tr}{2}.
\end{equation}
Consequently,
\[
    \mathfrak R_T^\star
    \ge
    \sup_{\mu,\chi,r}
    \left\{
        \frac{Tr}{2}:
        \overline C_1(s_1)+\log2
        \le
        \frac12\underline{\calE}_1^r(\mu,\chi)
    \right\}.
\]
\end{theorem}

\proofinappendix{app:proof-thm:bellman-lower}

The critical average-regret radius is
\begin{equation}
\label{eq:critical-radius}
    r_\star(\mu,\chi)
    :=
    \sup\left\{
        r:
        \overline C_1(s_1)+\log2
        \le
        \frac12\underline{\calE}_1^r(\mu,\chi)
    \right\}.
\end{equation}
This gives a lower bound of order $Tr_\star$. In a two-point Le Cam
argument the entropy scale is constant; Fano and local-prior-mass
arguments can instead use a scale of order dimension or effective
dimension.

\paragraph{Relation to DEC.}
The Bellman comparison checks three conditions separately:
hard-prior admissibility, the upper value at the lower entropy scale,
and the growth of that value. DEC gives upper and lower guarantees
through coefficients that need not retain the full reference
trajectory. Appendix~\ref{app:dec} discusses when these coefficients
preserve the rate.

\subsection{The information-complexity sandwich}
\label{sec:info-sandwich}

Both bounds use a logarithmic information budget. For a comparison
class \(\Omega_0\), let
\[
    L_0
    =
    \sup_{\omega\in\Omega_0}
    \log\frac1{q_1(\chi(\omega))}
\]
be the coordinate code length paid by the upper Bellman telescope.  For a hard
prior \(\mu\) and average-loss radius \(r\), the ghost entropy is
\[
    E_{\mu,r}^{\star}
    =
    \log\frac1{p_{\mu,r}^{\star}}
\]
where \(p_{\mu,r}^{\star}\) is the supremum of the ghost success
probability over algorithms; see \eqref{eq:main-uniform-ghost-entropy}.
Under the hypotheses of Theorem~\ref{thm:frequentist-info-risk-upper},
the upper construction gives
\(\mathfrak R_T^\star(\Omega_0)\le\mathsf U_T(L_0)\), with
\(\mathsf U_T\) defined in \eqref{eq:upper-budget-value}. If the hard prior
is Bellman--Fano admissible at radius \(r\), the lower construction gives
\(\mathfrak R_T^\star(\Omega_0)\ge Tr/2\).

Once the upper and lower bounds are established, Theorem~\ref{thm:interactive-info-risk-sandwich} uses two additional conditions to control the gap between them. First, calibration at the lower entropy scale bounds the upper
value at the ghost entropy: \(\mathsf U_T(E_{\mu,r}^{\star})\lesssim Tr\).
Second, if the actual code length \(L_0\) is larger, a growth bound for
\(\mathsf U_T\) controls the cost of increasing the budget from
\(E_{\mu,r}^{\star}\) to \(L_0\). Comparing the two budgets alone does not
give matching regret rates. Appendix~\ref{app:sandwich-details} states the
conditions and Appendix~\ref{app:sandwich-proof} gives the proofs.

\section{Applications: kernel bandits and information-complexity sandwiches}
\label{sec:applications}

For kernel bandits, we compare a finite-marginal AIR Bellman policy with
two globally calibrated GP--UCB rules on the full RKHS ball, including the
original RKHS GP--UCB schedule specialized to the unit ball. The same
construction gives a LinUCB separation on a finite action set with unequal
feature norms and complements results for Mat\'ern kernels.

\subsection{Kernel bandits: Bellman decompositions and algorithms}
\label{sec:kernel}

We first restrict the actions to a finite set. Only the reward vector on
this set enters the observations and decisions, even if the unknown reward
function lies in an infinite-dimensional RKHS. This vector suffices for the
GP update and gives two information indices: the vector itself for MAIR
and its optimal action for AIR. Appendix~\ref{app:kernel-details} gives the
full recursion and further comparisons.

\paragraph{Problem and finite active marginal.}
Let $\calH_k$ be a reproducing kernel Hilbert space (RKHS) on a possibly infinite domain $\calX_{\rm all}$, with kernel $k$ satisfying $k(x,x)\le \kappa^2$.  The learner is evaluated on a finite active set
\[
    \calX=\{x_1,\ldots,x_n\}\subset\calX_{\rm all}.
\]
The unknown reward function $f^\star\in\calH_k$ satisfies $\|f^\star\|_{\calH_k}\le B$, and at round $t$ the learner chooses $X_t\in\calX$ and observes
\begin{equation}
\label{eq:kernel-observation-new}
    Y_t=f^\star(X_t)+\varepsilon_t,
    \qquad
    \E[\varepsilon_t\mid H_{t-1},X_t]=0,
\end{equation}
where the noise is conditionally $R$-sub-Gaussian.  For the Gaussian calculations below we use the algorithmic reference model with variance parameter $\lambda>0$; the frequentist confidence event is obtained by the usual self-normalized argument.

For an action set $\calX$ and reward function $f$, define cumulative regret by
\begin{equation}
\label{eq:kernel-regret-definition}
    \Reg_T(f;\Alg)
    :=
    \sum_{t=1}^T\left\{\max_{x\in\calX}f(x)-f(X_t)\right\}.
\end{equation}
We write $\Reg_T(f)$ when the algorithm is clear and $\Reg_T(\Alg)$ when the
truth is indicated by the expectation.

The posterior and AIR/MAIR chain-rule identities in this subsection use
the well-specified Gaussian observation model
$\varepsilon_t\stackrel{\iid}{\sim}N(0,\lambda)$.  The pathwise GP--UCB
confidence argument also holds for conditionally sub-Gaussian noise, but under
non-Gaussian noise a Gaussian pseudo-posterior is only an algorithmic state and
does not satisfy the Bayesian chain rule without a separate transfer argument.

Let $F=(f(x_1),\ldots,f(x_n))\in\R^n$ be the active reward vector and let $K_{\calX}$ be the $n\times n$ kernel matrix.  The GP reference prior is
\[
    F\sim N(0,K_{\calX}),
    \qquad
    Y_t\mid F,X_t=x_i\sim N(F_i,\lambda).
\]
After history $H_{t-1}$, the algorithmic posterior on $F$ is Gaussian,
\[
    F\mid H_{t-1}\sim N(m_t,\Sigma_t).
\]
For $x_i\in\calX$, write $m_t(x_i)=e_i^\top m_t$ and $s_t^2(x_i)=e_i^\top\Sigma_t e_i$.  If $X_t=x_i$, the exact posterior update is
\begin{align}
\label{eq:gp-posterior-update-mean}
    m_{t+1}
    &=m_t+
      \frac{\Sigma_t e_i}{\lambda+s_t^2(x_i)}
      \{Y_t-m_t(x_i)\},\\
\label{eq:gp-posterior-update-cov}
    \Sigma_{t+1}
    &=\Sigma_t-
      \frac{\Sigma_t e_i e_i^\top\Sigma_t}{\lambda+s_t^2(x_i)}.
\end{align}
\paragraph{Model-index information and action-index information.}
For the model-index specialization, take the index to be the active reward vector $Y_{\MAIR}=F$.  The one-step GP/MAIR information at state $H_{t-1}$ is
\begin{equation}
\label{eq:gp-mair-info}
    \calI_t^{\rm GP}(p)
    :=
    \E_{x\sim p}\frac12\log\left(1+\frac{s_t^2(x)}{\lambda}\right).
\end{equation}
Along a realized action sequence,
\begin{equation}
\label{eq:realized-gp-info}
    C_T^{\rm GP}
    :=
    \sum_{t=1}^T\calI_t^{\rm GP}(\delta_{X_t})
    =
    \frac12\log\det\!\left(I+\lambda^{-1}K_{X_{1:T},X_{1:T}}\right).
\end{equation}
The Gram matrix includes repeated actions. This realized gain can be much
smaller than its maximum over length-$T$ designs.

For the action-index specialization, assume ties are broken by a fixed rule and set
\[
    A^\star(F)\in\argmax_{x\in\calX}F_x,
    \qquad
    q_t(a):=\Pp(A^\star=a\mid H_{t-1}).
\]
Let $\nu_t^a$ be the exact conditional law of $F$ given $H_{t-1}$ and $A^\star=a$.  It is generally a truncated Gaussian on the cone where $a$ is optimal.  Define the conditional predictive law
\[
    P_{t,a,x}(\cdot)
    :=
    \int N(F_x,\lambda)(\cdot)\,\nu_t^a(dF),
\]
for the observation at action $x$ given the event $A^\star=a$.  The marginal predictive is
\[
    P_{t,x}=\sum_{a\in\calX}q_t(a)P_{t,a,x}
           =N(m_t(x),\lambda+s_t^2(x)).
\]
The AIR information increment is
\begin{equation}
\label{eq:kernel-air-info-new}
    \calI_t^{\AIR}(p)
    :=
    \E_{x\sim p}
    \sum_{a\in\calX}q_t(a)
    \KL(P_{t,a,x}\|P_{t,x}).
\end{equation}
The conditional AIR regret is
\begin{equation}
\label{eq:kernel-air-regret-new}
    \Delta_t^{\AIR}(p)
    :=
    \sum_{a\in\calX}q_t(a)
    \E_{F\sim\nu_t^a}\E_{x\sim p}[F_a-F_x].
\end{equation}
The exact indexed chain rule gives
\[
    \sum_{t=1}^T\E\calI_t^{\AIR}(p_t)
    =I(A^\star;H_T)
    \le H(A^\star)\le\log n.
\]
AIR thus charges information about the optimal action, rather than the
whole reward vector. A bound on the AIR bracket is still needed to convert
that information into a regret guarantee.

\paragraph{Algorithm 1: finite-marginal indexed Bellman DP.}

The action-index program uses the Bellman state $(m_t,\Sigma_t,q_t)$.
In the Bayesian recursion, $q_t$ is the posterior law of the optimal active
action; in the frequentist recursion, it is the reference law supplied by
the chosen update. Appendix~\ref{subsec:two_more_algorithms} describes the
robust AIR/AMS control and its EBO interpretation.

The next theorem gives a uniform frequentist rate for the action-index branch.  For a finite retained action set $\mathcal R$, write
\[
    \Reg_T^{\rm all}(f)
    :=
    \sum_{t=1}^T
    \left\{\sup_{x\in\calX_{\rm all}}f(x)-f(X_t)\right\}
\]
for regret against the full decision domain, and define
\[
    \varepsilon_{\mathcal R}
    :=
    \sup_{f\in\calF}
    \left\{
        \sup_{x\in\calX_{\rm all}}f(x)
        -\max_{x\in\mathcal R}f(x)
    \right\}.
\]

\begin{theorem}[Finite-marginal action-index AIR Bellman bound]
\label{thm:finite-marginal-air-upper}
Let $|\mathcal R|=K\ge2$, suppose the retained reward-vector class
\[
    \Theta_{\mathcal R}
    :=
    \{(f(x))_{x\in\mathcal R}:f\in\calF\}
\]
is compact and contained in $[-L,L]^K$, and let the observations have independent $N(0,\lambda)$ noise with $\lambda>0$.  For every $\gamma>0$, the finite-marginal action-index AIR Bellman program admits a policy, confined to $\mathcal R$ and initialized with the uniform law on the attainable optimal-action indices, such that
\begin{equation}
\label{eq:finite-marginal-air-upper}
    \sup_{f\in\calF}\E_f\Reg_T^{\rm all}
    \le
    T\varepsilon_{\mathcal R}
    +\frac{(\lambda+L^2)KT}{2\gamma}
    +\gamma\log K.
\end{equation}
Consequently, the choice
\[
    \gamma
    =
    \sqrt{\frac{(\lambda+L^2)KT}{2\log K}}
\]
gives
\begin{equation}
\label{eq:finite-marginal-air-optimized}
    \sup_{f\in\calF}\E_f\Reg_T^{\rm all}
    \le
    T\varepsilon_{\mathcal R}
    +\sqrt{2(\lambda+L^2)KT\log K}.
\end{equation}
The robust AIR/AMS saddle of \citet{xu2025bayesian}, specialized to the
finite marginal, gives the control used in a Bellman supersolution and
attains the same bound. The EBO interpretation follows from the minimax
and KL-duality principle of \citet{lattimore2021mirror}; the two
implementations need not coincide.
\end{theorem}

\proofinappendix{app:proof-thm:finite-marginal-air-upper}

Since $A^\star$ is a function of $F$, data processing bounds its
posterior-averaged information by that of $F$. A frequentist MAIR bound for
the continuous index $F$, however, must retain the terminal log-density
term or use neighborhoods, a finite cover, or a PAC--Bayes comparison.

\paragraph{Algorithm 2: GP--UCB.}
GP--UCB uses the Gaussian marginal posterior $(m_t,\Sigma_t)$. The posterior
standard deviation $s_t(x)=(e_x^\top\Sigma_t e_x)^{1/2}$ is the uncertainty
scale $\sigma_t(x)$ in Lemma~\ref{lem:ucb-bracket}; $e_x$ is the coordinate
vector of $x$. The confidence half-width is
\[
    w_t(x):=\beta_t s_t(x),
\]
where $\beta_t$ is the confidence multiplier.  The algorithm chooses
\begin{equation}
\label{eq:gpucb-rule-new}
    X_t\in\argmax_{x\in\calX}\{m_t(x)+\beta_t s_t(x)\}.
\end{equation}
The frequentist confidence event is
\begin{equation}
\label{eq:gp-confidence-event-new}
    \mathcal E_{\rm GP}
    :=
    \left\{
        |f^\star(x)-m_t(x)|\le w_t(x)=\beta_t s_t(x)
        \quad\text{for all }t\le T,
        x\in\calX
    \right\}.
\end{equation}
For finite $\calX$, standard RKHS self-normalized concentration gives $\Pp(\mathcal E_{\rm GP})\ge1-\delta$ for the usual choice of $\beta_t$.  For example, up to universal constants one may take
\[
    \beta_t
    \asymp
    B+\frac{R}{\sqrt{\lambda}}\sqrt{\Gamma_{t-1}^{\rm GP}+\log(1/\delta)},
    \qquad
    \Gamma_m^{\rm GP}:=
    \sup_{x_{1:m}\in\calX^m}
    \frac12\log\det\!\left(I+\lambda^{-1}K_{x_{1:m},x_{1:m}}\right),
\]
with the precise form depending on the normalization of $k$ and the ridge parameter.  Here $K_{x_{1:m},x_{1:m}}$ is the Gram matrix of the possibly repeated design sequence $(x_1,\ldots,x_m)$.  The quantity $\Gamma_m^{\rm GP}$ is a worst-case calibration device for the confidence event; it is different from the realized information gain $C_T^{\rm GP}$ paid in the regret bound.  The next theorem isolates the deterministic part of the GP--UCB proof.

\begin{theorem}[GP--UCB realized-information bound]
\label{thm:gpucb-realized}
Assume that $(\beta_t)_{t\le T}$ is a deterministic nondecreasing sequence.  On the event $\mathcal E_{\rm GP}$, the GP--UCB rule \eqref{eq:gpucb-rule-new} satisfies
\begin{equation}
\label{eq:gpucb-pathwise-bound-new}
    \Reg_T(f^\star)
    \le
    2\beta_T\sum_{t=1}^T s_t(X_t)
    \le
    2\beta_T
    \sqrt{T\,c_{\kappa,\lambda}\,C_T^{\rm GP}},
\end{equation}
where
\[
    c_{\kappa,\lambda}
    :=
    \frac{2\kappa^2}{\log(1+\kappa^2/\lambda)}
\]
and $C_T^{\rm GP}$ is the realized information gain in \eqref{eq:realized-gp-info}.  Consequently, if rewards are bounded in $[0,1]$ and $\Pp(\mathcal E_{\rm GP})\ge1-\delta$, then
\[
    \E\Reg_T(f^\star)
    \le
    2\beta_T\sqrt{T\,c_{\kappa,\lambda}\,\E C_T^{\rm GP}}+T\delta.
\]
\end{theorem}

\proofinappendix{app:proof-thm:gpucb-realized}

\paragraph{Calibration and realized information.}
The theorem uses the calibration and optimism of Lemma~\ref{lem:ucb-bracket}
together with the log-determinant identity
\[
    \sum_t \frac12\log\left(1+s_t^2(X_t)/\lambda\right)=C_T^{\rm GP}.
\]
The log-determinant argument controls the sum of $s_t^2(X_t)$, whereas
the confidence half-width $\beta_t s_t(X_t)$ controls one-round regret.
Thus the regret bound uses the realized $C_T^{\rm GP}$ even when the
acquisition rule uses a worst-case information gain to set its multiplier.
The log-determinant increments are posterior-averaged information terms.
They are not, in general, the fixed-truth MAIR gains, which are KL
divergences from the true predictive law to the posterior predictive law.

\paragraph{Kernel and linear formulations.}
Under the canonical feature map \(\varphi(x)=k(x,\cdot)\),
\[
    f(x)=\langle f,\varphi(x)\rangle_{\calH_k}.
\]
The zero-mean GP posterior mean and variance coincide with the Hilbert-space
ridge estimate and ellipsoidal width.  Hence GP--UCB and the corresponding
ridge LinUCB rule produce the same actions, observations, and regret
under common noise and compatible tie-breaking. For finite \(\calX\), the
Euclidean dimension is
\(\operatorname{rank}(K_{\calX})\).  This is the rank of the full action
geometry, not the rank of the design observed during \(T\) rounds: unsampled
directions remain in the UCB maximization.  Proposition~\ref{prop:kernel-linear-equivalence}
gives the precise operator identity and its infinite-dimensional qualification.

\subsection{A full-RKHS separation for maximal-information GP--UCB}
\label{sec:kernel-separation}
The construction starts from the weighted-basis hub--cloud geometry of
\citet{xu2026separation}. If $G_x=W(\psi(x))$ for an isonormal Gaussian
process $W$, then
\[
    k(x,x')=\mathbb E[G_xG_{x'}]
            =\langle\psi(x),\psi(x')\rangle.
\]
The vectors $\psi(x)$ serve as both kernel features and linear-bandit
actions. We reverse the amplitude ordering: the cloud rewards are small
enough to ignore at little cost, but the number of cloud actions still
inflates the global exploration multiplier. The regret comparison is
over the full unit parameter ball. The finite blocks require only
finite-dimensional features (Appendix~\ref{sec:kernel-separation-details});
infinite dimension enters when we combine the blocks into one fixed
kernel.

For a bounded kernel $k$ and unit Gaussian observation noise, define the maximal information gain
\begin{equation}
\label{eq:fixed-kernel-max-info}
    \overline\Gamma_s(k)
    :=
    \sup_{x_{1:s}\in\calX^s}
    \frac12\log\det\!\left(I+K_{x_{1:s},x_{1:s}}\right),
    \qquad \overline\Gamma_0(k):=0.
\end{equation}
Let $m_t,s_t$ be the posterior mean and standard deviation under the
zero-mean GP with covariance $k$ and unit ridge parameter. The anytime rule is
\begin{equation}
\label{eq:fixed-kernel-anytime-rule}
    \overline\beta_t
    :=1+\sqrt{2\left\{\overline\Gamma_{t-1}(k)+2\log(t\vee2)\right\}},
    \qquad
    X_t\in\argmax_{x\in\calX}
    \left\{m_t(x)+\overline\beta_t s_t(x)\right\},
\end{equation}
with a fixed tie-breaking rule.

This multiplier has the maximal-information form of the anytime calibration
of \citet{chowdhury2017kernel}. The original RKHS schedule of
\citet{srinivas2010gaussian} instead takes information gain over sets of
distinct points. For fixed \(\delta\in(0,1)\), its unit-ball specialization uses
\[
    \Gamma_t^{\rm SKKS}(k)
    :=
    \sup_{\substack{A\subset\calX\\ |A|=t}}
    \frac12\log\det(I+K_A),
\]
and
\begin{equation}
\label{eq:original-rkhs-gpucb-rule}
    b_t^{\rm SKKS}
    :=
    \left[
       2+300\,\Gamma_t^{\rm SKKS}(k)
       \log^3\!\left(\frac{t}{\delta}\right)
    \right]^{1/2},
    \qquad
    X_t\in\argmax_{x\in\calX}
    \{m_t(x)+b_t^{\rm SKKS}s_t(x)\},
\end{equation}
with the same fixed tie-breaking convention. Srinivas et al.\ write the
multiplier as \(\sqrt{\beta_t}\) and set
\(\beta_t=2B+300\gamma_t\log^3(t/\delta)\) for \(\|f\|_k^2\le B\).
Here \(B=1\), \(\gamma_t=\Gamma_t^{\rm SKKS}(k)\), and the noise variance is
one. We use their acquisition rule under independent Gaussian noise,
not their confidence theorem for bounded martingale noise.

The two rules differ only in the multiplier. The finite-horizon
rule~\eqref{eq:max-info-gpucb} in Appendix~\ref{sec:kernel-separation-details}
uses the anytime rule's information-gain convention but replaces the
time-dependent logarithm by its horizon-dependent counterpart.

\begin{theorem}[Fixed-kernel separation: GP--UCB versus AIR and minimax]
\label{thm:fixed-kernel-gpucb}
For every $0<\alpha<1/4$, there exist a compact countable action space $\calX$, a bounded continuous kernel $k$ on $\calX$, its fixed unit RKHS ball $\calF$, a sequence $T_m\to\infty$, and a single epochwise finite-marginal action-index AIR Bellman policy such that, under observations $Y_t=f(X_t)+\xi_t$ with $\xi_t\stackrel{\iid}{\sim}N(0,1)$,
\begin{equation}
\label{eq:fixed-kernel-minimax}
    \inf_{\Alg}\sup_{f\in\calF}\E_f\Reg_{T_m}(\Alg)
    =\Theta(T_m^{1-\alpha}),
\end{equation}
and the AIR Bellman policy, fixed in advance and independent of the evaluation horizon, satisfies
\begin{equation}
\label{eq:fixed-kernel-air}
    \sup_{f\in\calF}\E_f\Reg_{T_m}(\mathrm{AIR\mbox{--}Bellman})
    =O(T_m^{1-\alpha}).
\end{equation}
In contrast, the anytime GP--UCB rule \eqref{eq:fixed-kernel-anytime-rule} and,
for every fixed \(\delta\in(0,1)\), the original-schedule rule
\eqref{eq:original-rkhs-gpucb-rule} satisfy
\begin{equation}
\label{eq:fixed-kernel-gpucb}
    \sup_{f\in\calF}\E_f\Reg_{T_m}(\mathsf R)
    =\Omega(T_m),
    \qquad
   \mathsf R\in
    \{\mathrm{GP\mbox{--}UCB}_{\rm any},
      \mathrm{GP\mbox{--}UCB}_{\rm SKKS}\}.
\end{equation}
For each of these GP--UCB rules, both the minimax ratio and the worst-case
regret ratio relative to the AIR Bellman policy are \(\Omega(T_m^\alpha)\)
along the sequence \((T_m)\).
\end{theorem}

\proofinappendix{app:proof-thm:fixed-kernel-gpucb}
\subsection{Finite-scale illustration of the hub--cloud mechanism}
\label{subsec:numerical-illustration}
We simulate the small-cloud instance of Appendix~\ref{sec:kernel-separation-details}
at several horizons, taking $\alpha=1/5$, $a=T^{-1/5}$, $N=4T$, and hub gap
$\Delta=1/4$ at each horizon.

\begin{figure}[H]
    \centering
    \includegraphics[width=0.98\linewidth]{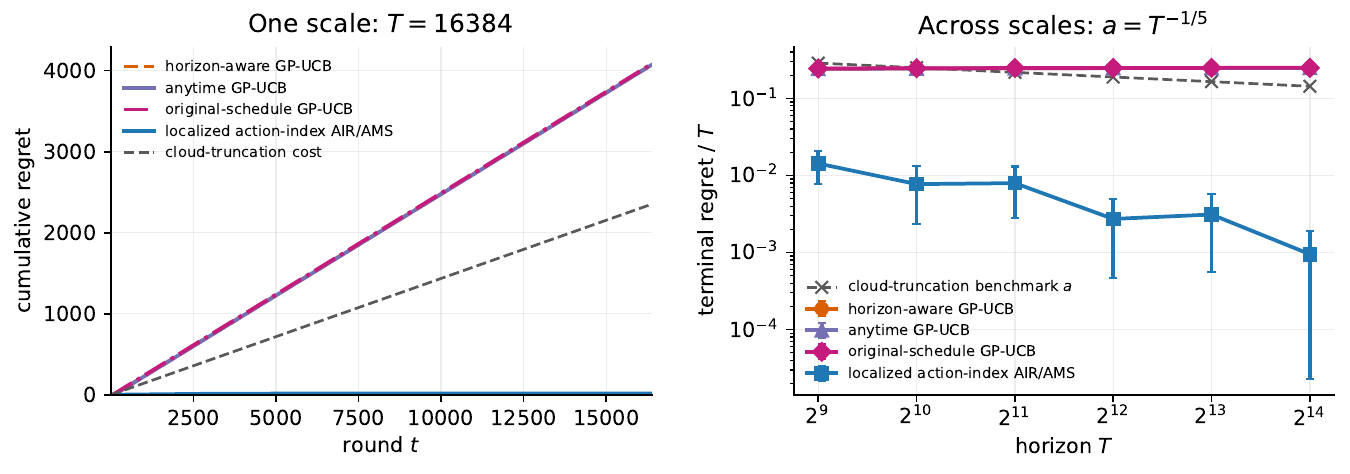}
    \caption{Finite-scale hub--cloud experiment. Algorithmic curves show
    means with bands or bars of two standard errors over $200$ runs under
    the same positive hub truth. Left:
    cumulative pseudoregret at $T=16384$ for three GP--UCB calibrations and the
    localized AIR/AMS control; the dashed curve shows the deterministic
    cost of ignoring the cloud under a one-hot cloud truth. Right: terminal
    regret divided by $T$, with normalized truncation benchmark $a=(aT)/T$.}
    \label{fig:kernel-air-simulation}
\end{figure}

We compare three GP--UCB calibrations---horizon-aware
\eqref{eq:max-info-gpucb}, anytime \eqref{eq:fixed-kernel-anytime-rule}, and
original-schedule \eqref{eq:original-rkhs-gpucb-rule} with $\delta=0.1$---with
the two-model AIR/AMS saddle on $\{x_0,x_h\}$. By
Proposition~\ref{prop:kernel-linear-equivalence}, the GP--UCB curves also
describe the corresponding LinUCB rules on these same blocks.

At \(T=16384\), the \(200\)-run mean regrets under the common positive hub
truth are \(4074.8\), \(4072.5\), and \(4083.8\) for the horizon-aware,
anytime, and original-schedule GP--UCB rules, respectively, compared with
\(15.6\) for AIR/AMS. The simulated AIR/AMS rule knows the magnitude
\(\Delta\), but not its sign. The dashed curve is a separate comparison:
it gives the one-hot-cloud truncation cost, \(aT\approx2353\) at
\(T=16384\). At \(T=512\), the normalized truncation cost exceeds the
hub-truth GP--UCB regret. These are different truths, so their comparison
does not estimate a minimax regret ratio. The kernel changes with \(T\),
so the experiment illustrates the finite-block construction rather than
the fixed-kernel separation in Theorem~\ref{thm:fixed-kernel-gpucb}.
Appendix~\ref{app:numerics} gives the protocol, the result for the negative
hub truth, and the bracket checks.

\subsection{Finite-dimensional consequence}

For every sufficiently large \(T\), Corollary~\ref{cor:finite-linucb-separation}
uses a smaller block of \(T+1\) actions in \(\mathbb R^T\) on which the specified
maximal-information-calibrated LinUCB rule has worst-case regret
\(\Omega(T)\), while the unrestricted minimax regret is
\(\Theta(T^{1-\alpha})\).  The ratio is therefore
\(\Omega(T^\alpha)\).  The action set depends on \(T\) and has heterogeneous
norms. A comparable separation for \(d=o(T)\) (in particular, fixed \(d\))
or for more homogeneous action sets remains open.
Remark~\ref{rem:finite-linucb-scope} compares the result with earlier
optimism lower bounds, OFUL, and the fixed-confidence identification
result of \citet{maiti2026power}.

\begin{remark}[Bellman and hub--cloud interpretation]
The weighted-basis hub--cloud construction of
\citet{xu2026separation} has cloud scale larger than hub scale.  When that
geometry is reinterpreted as a kernel action set, it can draw GP--UCB toward
the cloud, but the same large cloud scale also makes the full RKHS ball
minimax-hard.  Theorem~\ref{thm:small-cloud-gpucb} reverses the amplitude
comparison while retaining the multiplicity comparison.  The pointwise RKHS
envelope $|f(c_j)|\le a$ makes the entire cloud worth at most $aT$ to an
unrestricted algorithm.  Nevertheless, its maximal information gain is of
order $Ta^2$, so the global confidence multiplier is of order $a\sqrt T$ and
the acquisition width of each unqueried cloud point is of order
$a^2\sqrt T$.  When $a=T^{-\alpha}$ and $\alpha<1/4$, that width diverges even
though the decision value of the cloud vanishes.

The AIR policy truncates the cloud on a predetermined epoch schedule. The
pointwise envelope bounds the loss from truncation by $a$ per round, and
the AIR bound uses only the optimal-action index on the retained set.
GP--UCB instead calibrates its multiplier over the full action set.
Maximal information gain therefore changes which actions GP--UCB takes,
not just the bound on its regret.
\end{remark}

\subsection{Resolved and open questions}

Does the gap between minimax rates and known GP--UCB guarantees reflect
the analysis or the acquisition rule
\citep{vakili2021open,whitehouse2023sublinear,wang2026suboptimality}?
Theorem~\ref{thm:fixed-kernel-gpucb} answers this for two specified global
calibrations: both have linear regret along its horizon sequence \((T_m)\)
under one fixed hub truth, although the minimax regret on the same RKHS
ball is sublinear. The theorem does not cover arbitrary localized,
predictable, or data-dependent tuning. \citet{lattimore2023adaptive}
resolves the continuous-time analogue of the adaptive kernel-regression
confidence question of \citet{vakili2021open}, not the discrete-time
regret question considered here. With Gaussian noise,
\citet{takeno2026thinning} obtain near-optimal cumulative regret using
data thinning and phased elimination, a different algorithm from the
two GP--UCB rules in the theorem.

For standard Mat\'ern kernels, \citet{wang2026suboptimality} show that
polynomial effective optimism gives a polynomial minimax gap under their
schedule and simultaneous uniform-confidence assumptions; see
Appendix~\ref{app:matern-scope}. Their hard function may depend on the
horizon. It remains open whether the polynomial gap holds for their
prespecified schedules without the uniform-confidence assumption, and
whether GP--UCB with its usual posterior update and acquisition rule
can attain the Mat\'ern minimax rate up to logarithms through
polylogarithmic effective optimism or localized or data-dependent tuning.
A fixed-truth, linear-regret Mat\'ern separation for either schedule in
Theorem~\ref{thm:fixed-kernel-gpucb} is also open.

\subsection{Further matching examples}
For Gaussian multi-armed bandits, the finite-index Bellman--Fano bound
recovers the minimax rate \(\Theta(\sqrt{KT})\). For hypercube linear
bandits, it gives a lower bound of order \(d\sqrt T\), matching standard
upper bounds up to logarithms. Both comparisons use separately established
upper bounds and do not verify the upper Bellman calibration in
Theorem~\ref{thm:interactive-info-risk-sandwich}.
Appendices~\ref{app:mab} and~\ref{app:linear} give the calculations.

\section{Conclusion}
\label{sec:conclusion}

The Bellman state retains what is needed to describe the controlled
experiment; the index identifies what the learner needs to distinguish.
The fixed-truth logarithmic potential and Bellman--Fano ghost entropy give
upper and lower bounds on this representation. When they match under
Theorem~\ref{thm:interactive-info-risk-sandwich}, the result is a sequential
counterpart of matching information and local entropy in a fixed experiment.

The kernel example shows why the distinction matters for algorithm design.
Low-value cloud directions are cheap to ignore but large enough in number
to drive the exploration bonuses of two globally calibrated GP--UCB rules.
The same geometry gives a finite-horizon LinUCB separation. The separation
on one fixed action set uses an infinite-dimensional feature space; the
finite-dimensional set depends on \(T\). Whether a similar gap persists
for \(d=o(T)\), or whether data-dependent calibration can close it, remains
open.

\clearpage
\appendix
\section{Technical form of the information-complexity sandwich}
\label{app:sandwich-details}

The two bounds in Section~\ref{sec:info-sandwich} rest on separate hypotheses: those of Theorem~\ref{thm:frequentist-info-risk-upper} for the upper bound and Bellman--Fano admissibility for the lower bound. Two further conditions control their gap: calibration at the lower entropy scale and a bound on how the upper value grows with its information budget. The finite-index comparison separately bounds the ratio of the upper code length to the lower ghost entropy. 

The conditions are intuitive, and the proofs in Appendix~\ref{app:sandwich-proof} are elementary and together fit on a single page. The Gaussian multi-armed and linear bandit examples illustrate the Bellman--Fano lower-bound method by comparing its lower bounds with separately established algorithmic upper bounds. The kernel bandit example illustrates the Bellman upper-bound method: the AIR/AMS upper bound matches the rate of a separate minimax lower bound. None of these examples verifies both calibration and growth for the exact Bellman program.

\subsection{Code length, hard priors, and ghost entropy}

For a problem class \(\Omega_0\subseteq\Omega\), an index
\(\chi:\Omega\to\calY\), and an initial reference marginal \(q_1\), the
fixed-truth upper bound pays the worst-case initial code length
\begin{equation}
\label{eq:upper-code-length}
    L_0(\Omega_0;q_1,\chi)
    :=
    \sup_{\omega\in\Omega_0}
    \log\frac1{q_1(\chi(\omega))},
\end{equation}
with the convention that the value is infinite if \(q_1(\chi(\omega))=0\) for some \(\omega\in\Omega_0\).

The Bellman values below are conditional on a fixed initial packet and
resulting state \(s_1\). If \(S_1\) is random, condition on \(O_0\) and average
the full conditional bounds. To compute an unconditional ghost entropy,
average the ghost masses before taking \(-\log\).

For a hard prior \(\mu\) supported on \(\Omega_0\), write
\(q_{1,\mu}:=\chi_\#\mu\).  At an average-loss radius \(r\), the ghost-good probability of an algorithm \(\Alg\) is
\begin{equation}
\label{eq:main-ghost-entropy}
    p_{\mu,r}^{\Alg}
    :=p_r^{\Alg}(\mu,\chi)
    =
    \Pp_{Y\sim {q_{1,\mu}},\,H_T'\sim\bar\Pp_\mu^{\Alg}}
    \{\bar L_\chi(Y,H_T')\le r\},
\end{equation}
where \(Y\) is independent of the reference history \(H_T'\). Taking the
supremum over algorithms gives the algorithm-uniform ghost mass and its entropy:
\begin{equation}
\label{eq:main-uniform-ghost-entropy}
    p_{\mu,r}^{\star}
    :=\sup_{\Alg\in\mathfrak A_T^{\rm na}}p_{\mu,r}^{\Alg}
    =\Gamma_1^{r,\star}({q_{1,\mu}},s_1,0),
    \qquad
    E_{\mu,r}^{\star}:=\log\frac1{p_{\mu,r}^{\star}}.
\end{equation}
A ghost-mass supersolution gives
\(p_{\mu,r}^{\star}\le\overline\Gamma_1^r({q_{1,\mu}},s_1,0)\) and hence the
computable entropy lower bound
\(-\log\overline\Gamma_1^r({q_{1,\mu}},s_1,0)\le E_{\mu,r}^{\star}\).

\begin{definition}[Bellman--Fano admissible hard prior]
\label{def:bf-admissible-prior}
Fix \(r>0\). A prior \(\mu\) supported on \(\Omega_0\) is
\emph{Bellman--Fano admissible at radius \(r\)} if the lower Bellman recursions
or their supersolutions give a nonnegative upper bound \(C_\mu\) on indexed
information and a lower bound
\(\underline E_{\mu,r}\le E_{\mu,r}^{\star}=\log(1/p_{\mu,r}^{\star})\)
on algorithm-uniform ghost entropy such that
\begin{equation}
\label{eq:bf-admissible-condition}
	    C_{\mu}+\log 2
	    \le
	    \frac12 \underline E_{\mu,r}.
\end{equation}
The exact recursions give \(C_{\mu}=C_1^\star(s_1)\) and
\(\underline E_{\mu,r}=E_{\mu,r}^{\star}=-\log\Gamma_1^{r,\star}({q_{1,\mu}},s_1,0)\),
where \(C^\star\) and \(\Gamma^{r,\star}\) are the exact information-capacity
and ghost-good Bellman values from Section~\ref{subsec:bellman-lower}.
Alternatively, supersolutions give
\(C_{\mu}=\overline C_1(s_1)\) and
\(\underline E_{\mu,r}=-\log\overline\Gamma_1^r({q_{1,\mu}},s_1,0)\).
\end{definition}

A bound \(p_{\mu,r}^{\star}\le \varrho<1\) ensures that the ghost entropy is positive. Bellman--Fano admissibility gives the stronger guarantee that the Fano factor is at least \(1/2\).

\begin{lemma}[Admissibility implies nondegenerate ghost entropy]
\label{lem:admissible-implies-nondegenerate}
If \(\mu\) is Bellman--Fano admissible at radius \(r\), then
\[
    E_{\mu,r}^{\star}\ge\underline E_{\mu,r}\ge 2\log2,
    \qquad
    p_{\mu,r}^{\star}\le \frac14.
\]
Conversely, \(p_{\mu,r}^{\star}\le\varrho<1\) gives
\(E_{\mu,r}^{\star}\ge\log(1/\varrho)\), but admissibility also requires the
information bound in \eqref{eq:bf-admissible-condition}.
\end{lemma}

\proofinappendix{app:proof-lem:admissible-implies-nondegenerate}

\begin{proposition}[Finite-index entropy comparison with an explicit hard prior]
\label{prop:entropy-gap}
Let \(\calY_h\subseteq\chi(\Omega_0)\) be finite with \(|\calY_h|=M\ge2\).  Choose one representative environment \(\omega^y\in\Omega_0\) for each \(y\in\calY_h\) such that \(\chi(\omega^y)=y\), and let \(\mu_h\) be the uniform prior on \(\{\omega^y:y\in\calY_h\}\).  Then \(q_{1,\mu_h}:=\chi_\#\mu_h\) is uniform on \(\calY_h\), and
\[
    L_0(\{\omega^y:y\in\calY_h\};{q_{1,\mu_h}},\chi)=\log M.
\]
If \(\mu_h\) is Bellman--Fano admissible at radius \(r\), then the entropy mismatch between the upper coordinate code length and the lower ghost entropy satisfies
\begin{equation}
\label{eq:log-gap-from-admissibility}
    \frac{L_0}{E_{\mu_h,r}^{\star}}
    \le
    \frac{\log M}{2\log2}.
\end{equation}
Thus admissibility alone bounds the ratio by a logarithmic factor. For a
sharper bound, suppose that for every reference history \(h'\), the set
\[
    G_r(h'):=\{y\in\calY_h:\bar L_\chi(y,h')\le r\}
\]
has cardinality at most \(m_r\), for some integer \(1\le m_r<M\), then
\begin{equation}
\label{eq:localized-entropy-gap}
    E_{\mu_h,r}^{\star}
    \ge
    \log\frac{M}{m_r},
    \qquad
    \frac{L_0}{E_{\mu_h,r}^{\star}}
    \le
    \frac{\log M}{\log(M/m_r)}.
\end{equation}
For example, \(m_r\le M^{1-\alpha}\) for a constant \(\alpha>0\) gives
\(E_{\mu_h,r}^{\star}\ge c\log M\), which suffices for the one-sided
comparison. Two-sided equivalence requires
\(E_{\mu_h,r}^{\star}\asymp\log M\). By contrast, \(p_{\mu_h,r}^{\star}\le\varrho<1\)
alone gives only \(L_0/E_{\mu_h,r}^{\star}\le \log M/\log(1/\varrho)\).
\end{proposition}

\proofinappendix{app:proof-prop:entropy-gap}

\subsection{From entropy comparison to regret comparison}

Comparing the information budgets does not yet tell us whether the regret bounds match. For example, a uniform code on a finite index set gives a log-cardinality budget, but does not by itself identify a regret radius.

The most direct way to close the sandwich is to require that the two regret bounds match up to constant factors at a common radius. Instead, we give two conditions that let us bound the gap between them. Calibration bounds the upper Bellman value evaluated at the ghost entropy associated with the lower-bound radius. The growth condition then controls how much the upper value can increase if the code length exceeds that entropy. The use of critical radii and growth control to compare bounds has precedents in statistical learning \citep{bartlett2005local,mendelson2014learning,mendelson2017local} and parallels in constrained DEC analyses \citep{foster2023tight}.

For a fixed state/index representation, comparison and control correspondences, and reference-update family in the log-penalized Bellman program of Section~\ref{subsec:information-Bellman}, write
\begin{equation}
\label{eq:upper-budget-value}
    \mathsf U_T(L)
    :=
    \inf_{\gamma>0}
    \left\{
        W_1^\gamma(s_1)+\gamma L+\Delta_T^\gamma
    \right\},
\end{equation}
where \(W^\gamma\) is the log-penalized Bellman value and
\(\Delta_T^\gamma\ge0\) accounts for statewise bracket, approximation, and
optimization errors as in Theorem~\ref{thm:frequentist-info-risk-upper}.
With exact statewise optimization and sure calibration, \(\Delta_T^\gamma=0\).
As a pointwise infimum of affine functions of \(L\) with positive slopes,
\(\mathsf U_T\) is nondecreasing and concave on its finite domain.

\begin{assumption}[Bellman upper-growth condition]
\label{ass:upper-growth}
There exist constants \(C_{\rm gr}\ge1\) and \(\beta\in[0,1]\) such that, for all \(L>0\) and \(a\ge1\),
\[
    \mathsf U_T(aL)
    \le
    C_{\rm gr}a^\beta\mathsf U_T(L).
\]
Finite index cardinality alone does not give this growth bound. As a high-level interpretation, the
condition plays the role of a sub-root or growth condition in fixed
experiments and constrained DEC analyses.
\end{assumption}

If \(\mathsf U_T(0)\ge0\), as is the case when all intercepts
\(W_1^\gamma(s_1)+\Delta_T^\gamma\) are nonnegative, concavity gives
\[
    \mathsf U_T(aL)
    \le a\mathsf U_T(L)-(a-1)\mathsf U_T(0)
    \le a\mathsf U_T(L)
    \qquad(a\ge1).
\]
Thus the condition holds with \(C_{\rm gr}=\beta=1\) in this case.
Nonnegative decision losses do not by themselves imply a nonnegative
zero-budget value, since the Bellman bracket subtracts the information reward.

\begin{lemma}[Monotonicity alone does not control growth]
\label{lem:growth-needed}
For any \(0<E<L\) and \(B>0\), there is a positive nondecreasing budget
proxy \(\mathsf U\) with \(\mathsf U(L)/\mathsf U(E)\ge B\).
This example concerns arbitrary monotone proxies, not the concave value
in \eqref{eq:upper-budget-value}.
\end{lemma}

\proofinappendix{app:proof-lem:growth-needed}

\begin{definition}[Upper-at-lower-entropy calibration]
\label{def:upper-at-lower-calibration}
Let \(\mu\) be Bellman--Fano admissible at radius \(r\) with finite
algorithm-uniform ghost entropy \(E_{\mu,r}^{\star}\). The upper value is
\emph{calibrated at the lower entropy scale} if
\begin{equation}
\label{eq:upper-at-lower-entropy}
    \mathsf U_T(E_{\mu,r}^{\star})
    \le
    C_{\rm base}\,T r
\end{equation}
for a universal or problem-controlled constant \(C_{\rm base}\).
Here calibration is a rate-matching condition, distinct from the statewise truth-containment condition in Theorem~\ref{thm:frequentist-info-risk-upper}.  The phrase ``the upper and lower values close at the same radius'' is shorthand for \eqref{eq:upper-at-lower-entropy}.  The condition may be checked directly, through a localized or constrained information-gain estimate, or by choosing \(r\) from the fixed-point relation \(\mathsf U_T(E_{\mu,r}^{\star})\asymp Tr\).
\end{definition}

\begin{proposition}[Three scale-closing diagnostics]
\label{prop:scale-closing-diagnostics}
For a Bellman--Fano admissible pair \((\mu,r)\) with finite
\(E_{\mu,r}^{\star}\), each of the following implies
\eqref{eq:upper-at-lower-entropy}.
\begin{enumerate}[label=(\roman*)]
\item Direct budget control: \(\mathsf U_T(E_{\mu,r}^{\star})\le C_{\rm base}Tr\).
\item Coefficient control: there exists \(\gamma>0\) such that
\[
    W_1^\gamma(s_1)+\gamma E_{\mu,r}^{\star}+\Delta_T^\gamma
    \le
    C_{\rm base}Tr.
\]
\item Fixed-point control: \(r\) is chosen so that \(\mathsf U_T(E_{\mu,r}^{\star})/T\le C_{\rm base}r\).
\end{enumerate}
As a high-level interpretation, condition (ii) is the form closest to constrained DEC and localized information-gain analyses: the same localized comparison that controls the information budget also bounds the decision term at radius \(r\).  Condition (iii) packages the same requirement as a rate fixed point.
\end{proposition}

\proofinappendix{app:proof-prop:scale-closing-diagnostics}

Once both bounds are established, Theorem~\ref{thm:interactive-info-risk-sandwich} controls their gap under two additional conditions: calibration at the lower entropy scale (Definition~\ref{def:upper-at-lower-calibration}) and a growth bound on the upper value (Assumption~\ref{ass:upper-growth}).
\begin{theorem}[Frequentist Bellman information-risk sandwich]
\label{thm:interactive-info-risk-sandwich}
Fix \(\Omega_0\subseteq\Omega\), a horizon \(T\), an index \(\chi\), a Bellman-sufficient state/reference update, and an initial index law \(q_1\) with finite \(L_0=L_0(\Omega_0;q_1,\chi)\).  Suppose the upper side satisfies Theorem~\ref{thm:frequentist-info-risk-upper}, with budget value \(\mathsf U_T\) in \eqref{eq:upper-budget-value}.  Then
\begin{equation}
\label{eq:sandwich-upper}
    \mathfrak R_T^\star(\Omega_0)
    \le
    \mathsf U_T(L_0).
\end{equation}
Conversely, if a prior \(\mu\) supported on \(\Omega_0\) is Bellman--Fano admissible at radius \(r\), then
\begin{equation}
\label{eq:sandwich-lower}
    \mathfrak R_T^\star(\Omega_0)
    \ge
    \frac{Tr}{2}.
\end{equation}
If, in addition, Assumption~\ref{ass:upper-growth} holds and the upper value is calibrated at the lower entropy scale in the sense of Definition~\ref{def:upper-at-lower-calibration}, then
\begin{equation}
\label{eq:regret-gap-comparison}
    \mathfrak R_T^\star(\Omega_0)
    \le
    C_{\rm gr}C_{\rm base}
    \left(\max\left\{1,\frac{L_0}{E_{\mu,r}^{\star}}\right\}\right)^\beta
    Tr,
\end{equation}
and hence the ratio between the displayed upper value and the Bellman--Fano lower bound \(Tr/2\) is at most
\[
    2C_{\rm gr}C_{\rm base}
    \left(\max\left\{1,\frac{L_0}{E_{\mu,r}^{\star}}\right\}\right)^\beta.
\]
If, in addition, \(\mu=\mu_h\) is the uniform representative prior of Proposition~\ref{prop:entropy-gap}, the upper reference law is the same marginal \(q_1=q_{1,\mu_h}=\chi_\#\mu_h\), and \(\chi(\Omega_0)\) consists of those \(M\) indices, then \(L_0=\log M\), and admissibility alone gives the crude logarithmic comparison
\[
    \frac{L_0}{E_{\mu,r}^{\star}}
    \le
    \frac{\log M}{2\log2},
\]
while the multiplicity/localization condition \(|G_r(h')|\le m_r\) gives
\[
    \frac{L_0}{E_{\mu,r}^{\star}}
    \le
    \frac{\log M}{\log(M/m_r)}.
\]
The same statement applies after replacing \(\Omega_0\) throughout by the representative subfamily used by a uniform hard prior. Thus logarithmic-gap matching on a common finite comparison class requires only Bellman--Fano admissibility, upper growth, and calibration at the lower entropy scale. Constant-factor matching follows from the stronger one-sided localization \(E_{\mu,r}^{\star}\ge c\log M\), or from an equivalent constraint or fixed-point argument; a two-sided entropy characterization additionally requires \(E_{\mu,r}^{\star}=O(\log M)\).
\end{theorem}

\proofinappendix{app:proof-thm:interactive-info-risk-sandwich}

\section{Examples of Bellman-sufficient state representations}
\label{app:states}

The following examples verify Definition~\ref{def:index-bellman} directly.  They also fix the interpretation of the state used later: the fixed truth remains the environment, while sufficient statistics, posteriors, and finite marginals are retained Bellman states or reference objects.

\paragraph{Adversarial and estimated losses.}
The framework accommodates adversarial losses beyond stochastic regret
\citep{cesa2006prediction,abernethy2011blackwell}.  A nonstochastic or
adaptive adversary may be treated as part of the fixed environment \(\omega\)
when its rule is nonanticipating.  If the adversary observes the history and
the learner's announced mixed action, then that mixed action is the current
control, not an additional component of the past, and the primitives are
\[
    \ell_t(\omega,H_{t-1},p,a),\qquad
    P_{\omega,t}(\cdot\mid H_{t-1},p,a).
\]
With the full history as state and \(p\) as a control argument, this
representation is automatic for any fixed nonanticipating adversary rule.
Estimated-loss upper bounds may replace the true loss in the Bellman bracket
by a state-measurable surrogate or domination bound while carrying the
approximation error.  A lower bound still requires the cumulative indexed
loss used in the ghost event to be lower bounded, usually nonnegative after a
known shift.

\subsection{Basic examples of Bellman-sufficient representations}\label{sec:bellman-sufficient-examples}
Guided by the classical intuition of sufficient statistics, the Bellman-sufficiency framework is particularly transparent in structured bandit problems.
\begin{example}[Classical sufficient statistics as one-step Bellman states]
\label{ex:classical-sufficiency}
Consider a dominated, nonadaptive experiment with independent observations
\(X_1,\ldots,X_n\) from \(P_\theta\). Suppose \(T_t=T_t(X_1,\ldots,X_t)\)
is sufficient for each prefix and admits the update
\(T_t=u_t(T_{t-1},X_t)\). If the loss and index depend on \(\theta\) only
through quantities whose posterior distribution is determined by \(T_t\),
then \(S_t=(t,T_{t-1})\) is Bellman sufficient. At a fixed truth the next
observation has law \(P_\theta\); under a reference law, the prefix
factorization makes the posterior predictive and index laws functions
of the statistic. For a regular exponential family
\[
    p_\theta(x)=h(x)\exp\{\vartheta(\theta)^\top T(x)-A(\theta)\},
\]
with a conjugate or otherwise statistic-measurable reference law, the
corresponding state uses the cumulative statistic \(\sum_{i<t}T(X_i)\).
\end{example}

\begin{example}[Finite stochastic multi-armed bandits]
\label{ex:mab-sufficient-state}
Let \(\Omega\subseteq\Theta^K\) and let arm \(a\) produce an observation from \(P_{\theta_a}\).  For unit-variance Gaussian arms, \(P_{\theta_a}=N(\theta_a,1)\).  A reference Gaussian prior with independent coordinates gives the posterior state
\[
    S_t=\bigl(t,(n_{t,a},\bar X_{t,a},v_{t,a})_{a=1}^K\bigr),
\]
where \(n_{t,a}\) is the number of previous pulls of arm \(a\), \(\bar X_{t,a}\) is the empirical mean when \(n_{t,a}>0\), and \(v_{t,a}\) is the posterior variance.  Equivalently one may store the posterior hyperparameters.  For a fixed truth \(\omega=\theta\),
\[
    P_{\omega,s,a}=N(\theta_a,1),
    \qquad
    \ell_\omega(s,a)=\max_b\theta_b-\theta_a,
\]
and the update of \((n,\bar X,v)\) depends only on \((s,a,o)\). The
posterior predictive law is \(N(m_{t,a},1+v_{t,a})\), where \(m_{t,a}\)
is the posterior mean. For \(Y=A^\star(\theta)\), the laws \(P_{s,a}^y\)
and losses \(\ell_y^\chi(s,a)\) follow by conditioning the Gaussian
posterior on \(A^\star(\theta)=y\) and integrating against it. This
verifies Definition~\ref{def:index-bellman}. The empirical mean alone
does not determine the posterior predictive law: the count or posterior
variance is also needed for prediction, confidence, and information gain.
\end{example}

\begin{example}[Finite-dimensional Gaussian linear bandits]
\label{ex:linear-sufficient-state}
Let \(\Omega\subseteq\R^d\), actions satisfy \(\|a\|_2\le1\), and
\[
    O_t=\langle \theta,A_t\rangle+\xi_t,
    \qquad \xi_t\sim N(0,1).
\]
Under a Gaussian reference prior, the exact Bayesian state is \(S_t=(t,m_t,\Sigma_t)\), where
\[
    \Sigma_{t+1}^{-1}=\Sigma_t^{-1}+A_tA_t^\top,
    \qquad
    m_{t+1}=\Sigma_{t+1}\{\Sigma_t^{-1}m_t+A_tO_t\}.
\]
For fixed \(\theta\), \(P_{\theta,s,a}=N(\langle\theta,a\rangle,1)\) and the regret loss is \(\sup_{\|b\|_2\le1}\langle\theta,b\rangle-\langle\theta,a\rangle\).  The posterior predictive law is $N(\langle m_t,a\rangle,1+a^\top\Sigma_ta)$. For any of the indices $\chi(\theta)=\theta$, $\chi(\theta)=\theta/\|\theta\|_2$ (with an arbitrary fixed value at $\theta=0$), or $\chi(\theta)=A^\star(\theta)$, conditioning the Gaussian posterior on $\chi(\theta)=y$ determines $P_{s,a}^y$ and $\ell_y^\chi(s,a)$.  Therefore the Gaussian mean--covariance pair is an exact Bellman-sufficient state.  A frequentist UCB state \((\hat\theta_t,V_t,\beta_t)\) is different: it can control a fixed-truth bracket on the calibration event \(\|\hat\theta_t-\theta^\star\|_{V_t}\le\beta_t\), but it is not an exact posterior-reference state unless a reference belief or confidence-to-belief map is added.
\end{example}

\begin{example}[Finite active marginals in kernel bandits]
\label{ex:kernel-sufficient-state}
Let \(f^\star\) belong to an RKHS on a possibly infinite domain, but suppose the learner is evaluated on a finite active set \(\calX=\{x_1,\ldots,x_n\}\).  Under the GP reference law on the active reward vector \(F=(f(x_1),\ldots,f(x_n))\), the finite marginal posterior \((m_t,\Sigma_t)\in\R^n\times\R^{n\times n}\) is Bellman sufficient for the finite experiment.  For fixed \(f^\star\), pulling \(x_i\) has law \(N(f^\star(x_i),\lambda)\) in the Gaussian reference calculation; the posterior predictive law is \(N(m_{t,i},\lambda+\Sigma_{t,ii})\); and the usual rank-one Gaussian update is a function of \((m_t,\Sigma_t,i,O_t)\).  If the index is the optimal active action \(Y=A^\star(F)\), then index sufficiency follows by conditioning the finite Gaussian posterior on \(A^\star(F)=y\), using the fixed tie-breaking rule. The infinite-dimensional RKHS function is still the frequentist truth; the sufficient state is finite only because the decisions and observations use the active marginal.
\end{example}
\subsection{The full environment posterior as a fallback sufficient state}\label{subsec:full-model}
For sequential decision making, a natural Bellman-sufficient representation is obtained by retaining the full environment posterior. For suitably defined Markovian model classes over the state space, which include standard episodic model-based reinforcement learning settings, this posterior serves as a fallback sufficient state.
\begin{example}[Full environment posterior as a fallback sufficient state]
  Consider a realizable 
model class in which each environment \(\omega\in\Omega\) specifies, for every retained physical
state or context \(z\), action \(a\), and stage \(t\), an observation kernel
\(P_{\omega,t}(\cdot\mid z,a)\), a loss \(\ell_{\omega,t}(z,a)\), and the physical-state
transition rule.  Let \(Z_t\) denote the physical state, context, stage, or episode
coordinate needed to make these kernels Markov.  Fix a reference prior \(\mu\) and set
\[
  \Pi_t=\mathcal L_\mu(\Omega\mid H_{t-1}),\qquad S_t=(Z_t,\Pi_t).
\]
Assume that the observation laws are dominated on the relevant support and that the Bayesian update denominator is positive along the reference experiment. Fixed-truth coordinate identities additionally require the evaluated truth's predictive laws to be absolutely continuous with respect to the corresponding reference predictive laws. Then \(S_t\)
gives an exact Bellman-sufficient state.  Indeed, for each fixed truth \(\omega\),
\[
    P_{\omega,S_t,a}=P_{\omega,t}(\cdot\mid Z_t,a),\qquad
    \ell_\omega(S_t,a)=\ell_{\omega,t}(Z_t,a),
\]
and the next retained state is obtained by updating the physical coordinate and
the posterior
\[
    \Pi_{t+1}(d\omega)
    =
    \frac{p_{\omega,t}(O_t\mid Z_t,A_t)\Pi_t(d\omega)}
         {\int p_{\omega',t}(O_t\mid Z_t,A_t)\Pi_t(d\omega')}.
\]
For an index \(\chi:\Omega\to\mathcal Y\), the indexed marginal and conditional
predictive components are
\[
   q_t=\chi_\#\Pi_t,\qquad
    \Pi_t^y=\mathcal L_{\Pi_t}(\Omega\mid \chi(\Omega)=y),
\]
\[
    P_{S_t,a}^y
    =
    \int P_{\omega,t}(\cdot\mid Z_t,a)\,\Pi_t^y(d\omega),
    \qquad
    \ell_y^\chi(S_t,a)
    =
    \int \ell_{\omega,t}(Z_t,a)\,\Pi_t^y(d\omega).
\]
Thus prediction, loss evaluation, posterior/reference updating, and indexed
information accounting all close on \(S_t\).  This verifies
Definition~\ref{def:index-bellman}.
\end{example}

\begin{example}[Stochastic and adversarial DMSO]
Take one Bellman step per episode to recover independent-episode DMSO
\citep{foster2021statistical}.
Given a reference prior and likelihood, let
$\Pi_k=\mathcal L(M\mid H_{k-1})$ be the posterior at the start of
episode $k$. The posterior, together with any relevant public context
and decision constraints, determines the predictive law, the posterior
update, and the index marginal $\chi_\#\Pi_k$.

The model also represents adversarial DMSO: the fixed environment is
Nature's nonanticipating rule for choosing the round model, rather
than a single model used in every round. Standard adversarial DMSO
uses one comparator decision across rounds, while the hybrid formulation of \citet{liu2025decision} lets the comparator vary with
the learner's history. Both evaluate the comparator in the models
selected during the learner's interaction.

For reinforcement learning with dependencies across episodes, the
controlled-process model also accommodates a counterfactual benchmark:
the comparator is evaluated along the trajectory it would itself
generate under the same environment rule.
\end{example}

\section{AIR/MAIR variational calculation and interactive Fano adaptation}
\subsection{Posterior scoring and AIR/MAIR variational calculation}\label{app:air-mair}
\paragraph{Deriving the AIR bracket from posterior scoring.}
\label{rem:danskin}
Fix a state \(s\), an action law \(p\), and a reference index marginal
\(q_s\in\Delta(\mathcal Y)\). Let
\[
    \nu\in\Delta(\Omega\times\mathcal Y)
\]
be an admissible pair belief supported on \(Y=\chi(\Omega)\). The
exact-posterior case has \(q_\nu=q_s\). Robust AIR/AMS/EBO allows
\(q_\nu\ne q_s\), at a KL cost. Given a prediction rule
\[
    Q:(s,p,a,o)\mapsto\Delta(\mathcal Y),
\]
consider the posterior-scoring objective
\[
    \mathcal J_\gamma(s,p,\nu,Q;q_s)
    :=
    \E_{\substack{(\Omega,Y)\sim\nu\\ A\sim p\\
        O\sim P_{\Omega,s,p,A}}}
    \left[
        \ell_{\Omega,Y}(s,p)
        -
        \gamma
        \log\frac{Q(s,p,A,O)(Y)}{q_s(Y)}
    \right].
\]
We hold \(q_s\) fixed in this calculation. If the Bellman program also
optimizes over the marginal, apply the local Danskin argument after fixing
its current value in the control or state.

For fixed \((s,p,\nu,q_s)\), the logarithmic score is convex in \(Q\) and is minimized, pointwise in \((s,p,a,o)\), by the posterior index marginal generated by the same pair belief \(\nu\):
\[
    Q_\nu(s,p,a,o)
    =
    q_\nu^+(\cdot\mid s,p,a,o)
    :=
    \nu\!\left(Y\in\cdot\mid s,p,a,O=o\right),
\]
using a regular conditional law or Bayes' formula in a dominated model.
The optimized value
\[
    \mathcal A_\gamma(s,p,\nu;q_s)
    :=
    \inf_Q \mathcal J_\gamma(s,p,\nu,Q;q_s)
\]
is the one-step indexed AIR objective: immediate indexed loss minus
\(\gamma\) times the posterior log gain.

For fixed \(Q\) and \(q_s\), the objective is affine in \(\nu\); its
infimum over \(Q\) is therefore concave in \(\nu\). Under the compactness,
support, and differentiability conditions for Danskin's theorem, one
can differentiate \(\mathcal A_\gamma\) while holding
\(Q_\nu=q_\nu^+\) fixed. The resulting gradient coordinate at a fixed
truth \((\omega^\star,y^\star)\) is
\[
    \ell_{\omega^\star,y^\star}(s,p)
    -
    \gamma\,
    \E_{\substack{A\sim p\\ O\sim P_{\omega^\star,s,p,A}}}
    \log
    \frac{
        q_\nu^+(y^\star\mid s,p,A,O)
    }{
        q_s(y^\star)
    },
\]
This is the AIR gradient bracket. Its log score telescopes when the
reference update uses the same posterior index predictor.

The notation in \citet[Section~5]{xu2025bayesian} uses the action index \(Y=A^\star(\Omega)\).  Taking \(Y=\Omega\) gives the MAIR or model-index form.  In the stationary-posterior case, \(\nu\) is the current posterior over \((\Omega,Y)\) and \(q_s=\nu_Y\).  In robust AIR, AMS, or EBO, \(\nu\) may be a selected or optimized admissible belief, and validity comes from the corresponding Bellman bracket or calibration bound.  The same concavity--convexity structure also explains the Nash-equilibrium statement in \citet[Lemma~5.2]{xu2025bayesian}, although the formal proof still has to check compactness, support, and boundary conditions.

\begin{lemma}[MAIR gradient bracket and three canonical specializations]\label{lem:mair-gradient}
Assume the observation laws are dominated and the displayed derivatives exist.  For fixed $s,p,\rho,\mu$, up to an additive constant on the probability simplex,
\[
\frac{\partial}{\partial\mu(\omega)}\MAIR_{\rho,\gamma}(s,p,\mu)
=\ell_\omega(s,p)-\gamma\log\frac{\mu(\omega)}{\rho(\omega)}
-\gamma\E_{a\sim p}\KL(P_{\omega,s,p,a}\|P_{\mu,s,p,a}).
\]
Consequently, for every $\omega^\star$,
\begin{align}\label{eq:mair-gradient-bracket}
&\MAIR_{\rho,\gamma}(s,p,\mu)
+\left\langle\nabla_\mu\MAIR_{\rho,\gamma}(s,p,\mu),\delta_{\omega^\star}-\mu\right\rangle\notag\\
&\qquad=
\ell_{\omega^\star}(s,p)
-\gamma\E_{a\sim p}\KL(P_{\omega^\star,s,p,a}\|P_{\mu,s,p,a})
-\gamma\log\frac{\mu(\omega^\star)}{\rho(\omega^\star)}.
\end{align}
This bracket has three common specializations.
\begin{enumerate}[leftmargin=*]
\item If $\mu=\rho$, the logarithmic correction vanishes and the bracket becomes the KL-DEC offset
\begin{equation}\label{eq:kl-dec-offset}
\ell_{\omega^\star}(s,p)-\gamma\E_{a\sim p}\KL(P_{\omega^\star,s,p,a}\|P_{\mu,s,p,a}).
\end{equation}
\item If $\bar\mu\in\argmax_{\mu\in\Delta(\Omega)}\MAIR_{\rho,\gamma}(s,p,\mu)$ is an interior maximizer, then
\begin{equation}\label{eq:mams-max-bracket}
\MAIR_{\rho,\gamma}(s,p,\bar\mu)
+\left\langle\nabla_\mu\MAIR_{\rho,\gamma}(s,p,\bar\mu),\delta_{\omega^\star}-\bar\mu\right\rangle
\le \MAIR_{\rho,\gamma}(s,p,\bar\mu).
\end{equation}
This is the model-index AMS/EBO bracket control.
\item Let $p_{\omega,a}=dP_{\omega,s,p,a}/d\nu_{s,p,a}$ and define, for an auxiliary pair $(\bar a,\bar o)$,
\begin{equation}\label{eq:square-root-posterior}
\frac{d\mu_{\bar a,\bar o}^{1/2}}{d\rho}(\omega)
:=\frac{\sqrt{p_{\omega,\bar a}(\bar o)}}{\int\sqrt{p_{\omega',\bar a}(\bar o)}\,\rho(d\omega')}.
\end{equation}
Let $h^2(P,Q):=1-\int\sqrt{dP\,dQ}$.  If $\bar a\sim p$ and $\bar o\sim P_{\omega^\star,s,p,\bar a}$, then
\begin{align}\label{eq:sqrt-posterior-bracket}
&\E_{\bar a,\bar o}\Big[\MAIR_{\rho,\gamma}(s,p,\mu_{\bar a,\bar o}^{1/2})
+\langle\nabla_\mu\MAIR_{\rho,\gamma}(s,p,\mu_{\bar a,\bar o}^{1/2}),\delta_{\omega^\star}-\mu_{\bar a,\bar o}^{1/2}\rangle\Big]\notag\\
&\qquad\le
\ell_{\omega^\star}(s,p)-\gamma\E_{a\sim p}\E_{\omega\sim\rho}h^2(P_{\omega^\star,s,p,a},P_{\omega,s,p,a}).
\end{align}
Thus the square-root posterior turns the MAIR bracket into a Hellinger DEC offset.
\end{enumerate}
\end{lemma}

\proofinappendix{app:proof-lem:mair-gradient}

\subsection{Adaptation of interactive Fano}
\label{app:lower-details}

Theorem~\ref{thm:quantile-index} adapts the interactive Fano method of
\citet{chen2024unified}. It compares the true and ghost experiments by
applying data processing to the low-loss event. The resulting comparison
uses trajectory information and the negative log of the ghost-good
probability. Here the ghost history follows the Bayesian posterior
predictive law, independently of \(Y=\chi(\Omega)\). When the initial
packet is deterministic or independent of \(Y\), that same reference law
gives the Bellman-state chain rule
\[
    I_\mu(Y;H_T)
    =
    \E_{H_T'\sim\bar{\mathbb P}_\mu^\Alg}
    \sum_{t=1}^T
    \mathcal I_\chi(S_t',p_t').
\]
If $O_0$ is informative, this is the conditional information acquired after
$O_0$, and the unconditional budget includes $I_\mu(Y;O_0)$.  Taking
$Y=\Omega$ gives the model-index form, while $Y=A^\star(\Omega)$ gives the
action-index form.

\section{Algorithmic controls of the Bellman bracket}
\label{app:algorithm-comparisons}

\subsection{UCB families: calibration plus optimism}
An upper-confidence-bound (UCB) algorithm builds a confidence band around a
prediction and chooses an optimistic action
\citep{auer2002finite,abbasi2011improved,srinivas2010gaussian}. Write
\(\sigma_t(a)\) for the uncertainty scale (\(s_t(x)\) in the GP
specialization below). The confidence half-width is
\[
    w_t(a)=\beta_t\sigma_t(a),
\]
where \(\beta_t\) is the confidence multiplier. In reward maximization,
calibration puts the true reward in the band and optimism maximizes the
upper endpoint. Together they give
\[
    r_{\omega^\star}(s,a^\star)
    \le m_t(a^\star)+w_t(a^\star)
    \le m_t(a_t)+w_t(a_t)
    \le r_{\omega^\star}(s,a_t)+2w_t(a_t),
\]
so the instantaneous regret is at most \(2\beta_t\sigma_t(a_t)\).
An elliptical-potential argument bounds the realized sum of
\(\sigma_t^2(a_t)\) in the settings considered here. UCB thus controls
the fixed-truth bracket through the optimistic actions without optimizing
the AIR/MAIR objective directly.

\begin{lemma}[Calibration and optimism imply bracket control]
\label{lem:ucb-bracket}
Fix a truth $\omega^\star$ and condition on an event $\mathcal E$ on which the algorithm's state is calibrated for that truth.  Suppose that, for every reached round $t$, the action $a_t$ chosen by the algorithm and a predictable uncertainty scale $\sigma_t(a)\ge0$ satisfy
\begin{equation}
\label{eq:generic-optimism-width}
   \ell_{\omega^\star}(S_t,a_t)
    \le
    c_{\rm opt}\,\beta_t\sigma_t(a_t),
\end{equation}
where $c_{\rm opt}$ is a numerical optimism constant; in the usual two-sided reward-confidence proof above, $c_{\rm opt}=2$.  Then for every $\gamma>0$,
\begin{equation}
\label{eq:generic-ucb-offset}
 \ell_{\omega^\star}(S_t,a_t)-\gamma
    \sigma_t^2(a_t)
    \le
    \frac{c_{\rm opt}^2\beta_t^2}{4\gamma}.
\end{equation}
Consequently, for the deterministic-action fixed-truth bracket in \eqref{eq:index-fixed-bracket},
\begin{equation}
\label{eq:generic-ucb-gradient-bracket}
   B_{\chi,\gamma}(S_t,\delta_{a_t};\omega^\star)
    \le
    \frac{c_{\rm opt}^2\beta_t^2}{4\gamma}
    +
    \gamma\left(\sigma_t^2(a_t)-J_\chi(S_t,\delta_{a_t};\omega^\star)\right),
\end{equation}
where the fixed-truth index log gain is defined in \eqref{eq:index-fixed-loggain}.  If, in addition, the log-determinant/elliptical-potential argument gives a realized bound
\[
    \sum_{t=1}^T\sigma_t^2(a_t)\le c_{\rm sn}G_T,
\]
for an information telescope $G_T$, then summing \eqref{eq:generic-ucb-gradient-bracket} reduces UCB regret control to the comparison between $G_T$ and the fixed-truth log-gain telescope.
\end{lemma}

\proofinappendix{app:proof-lem:ucb-bracket}

The variance proxy \(\sigma_t^2\) and the fixed-truth log gain \(J_\chi\)
serve different purposes: the elliptical potential bounds the sum of
variances, while \(\beta_t\) converts uncertainty into a regret bound
through optimism.

\subsection{E2D: minimax decision-estimation tradeoff}
The estimation-to-decision (E2D) algorithms \citep{foster2021statistical}
choose \(p\) by minimizing the worst-case DEC offset at the current state.
The schematic KL form, with a model index, is
\begin{equation}
\label{eq:mair-e2d-offset}
    \inf_{p\in\Delta(\calA_t(s))}
    \sup_{\omega\in\mathfrak C_t(s)}
    \left\{
        \ell_\omega(s,p)
        -\gamma
        \E_{a\sim p}
        \KL\!\left(P_{\omega,s,p,a}\middle\|P_{\rho,s,p,a}\right)
    \right\},
\end{equation}
Here \(\rho\) is a reference belief or model and \(\mathfrak C_t(s)\) is the
local comparison set. As in Lemma~\ref{lem:mair-gradient}, the offset
balances decision loss against estimation error without a continuation term.
When an application needs a nonzero continuation value, that term must be
bounded separately. The same minimax principle underlies the original
Hellinger-based formulation and its constrained variants
\citep{foster2023tight}.

A frequentist guarantee using this offset requires the comparison set to
contain the truth along the realized trajectory and a cumulative divergence
bound for the chosen estimator. Under posterior averaging, the expected KL
terms may instead sum to mutual information.

\subsection{AMS/EBO: robust joint belief optimization}

The Exploration by Optimization (EBO) algorithm
\citep{lattimore2021mirror,foster2022complexity} and Adaptive Minimax Sampling
(AMS) \citep{xu2025bayesian,liu2025decision,liu2026improved} lead to closely
related robust controls of the log-penalized Bellman bracket. Their
comparison belief is not the learner's update control in
\eqref{eq:fixed-truth-log-penalized-dp}; KL duality relates the objectives.
Take an action or policy index \(Y=\chi(\Omega)\) with finite \(\calY\).
A dominated extension needs positivity, absolute continuity, and
differentiability. At state \(s\), let \(q\in\Delta(\calY)\) be positive
on each coordinate evaluated by the log score. Let \(\mathfrak B_t(s)\)
be a convex set of pair beliefs \(\tilde\nu\) on \((\omega,y)\) supported
on \(y=\chi(\omega)\). The one-step robust AIR/EBO objective is the indexed
AIR functional \eqref{eq:air-functional-general}, written as
\begin{equation}
\label{eq:amsebo-belief-objective}
    \mathfrak A_{q,\gamma}(s,p,\nu)
    =
    \ell_\nu(s,p)
    -\gamma\calI_\chi(s,p;\nu)
    -\gamma\KL(\nu_\chi\|q),
\end{equation}
where $\nu_\chi$ is the index marginal, $\ell_\nu(s,p)=\E_{(\omega,y)\sim\nu}\ell_\omega(s,p)$, and $\calI_\chi(s,p;\nu)=\E_{a\sim p}I_\nu(Y;O\mid s,p,a)$.
The rule minimizes over the action law and maximizes over the pair belief:
\begin{equation}
\label{eq:amsebo-saddle}
    p_t\in
    \argmin_{p\in\Delta(\calA_t(s))}
    \max_{\tilde\nu\in\mathfrak B_t(s)}
    \mathfrak A_{q_t,\gamma}(s,p,\tilde\nu),
    \qquad
    \bar\nu_t\in
    \argmax_{\tilde\nu\in\mathfrak B_t(s)}
    \mathfrak A_{q_t,\gamma}(s,p_t,\tilde\nu).
\end{equation}
The selected belief \(\bar\nu_t\) supplies the posterior index update.
Evaluating the objective at a belief without the inner maximization is a
posterior AIR calculation, not the robust rule.

For a fixed truth \((\omega^\star,y^\star)\), the result below requires the one-sided direction \(\delta_{(\omega^\star,y^\star)}-\bar\nu_t\) to be admissible at the selected maximizer.  A sufficient condition is that the truth point mass belongs to \(\mathfrak B_t(s)\), since this set is convex.

The fixed-truth proof uses the coordinate log loss
\(\gamma\log(1/q_t(y^\star))\). Its average under a candidate belief is
\(\gamma H(\nu_\chi)+\gamma\KL(\nu_\chi\|q_t)\). If the next reference
marginal is the posterior induced by \(\nu\), the expected change is
\(-\gamma\calI_\chi(s,p;\nu)-\gamma\KL(\nu_\chi\|q_t)\).
The EBO potential uses the log-partition dual \eqref{eq:log-partition-dual}.
The two potentials are related by KL duality and the choice of update; they
are not the same scalar function.

For \eqref{eq:amsebo-belief-objective}, the concavity hypothesis in
Lemma~\ref{lem:amsebo-fixed-bracket} follows from posterior scoring
(Appendix~\ref{app:air-mair}). When the prediction class and support
convention are fixed across beliefs, the AIR functional is an infimum
of affine functions of \(\tilde\nu\). For modified relaxations or
restricted prediction classes, verify concavity or use a valid concave
upper bound.

\begin{lemma}[AMS/EBO saddle control of the AIR fixed-truth bracket]
\label{lem:amsebo-fixed-bracket}
Assume the finite or dominated differentiability setting of Lemma~\ref{lem:air-gradient-bracket}.  Fix $(\omega^\star,y^\star)$ with $y^\star=\chi(\omega^\star)$.  Suppose $\tilde\nu\mapsto\mathfrak A_{q_t,\gamma}(s,p_t,\tilde\nu)$ is concave on the convex set $\mathfrak B_t(s)$, $\bar\nu_t$ maximizes it as in \eqref{eq:amsebo-saddle}, and the comparison direction $\delta_{(\omega^\star,y^\star)}-\bar\nu_t$ is admissible for $\mathfrak B_t(s)$.  If
\[
    \max_{\tilde\nu\in\mathfrak B_t(s)}
    \mathfrak A_{q_t,\gamma}(s,p_t,\tilde\nu)
    \le \varepsilon_t,
\]
then the fixed-truth AIR bracket with coefficient $\gamma$ satisfies
\begin{equation}
\label{eq:amsebo-fixed-bracket}
    \ell_{\omega^\star}(s,p_t)
    -\gamma
    \E_{a\sim p_t,O\sim P_{\omega^\star,t}(\cdot\mid s,p_t,a)}
    \log\frac{\bar\nu_t(Y=y^\star\mid s,p_t,a,O)}{q_t(y^\star)}
    \le
    \varepsilon_t .
\end{equation}
If the reference update is $q_{t+1}(\cdot)=\bar\nu_t(Y\in\cdot\mid s,p_t,A_t,O_t)$ and remains positive at $y^\star$, and if all preceding hypotheses hold almost surely along the trajectory, then Theorem~\ref{thm:mair-upper} with $V_t\equiv0$ gives
\begin{equation}
\label{eq:amsebo-fixed-regret}
	    \E_{\omega^\star}L_{\omega^\star}(H_T)
    \le
    \gamma\log\frac1{q_1(y^\star)}
+\E_{\omega^\star}\sum_{t=1}^T\varepsilon_t
\end{equation}
for finite or countable index sets.  If the comparison direction is only approximately admissible, its effect must be bounded explicitly in the one-step bracket and added to $\varepsilon_t$; a future-dependent containment event cannot be appended after the telescope.  Averaging this coordinate bound under a Bayesian prior replaces $\log(1/q_1(Y))$ by the corresponding cross-entropy or by $H_\mu(Y)$ when $q_1=\mu_\chi$.
\end{lemma}

\proofinappendix{app:proof-lem:amsebo-fixed-bracket}

\subsection{DEC and trajectory-implicit information complexity}
\label{app:dec}

With the continuation value set to \(V_t\equiv 0\), the offset DEC retains the immediate decision loss from the fixed-truth bracket \eqref{eq:upper-bracket} and uses divergence from a reference predictive law to measure statistical separation; see Lemma~\ref{lem:mair-gradient}. In the
comparison here, the estimator is fixed while a minimax optimization problem
selects the decision probability.

For the model index \(\chi(\omega)=\omega\), let \(\mathfrak C_t(s)\) be
the current comparison set and \(\bar\nu\) the reference belief. Write
\[
    P_{\bar\nu,s,a}
    :=
    \int P_{\omega,s,a}\,\bar\nu(d\omega)
\]
for its predictive law. The pointwise KL-DEC offset is
\[
\operatorname{dec}^{\textup{KL}}_\gamma(s;\bar\nu)
    :=
    \inf_{p\in\Delta(\calA_t(s))}
    \sup_{\omega\in\mathfrak C_t(s)}
    \left\{
        \E_{a\sim p}\ell_\omega(s,a)
        -
        \gamma
        \E_{a\sim p}
        \KL\!\left(
            P_{\omega,s,a}\,\middle\|\,P_{\bar\nu,s,a}
        \right)
    \right\}.
    \tag{KL-DEC offset}
\]
The quantile PAC--DEC of \citet[Section~3.2.2]{chen2024unified} uses a
related decomposition of trajectory divergence in its lower-bound argument.
A local coefficient can recover the dynamic rate if it gives comparable
control of the same comparison class and information budget. The Bellman
formulation retains the reference trajectory, so the localization can
change with the state.

\paragraph{Localization in RKHS classes.}

In an infinite-dimensional RKHS ball, one can localize a comparison by
spectral truncation, restricting the active support, or using a posterior
or calibrated confidence set; see \citet{foster2023tight} for constrained
and localized DEC. Here the reference state evolves through the Bellman
recursion. In the kernel application, we fix a finite marginal and apply the robust action-index AIR/AMS saddle-point rule on that marginal, without specifying a localized variant of E2D.

\section{Kernel-bandit details}
\label{app:kernel-details}

\begin{proposition}[Exact Hilbert-space linear representation]
\label{prop:kernel-linear-equivalence}
Let $k$ be a positive-semidefinite kernel on $\calX$, let $\calH_k$ be its RKHS, and define the canonical feature map
\[
    \varphi(x):=k(x,\cdot)\in\calH_k.
\]
Then
\[
    f(x)=\langle f,\varphi(x)\rangle_{\calH_k},
    \qquad
    \|\varphi(x)\|_{\calH_k}^2=k(x,x).
\]
Consequently, the RKHS bandit with action set $\calX$, parameter class
$\{f\in\calH_k:\|f\|_{\calH_k}\le B\}$, and observations
\[
    Y_t=f(X_t)+\varepsilon_t
\]
is equivalent to the Hilbert-space stochastic linear bandit with action set
$\mathcal A_k:=\{\varphi(x):x\in\calX\}\subset\calH_k$, parameter
$\theta=f$, and observations
\[
    Y_t=\langle\theta,\varphi(X_t)\rangle_{\calH_k}+\varepsilon_t.
\]
Identify actions with identical features and use compatible tie-breaking.
With common noise, corresponding policies then give the same histories and
regret path by path. When $\operatorname{span}\varphi(\calX)$ has finite
dimension $d$, projection onto this span preserves rewards and cannot
increase the norm. Conversely, every vector in the radius-$B$ ball in this
span is an RKHS element. The induced parameter class is therefore the full
Euclidean radius-$B$ ball under any isometry from the span onto
$\mathbb R^d$. For finite $\calX$, $d=\operatorname{rank}(K_{\calX})\le|\calX|$.

For ridge parameter and Gaussian noise variance $\lambda>0$, let
\[
    V_t
    :=\lambda I_{\calH_k}
       +\sum_{r<t}\varphi(X_r)\otimes\varphi(X_r),
    \qquad
    \widehat\theta_t
    :=V_t^{-1}\sum_{r<t}Y_r\varphi(X_r).
\]
The zero-mean kernel-ridge/GP posterior quantities satisfy
\[
    m_t(x)=\langle\widehat\theta_t,\varphi(x)\rangle_{\calH_k},
    \qquad
    s_t^2(x)=\lambda
       \langle\varphi(x),V_t^{-1}\varphi(x)\rangle_{\calH_k}.
\]
Thus GP--UCB with score $m_t(x)+\beta_t s_t(x)$ chooses the same
action as ridge LinUCB with score
\[
    \langle\widehat\theta_t,u\rangle_{\calH_k}
    +\sqrt\lambda\,\beta_t\|u\|_{V_t^{-1}},
    \qquad u=\varphi(x).
\]
At the unit ridge and noise normalization used in the separation, the two
acquisition rules have the same multiplier.
\end{proposition}

\begin{proof}
The first identities are the reproducing property. For the posterior identities,
let $\Phi_t:\calH_k\to\mathbb R^{t-1}$ be the sampling operator
$(\Phi_t h)_r=\langle h,\varphi(X_r)\rangle_{\calH_k}$. Then
$K_t=\Phi_t\Phi_t^\ast$ and
$V_t=\lambda I+\Phi_t^\ast\Phi_t$. The push-through and Woodbury identities give
\[
    V_t^{-1}\Phi_t^\ast
    =\Phi_t^\ast(K_t+\lambda I)^{-1}
\]
and
\[
    I-\Phi_t^\ast(K_t+\lambda I)^{-1}\Phi_t
    =\lambda V_t^{-1}.
\]
Substitution into the standard GP posterior mean and covariance formulas proves
the claim.
\end{proof}

The operator identity does not represent the GP prior as an isotropic
Gaussian element of $\calH_k$: such an element does not exist in infinite
dimension, and GP sample paths need not belong to the RKHS. Also, after $T$
rounds the design update has rank at most $T$, but the acquisition rule
still maximizes over the full action set. The small-cloud lower bound
depends on directions outside the observed design.

\subsection{Finite-marginal indexed Bellman recursion}

Theorem~\ref{thm:frequentist-info-risk-upper} gives a Bellman upper bound
for a specified state, comparison sets, and family of reference updates.
Its value need not equal the unrestricted minimax value. The state is
$(m_t,\Sigma_t,q_t)$, where
$(m_t,\Sigma_t)$ is the algorithmic GP marginal on $F$ and $q_t$ is the
reference law of the index: $A^\star(F)$ for AIR or $F$ for MAIR. Let
$\mathfrak C_t(m,\Sigma,q)$ be the comparison set of active reward
vectors. Sure calibration requires $F^\star\in\mathfrak C_t(m_t,\Sigma_t,q_t)$
at every reached state. A control
$u\in\mathfrak U_t(m,\Sigma,q)$ specifies an action law $p_u$ and a reference
update.  In AIR form the update is generated by a pair belief $\nu_u$ on active
reward vectors and their optimal-action indices, so that $q^+=q_{\nu_u}^+$.
For the robust saddle calculation, candidate pair beliefs may be supported on
the closed optimal relation
$\{(f,y):y\in\argmax_x f(x)\}$; the true index still uses the prescribed
tie-breaking rule $A^\star(f)$.  In the action-index version define
\[
    y(f)=A^\star(f),
    \qquad
    \mathsf G_t^{\AIR}(u;f)
    :=
    \E_{x\sim p_u,\,O\sim N(f_x,\lambda)}
    \log\frac{q^+(y(f)\mid m,\Sigma,q,x,O)}{q(y(f))},
\]
where $q^+$ is the chosen reference update for the optimal-action marginal.
For a candidate reward vector $f$, write
\[
    \Delta_f(p):=\max_{x\in\calX}f(x)-\E_{x\sim p}f(x).
\]
The robust finite-marginal Bellman recursion is
\[
    W_t^\gamma(m,\Sigma,q)
    =
    \inf_{u\in\mathfrak U_t(m,\Sigma,q)}
    \sup_{f\in\mathfrak C_t(m,\Sigma,q)}
    \left\{
        \Delta_f(p_u)
        +\E_f W_{t+1}^\gamma(m^+,\Sigma^+,q^+)
        -\gamma\mathsf G_t^{\AIR}(u;f)
    \right\}.
\]
The model-index MAIR version replaces $y(f)$ by the finite reward vector $f$
and uses the corresponding model-coordinate posterior log gain.  For the
finite action index, Theorem~\ref{thm:frequentist-info-risk-upper} implies that
an approximate minimizer of this recursion has regret at most
\[
    W_1^\gamma(m_1,\Sigma_1,q_1)
    +\gamma\log\frac1{q_1(A^\star(F^\star))}
    +\text{explicit bracket/optimization errors}.
\]
For the continuous model index $F\in\R^n$, the terminal code-length argument
does not apply: a posterior density may exceed one. The finite action
index used in the main theorem avoids this issue.

If the maintained belief is the exact Bayesian posterior and one averages the
coordinate recursion over $F$, the robust supremum is replaced by posterior
expectation and the dynamic program reduces to
\[
    U_t^\gamma(m,\Sigma,q)
    =
    \inf_{p\in\Delta(\calX)}
    \left\{
        \Delta_t^{\AIR}(p)
        +
        \E U_{t+1}^\gamma(m^+,\Sigma^+,q^+)
        -
        \gamma\calI_t^{\AIR}(p)
    \right\}.
\]
Appendix~\ref{subsec:two_more_algorithms} gives the robust action-index
AIR/AMS control and its EBO interpretation.

\subsection{Finite-scale small-cloud construction}
\label{sec:kernel-separation-details}

We use a small-amplitude version of the hub--cloud construction of
\citet{xu2026separation}. Ignoring the cloud costs at most $aT$, yet the
cloud has enough points to inflate the GP--UCB multiplier.

Fix a horizon $T$, an integer $N\ge1$, and $a\in(0,1]$.  Let
\[
    \calX_T=\{x_0,x_h,c_1,\ldots,c_N\}
\]
and define the diagonal kernel
\begin{equation}
\label{eq:small-cloud-kernel}
    k_T(x_0,x)=0,\qquad
    k_T(x_h,x_h)=1,\qquad
    k_T(c_i,c_j)=a^2\1\{i=j\},
\end{equation}
with all remaining covariances zero.  Thus $k_T(x,x)\le1$.  Its unit RKHS ball is
\begin{equation}
\label{eq:small-cloud-rkhs-ball}
    \calF_T
    =
    \left\{f:f(x_0)=0,\quad
      f(x_h)^2+a^{-2}\sum_{j=1}^N f(c_j)^2\le1
    \right\}.
\end{equation}
Observations are $Y_t=f(X_t)+\xi_t$ with $\xi_t\stackrel{\iid}{\sim}N(0,1)$.

This kernel has the explicit feature map
\begin{equation}
\label{eq:small-cloud-feature-map}
    \varphi(x_0)=0,
    \qquad
    \varphi(x_h)=e_0,
    \qquad
    \varphi(c_j)=ae_j,\quad j\in[N],
\end{equation}
in $\mathbb R^{N+1}$. Hence
$k_T(x,x')=\langle\varphi(x),\varphi(x')\rangle$. Moreover, setting
\[
    \theta_0=f(x_h),
    \qquad
    \theta_j=a^{-1}f(c_j),\quad j\in[N],
\]
gives
\[
    f(x)=\langle\theta,\varphi(x)\rangle,
    \qquad
    \|\theta\|_2^2
      =f(x_h)^2+a^{-2}\sum_{j=1}^Nf(c_j)^2.
\]
Thus $\calF_T$ maps onto the full Euclidean unit parameter ball.

Let
\[
    \Gamma_s(T)
    :=
    \sup_{x_{1:s}\in\calX_T^s}
    \frac12\log\det(I+K_{x_{1:s},x_{1:s}})
\]
be the maximal information gain at noise variance one.  Consider the following horizon-aware maximal-information calibration
\begin{equation}
\label{eq:max-info-beta}
    \beta_t^{\max}
    :=
    1+\sqrt{2\{\Gamma_{t-1}(T)+2\log T\}},
\end{equation}
and the zero-mean, unit-ridge GP--UCB rule
\begin{equation}
\label{eq:max-info-gpucb}
    X_t\in\argmax_{x\in\calX_T}
    \{m_t(x)+\beta_t^{\max}s_t(x)\}.
\end{equation}
The conclusion is insensitive to tie-breaking.

This is the horizon-aware counterpart of
\eqref{eq:fixed-kernel-anytime-rule}, not the distinct-set
schedule~\eqref{eq:original-rkhs-gpucb-rule} of \citet{srinivas2010gaussian}.

\begin{theorem}[Small-cloud separation]
\label{thm:small-cloud-gpucb}
Fix $0<\alpha<1/4$ and put $a=T^{-\alpha}$.  There are constants $0<c_\alpha<C_\alpha<\infty$ such that, for all sufficiently large $T$ and every $N\ge4T$,
\begin{equation}
\label{eq:small-cloud-minimax}
    c_\alpha aT
    \le
    \inf_{\Alg}\sup_{f\in\calF_T}
       \E_f\Reg_T(\Alg)
    \le
    C_\alpha aT.
\end{equation}
In contrast, the GP--UCB rule \eqref{eq:max-info-gpucb} satisfies
\begin{equation}
\label{eq:small-cloud-gpucb-lower}
    \sup_{f\in\calF_T}\E_f\Reg_T(\mathrm{GP\mbox{--}UCB})
    \ge c_\alpha T.
\end{equation}
Consequently,
\begin{equation}
\label{eq:small-cloud-ratio}
    \frac{
      \sup_{f\in\calF_T}\E_f\Reg_T(\mathrm{GP\mbox{--}UCB})
    }{
      \inf_{\Alg}\sup_{f\in\calF_T}\E_f\Reg_T(\Alg)
    }
    \ge c_\alpha T^\alpha.
\end{equation}
Thus maximal-information-calibrated GP--UCB is polynomially minimax-suboptimal on a family of bounded finite kernels, even though the comparison is made over the entire unit RKHS ball.
\end{theorem}

\begin{proof}
\emph{Minimax lower bound.}
For $j\in[N]$, let $f^j(c_j)=a$ and let $f^j$ vanish on every other action.  By \eqref{eq:small-cloud-rkhs-ball}, $f^j\in\calF_T$.  Fix an arbitrary possibly randomized nonanticipating algorithm, let $U$ denote its independent random seed, draw $J$ uniformly from $[N]$, and pre-generate independent Gaussian noise stacks at every action.  Couple the experiments under $f^J$ and $f\equiv0$ using the same $U$ and noise stacks, and let $\tau_J=\inf\{t\ge1:X_t=c_J\}$.  The two histories agree strictly before $\tau_J$, so $\tau_J$ is identical in the coupled experiments.  Under the null experiment, if $S_T^0$ is the set of distinct cloud actions queried by time $T$, then $J$ is independent of $S_T^0$ and
\[
    \Pp(\tau_J\le T)
    =\Pp(J\in S_T^0)
    =\frac{\E|S_T^0|}{N}
    \le\frac{T}{N}\le\frac14.
\]
On $\{\tau_J>T\}$ every selected action has mean zero under $f^J$, whereas the optimal mean is $a$.  Thus the Bayes regret is at least $(1-T/N)aT\ge3aT/4$.  This proves the first inequality in \eqref{eq:small-cloud-minimax}.

\emph{Minimax upper bound.}
Ignore the cloud and run a horizon-$T$ minimax policy for two unit-sub-Gaussian arms on $\{x_0,x_h\}$; a sub-Gaussian version of MOSS is one example \citep{audibert2009minimax,lattimore2020bandit}.  Put
\[
    g_f=\max\{0,f(x_h)\},\qquad
    f_{\max}^\star=\max\{g_f,\max_{j\le N}f(c_j)\}.
\]
Since $|f(x_h)|\le1$, $\max_jf(c_j)\le a$, and $f(x_0)=0$, we have $0\le f_{\max}^\star-g_f\le a$.  The restricted policy has expected regret at most $C\sqrt T$ relative to $g_f$, so its regret relative to the full set is at most
\[
    C\sqrt T+aT\le C_\alpha aT,
\]
where the last inequality uses $\alpha<1/2$.  This proves the second inequality in \eqref{eq:small-cloud-minimax}.

\emph{GP--UCB lower bound.}
Take the hub truth
\begin{equation}
\label{eq:small-cloud-hub-truth}
    f^\dagger(x_h)=\Delta:=\frac14,
    \qquad
    f^\dagger(x_0)=f^\dagger(c_j)=0,\quad j\le N.
\end{equation}
It belongs to $\calF_T$.  For every $s\le T$, querying $s$ distinct cloud actions is feasible, while for an arbitrary design the hub contributes at most $\frac12\log(1+s)$ and the cloud contributes at most $a^2s/2$.  Since $\log(1+u)\ge u/2$ for $u\in[0,1]$,
\begin{equation}
\label{eq:small-cloud-gamma-bounds}
    \frac{s a^2}{4}
    \le
    \Gamma_s(T)
    \le
    \frac12\log(1+s)+\frac{s a^2}{2}.
\end{equation}
Indeed, if $n_h,n_1,\ldots,n_N$ are the design counts, then $n_h+\sum_jn_j\le s$, pulls of $x_0$ contribute zero, and diagonalization gives
\[
    \frac12\log(1+n_h)
    +\frac12\sum_{j=1}^N\log(1+a^2n_j)
    \le
    \frac12\log(1+s)+\frac{a^2s}{2}.
\]

An unqueried cloud action has posterior mean zero and posterior standard deviation $a$, independently of all observations at other actions.  By the lower bound in \eqref{eq:small-cloud-gamma-bounds}, at every $t\ge T/2$ its acquisition score is at least
\begin{equation}
\label{eq:unqueried-cloud-score}
    a\beta_t^{\max}
    \ge
    a^2\sqrt{\frac{t-1}{2}}
    \ge
    cT^{1/2-2\alpha}.
\end{equation}
This diverges because $\alpha<1/4$.  The upper bound in \eqref{eq:small-cloud-gamma-bounds} also gives, uniformly for $t\le T$,
\begin{equation}
\label{eq:max-beta-upper}
    \beta_t^{\max}
    \le
    C\{a\sqrt T+\sqrt{\log T}\}.
\end{equation}

Realize observations through the independent action-specific noise stacks: the $r$th pull of $x$ reveals $f(x)+\xi_{x,r}$.  Let $\xi_{h,r}=\xi_{x_h,r}$ and define
\[
    \mathcal E_T
    :=
    \left\{
       \sum_{r=1}^n\xi_{h,r}
       \le2\sqrt{n\log T}
       \text{ for every }1\le n\le T
    \right\}.
\]
A Gaussian tail bound and a union bound give
\[
    \Pp(\mathcal E_T^c)
    \le\sum_{n=1}^T
       \exp\!\left\{-\frac{(2\sqrt{n\log T})^2}{2n}\right\}
    \le T^{-1}.
\]
After $n$ hub pulls, diagonal conjugacy gives
\[
    m_t(x_h)
    =\frac{n\Delta+\sum_{r=1}^n\xi_{h,r}}{1+n},
    \qquad
    s_t(x_h)=\frac1{\sqrt{1+n}}.
\]

Suppose, toward a contradiction, that fewer than $T/4$ cloud pulls occur by time $T$.  The anchor $x_0$ is never selected: its acquisition score is zero, while an unqueried cloud has strictly positive score.  Hence, for every $t\in[\lceil T/2\rceil,T]$, the number $n$ of preceding hub pulls is at least $T/4-1$.  On $\mathcal E_T$, \eqref{eq:max-beta-upper} yields
\[
\begin{aligned}
    m_t(x_h)+\beta_t^{\max}s_t(x_h)
    &\le
    \Delta+C\sqrt{\frac{\log T}{T}}
      +C\left(a+\sqrt{\frac{\log T}{T}}\right) \\
    &\le \frac12
\end{aligned}
\]
for all sufficiently large $T$.  On the other hand, \eqref{eq:unqueried-cloud-score} is then larger than one.  Since $N\ge4T$, an unqueried cloud is always available.  GP--UCB must therefore select a cloud action on every round in the last half of the horizon, contradicting the supposition that it makes fewer than $T/4$ cloud pulls.  Thus, on $\mathcal E_T$, it makes at least $T/4$ cloud pulls.  Each such pull has regret $\Delta$, and consequently
\[
    \E_{f^\dagger}\Reg_T(\mathrm{GP\mbox{--}UCB})
    \ge
    \frac{\Delta T}{4}(1-T^{-1}),
\]
which proves \eqref{eq:small-cloud-gpucb-lower}.  Combining this lower bound with the minimax upper bound in \eqref{eq:small-cloud-minimax} proves \eqref{eq:small-cloud-ratio}.
\end{proof}

\begin{corollary}[Finite-horizon LinUCB separation]
\label{cor:finite-linucb-separation}
Fix $0<\alpha<1/4$. For every sufficiently large integer $T$, let
\[
    N=T-1,
    \qquad d=N+1=T,
    \qquad a=T^{-\alpha},
\]
and consider the $d$-dimensional action set
\[
    \mathcal A_T
    =\{0,e_0,ae_1,\ldots,ae_N\}\subset\mathbb R^d
\]
and parameter class
$\Theta_T=\{\theta\in\mathbb R^d:\|\theta\|_2\le1\}$. Observations are
\[
    Y_t=\langle\theta,A_t\rangle+\xi_t,
    \qquad \xi_t\stackrel{\iid}{\sim}N(0,1).
\]
Thus $|\mathcal A_T|=d+1=T+1$, while
$\|A\|_2\le1$ and $\|\theta\|_2\le1$.

Let
\[
    V_t=I_d+\sum_{r<t}A_rA_r^\top,
    \qquad
    \widehat\theta_t=V_t^{-1}\sum_{r<t}A_rY_r,
\]
and define the maximal linear information gain
\[
    \Gamma_s^{\rm lin}(\mathcal A_T)
    :=\sup_{u_1,\ldots,u_s\in\mathcal A_T}
      \frac12\log\det\!\left(
        I_d+\sum_{r=1}^s u_ru_r^\top
      \right).
\]
Consider the maximal-information-calibrated ridge LinUCB rule
\begin{equation}
\label{eq:max-info-linucb}
    A_t\in\argmax_{u\in\mathcal A_T}
    \left\{
      \langle\widehat\theta_t,u\rangle
      +\left[1+\sqrt{2\{\Gamma_{t-1}^{\rm lin}(\mathcal A_T)
                         +2\log T\}}\right]
       \|u\|_{V_t^{-1}}
    \right\}.
\end{equation}
With
\[
    \Reg_T(\pi)
    :=\sum_{t=1}^T
      \left\{\max_{u\in\mathcal A_T}\langle\theta,u\rangle
                    -\langle\theta,A_t\rangle\right\},
\]
we have
\[
    \inf_{\pi}\sup_{\theta\in\Theta_T}
       \mathbb E_\theta\Reg_T(\pi)
    =\Theta(T^{1-\alpha}),
\]
whereas
\[
    \sup_{\theta\in\Theta_T}
       \mathbb E_\theta\Reg_T(\mathrm{LinUCB}_{\max})
    =\Omega(T).
\]
Consequently, the worst-case regret of this LinUCB rule divided by the
minimax value on the same action set and parameter class is
$\Omega(T^\alpha)$. The same $\Omega(T)$ conclusion holds for the
finite-action specializations of the anytime multiplier
\eqref{eq:fixed-kernel-anytime-rule} and, for every fixed
$\delta\in(0,1)$, the original-schedule multiplier
\eqref{eq:original-rkhs-gpucb-rule}, for all sufficiently large \(T\), with
the threshold allowed to depend on \(\alpha\) and \(\delta\).
\end{corollary}

\noindent The proof follows after Remark~\ref{rem:finite-linucb-scope}.

\begin{remark}[Scope relative to linear-bandit literature]
\label{rem:finite-linucb-scope}
Corollary~\ref{cor:finite-linucb-separation} compares the specified
LinUCB rule with the unrestricted minimax value over the full unit parameter
ball on one action set for each sufficiently large $T$. The set depends on
$T$, has dimension $d=T$, and has unequal feature norms: one for the hub,
$T^{-\alpha}$ for the cloud, and zero for the anchor. A comparable
separation for $d=o(T)$ or for a uniformly normalized, well-conditioned
action set remains open.

Earlier sharp contextual-bandit analyses used the layered SupLinRel and
SupLinUCB constructions to recover the independence needed in their proofs
\citep{auer2002confidence,chu2011contextual}. Those works established strong
upper bounds but did not prove that the corresponding one-layer rules were
minimax-suboptimal. \citet{lattimore2017optimism} showed that a broad
class of optimistic policies can have arbitrarily worse asymptotic
instance-dependent constants on fixed finite action sets. More recently,
\citet{song2025truncated} proved a finite-time gap for LinUCB in an
i.i.d.-context, arm-specific-parameter model. These results address different
criteria or observation structures. To the best of our knowledge,
Corollary~\ref{cor:finite-linucb-separation} is the first polynomial-in-$T$,
finite-horizon separation in a finite-dimensional, finite-action problem
between the worst-case cumulative regret of the specified
maximal-information-calibrated LinUCB rule and the minimax regret on the same
noncontextual, shared-parameter stochastic linear-bandit class. The
standard OFUL confidence radius, for example, uses the realized design
determinant \citep{abbasi2011improved}, not its maximum over designs, so
the corollary does not automatically apply to OFUL.
Proposition~\ref{prop:kernel-linear-equivalence} also shows that
the horizonwise Mat\'ern lower bound of
\citet{wang2026suboptimality} has an infinite-dimensional Hilbert-linear
interpretation; the finite-dimensional conclusion here is distinct.

\citet{maiti2026power} study fixed-confidence $\varepsilon$-best-arm identification and
construct a union of $k$ orthogonal $r$-dimensional balls in
$\mathbb R^{kr}$. For $r\le k\le r^2$, their construction gives a
nonadaptive lower bound
\[
  \Omega\!\left(
    \frac{kr\log(1/\delta)+kr^2}{\varepsilon^2}
  \right),
\]
while an adaptive procedure uses
\[
  O\!\left(
    \frac{kr\log(1/\delta)+kr\log k}{\varepsilon^2}
  \right)
\]
samples. At fixed confidence these bounds give an adaptivity gap of at
least order $r/\log k$. Their lower bound applies to all nonadaptive fixed
designs; ours compares an adaptive optimistic rule with the unrestricted
cumulative-regret minimax value. Neither theorem implies the other.
\end{remark}

\begin{proof}[Proof of Corollary~\ref{cor:finite-linucb-separation}]
Set
\[
    H:=\left\lfloor\frac{T-1}{4}\right\rfloor .
\]
Then \(N=T-1\ge4H\) and \(H\asymp T\).  Take the feature map in
\eqref{eq:small-cloud-feature-map} with this value of \(N\). The
coordinate isometry sends the unit RKHS ball onto the Euclidean unit
ball and preserves rewards, histories, and regret path by path.
For every design sequence, Sylvester's determinant identity gives
\[
    \det(I_s+K_{x_{1:s},x_{1:s}})
    =
    \det\!\left(
       I_d+\sum_{r=1}^s
       \varphi(x_r)\varphi(x_r)^\top
    \right).
\]
Proposition~\ref{prop:kernel-linear-equivalence}, with \(\lambda=1\), also
gives
\[
    m_t(x)=\langle\widehat\theta_t,\varphi(x)\rangle,
    \qquad
    s_t^2(x)=\varphi(x)^\top V_t^{-1}\varphi(x).
\]
Thus \eqref{eq:max-info-linucb} coincides with the GP--UCB rule with the same
maximal-information multiplier.

We first prove the minimax rate.  For each \(j\in[N]\), take the unit parameter
\(\theta^j=e_j\), for which the unique positive cloud reward is
\(\langle\theta^j,ae_j\rangle=a\).  Draw \(J\) uniformly from \([N]\) and
couple the experiments under \(\theta^J\) and \(\theta=0\) with the same
algorithmic seed and action-specific Gaussian noise stacks.  If \(\tau_J\) is
the first pull of \(ae_J\), the histories agree before \(\tau_J\).  Under the
null experiment, the first \(H\) rounds visit at most \(H\) distinct cloud
actions, and hence
\[
    \mathbb P(\tau_J\le H)\le\frac{H}{N}\le\frac14.
\]
On \(\{\tau_J>H\}\), every one of the first \(H\) pulls has regret \(a\).
Therefore every algorithm has Bayes, and hence worst-case, regret at least
\[
    \left(1-\frac{H}{N}\right)aH
    \ge\frac34aH
    =\Omega(T^{1-\alpha}).
\]
Conversely, ignoring the cloud and using a two-armed minimax policy on
\(\{0,e_0\}\) incurs \(O(\sqrt T)\) regret relative to the better retained
action and at most \(aT\) additional regret relative to the full action set.
This is \(O(T^{1-\alpha})\).

For the LinUCB lower bound, take the hub truth
\(\theta^\dagger=\Delta e_0\), \(\Delta=1/4\).  For every \(s\le H\), the
availability of \(s\) distinct cloud actions and the same diagonal calculation
as in \eqref{eq:small-cloud-gamma-bounds} give
\[
    \frac{sa^2}{4}
    \le
    \Gamma_s^{\rm lin}(\mathcal A_T)
    \le
    \frac12\log(1+s)+\frac{sa^2}{2}.
\]
Thus, throughout \(H/2\le t\le H\), an unqueried cloud action has score at
least \(c a^2\sqrt H\), which tends to infinity because
\(\alpha<1/4\), while the multiplier in \eqref{eq:max-info-linucb} is at most
\(C(a\sqrt H+\sqrt{\log T})\).

Let
\[
    \mathcal E_H
    :=
    \left\{
       \sum_{r=1}^n\xi_{0,r}\le2\sqrt{n\log T}
       \text{ for every }1\le n\le H
    \right\},
\]
where \(\xi_{0,r}\) is the \(r\)th noise observed at the hub.  A union bound
gives \(\mathbb P(\mathcal E_H^c)\le H/T^2\le T^{-1}\).  Suppose that fewer
than \(H/4\) cloud pulls occur by time \(H\).  The zero action is never
selected because an unqueried cloud action has positive score.  Hence, before
every round \(t\in[\lceil H/2\rceil,H]\), the hub has been pulled at least
\(H/4-1\) times.  On \(\mathcal E_H\), its score is at most
\[
    \Delta
    +C\sqrt{\frac{\log T}{H}}
    +C\left(a+\sqrt{\frac{\log T}{H}}\right)
    <\frac12
\]
for all sufficiently large \(T\), while an unqueried cloud score is larger
than one.  Since \(N\ge4H\), such a cloud action is always available under the
supposition.  Every round in the last half of the prefix would therefore be a
cloud pull, a contradiction.  The rule consequently makes at least \(H/4\)
cloud pulls on \(\mathcal E_H\), and
\[
    \mathbb E_{\theta^\dagger}\Reg_T
    \ge
    \frac{\Delta H}{4}(1-T^{-1})
    =\Omega(T).
\]

The anytime multiplier is handled by the same prefix argument, since
\(\log(t\vee2)\asymp\log T\) for \(H/2\le t\le H\).  For the original
schedule and any fixed \(\delta\in(0,1)\), its distinct-set information gain
satisfies, for \(s\le H\),
\[
    \frac{sa^2}{4}
    \le \Gamma_s^{\rm SKKS}(k_T)
    \le \frac12\log2+\frac{sa^2}{2},
\]
because \(N\ge4H\).  Uniformly on the last half of the prefix, an unqueried
cloud score is therefore at least
\(c a^2\sqrt H\log^{3/2}H\to\infty\), whereas after \(H/4-1\) hub pulls the
hub bonus is
\[
    O\!\left(H^{-1/2}\log^{3/2}T
      +a\log^{3/2}T\right)=o(1).
\]
The same event and contradiction force \(\Omega(H)=\Omega(T)\) cloud pulls.
\end{proof}

\subsection{Interpretation and scope of the fixed-kernel separation}
\label{app:fixed-kernel-scope}

\begin{remark}[Scope of the theorem]
The kernel is spatially inhomogeneous and its compact domain is countable;
the theorem therefore does not settle the corresponding question for
translation-invariant kernels on Euclidean domains.  Nor does it rule out
localized, predictable, or data-dependent exploration multipliers.
\end{remark}

\begin{remark}[AIR Bellman, AMS/EBO, and GP--UCB]
On each predetermined marginal $\mathcal R_m$, the robust AIR/AMS saddle
supplies the statewise control used by the AIR Bellman supersolution; EBO is
related through the stated KL duality, not implemented as a separate
functional-estimator algorithm.  The comparison with GP--UCB uses the same
fixed kernel, RKHS ball, domain, and observation model.  The AIR/AMS policy's
retained sets and epoch schedule are fixed in advance from public class
geometry, whereas GP--UCB acts on the full domain with a global
maximal-information multiplier.  The theorem does not show that AIR/AMS learns
the truncation, that a state-dependent continuation value is necessary, or that a GP--UCB
rule localized to the same marginal must fail.
\end{remark}
\subsection{Mat\'ern lower bounds and effective optimism}
\label{app:matern-scope}

For standard Mat\'ern kernels on $[0,1]^d$ with Gaussian noise,
\citet[Theorem~1]{wang2026suboptimality} consider a prespecified exploration
multiplier $b_t$ on the posterior standard deviation and ridge parameter
$\rho_t$, with effective optimism $L_t=\rho_tb_t^2$. They assume that
$b_t^2\ge1$ and $\rho_t>0$ are nondecreasing, with two-sided bounds
$b_t^2\asymp t^{\theta_1}(\log t)^{\theta_3}$ and
$\rho_t\asymp t^{\theta_2}(\log t)^{\theta_4}$ for $3\le t\le T$, with
constants independent of $T$. The exponents are nonnegative and satisfy
$\frac{4\nu+d}{2\nu}\theta_1+\theta_2<1$. They also require a
$T$-independent positive probability that, simultaneously for all inputs
and rounds, the estimation error is at most a fixed fraction
$\zeta<1$ of the UCB bonus. Under these assumptions their worst-case
expected regret is
\[
    \Omega\!\left(T^{\frac{\nu+d}{2\nu+d}}L_T^{\frac{\nu}{2\nu+d}}\right).
\]
The Mat\'ern minimax rate is
$\widetilde\Theta(T^{(\nu+d)/(2\nu+d)})$
\citep{scarlett2017lower,salgia2021gpthreds}, so polynomial growth of $L_T$
gives a polynomial gap under the stated hypotheses. This is a
horizon-by-horizon result; it does not require the same hard function at
all horizons.
\subsection{AMS/EBO specialization for kernel bandits}\label{subsec:two_more_algorithms}

\paragraph{Algorithm 3: AMS/EBO on the action index.}

Write $s_t=(m_t,\Sigma_t,q_t)$, with $q_t$ the reference law of the
optimal-action index $A^\star$. Let $\mathfrak B_t(s_t)$ be a convex set of
beliefs $\nu$ on the active reward vector $F$, each inducing a pair
belief on $(F,A^\star(F))$. Under $\nu$, let $\Delta_\nu^{\AIR}(p)$ and
$\calI_\nu^{\AIR}(p)$ be the regret and information terms in
\eqref{eq:kernel-air-regret-new} and \eqref{eq:kernel-air-info-new}.
For fixed $\gamma>0$, the robust decision is
\begin{equation}
\label{eq:amsebo-kernel-rule-new}
    p_t\in\argmin_{p\in\Delta(\calX)}
    \sup_{\nu\in{\mathfrak B_t(s_t)}}
    \left\{
        \Delta_\nu^{\AIR}(p)
        -
        \gamma\calI_\nu^{\AIR}(p)
        -
        \gamma\KL(\nu_{A^\star}\|q_t)
    \right\}.
\end{equation}
Choose a measurable maximizing belief $\bar\nu_t$.  After drawing $X_t\sim p_t$ and observing $O_t$, update the reference law by
\[
    q_{t+1}=q_{\bar\nu_t}^{+}(\cdot\mid s_t,p_t,X_t,O_t).
\]
Maximizing over $\nu$ makes the decision robust; evaluating the bracket
at a single belief does not. On the finite index set, the last two terms
are the KL-dual form of the log-partition potential
\eqref{eq:log-partition-dual}. Suppose the robust value in
\eqref{eq:amsebo-kernel-rule-new} is at most $\varepsilon_t$, the concavity
and first-order conditions in Lemma~\ref{lem:amsebo-fixed-bracket} hold,
the selected update is positive at $A^\star(f^\star)$, and the direction
\(
\delta_{(f^\star,A^\star(f^\star))}-\bar\nu_t
\)
is admissible at every reached state.  Then Theorem~\ref{thm:mair-upper} gives
\begin{equation}
\label{eq:amsebo-air-bound-new}
    \E_{f^\star}\Reg_T
    \le
    \gamma\log\frac1{q_1(a^\star)}
    +
    \E_{f^\star}\sum_{t=1}^T\varepsilon_t,
\end{equation}
where $a^\star=A^\star(f^\star)$. This telescope uses a constant $\gamma$;
a varying coefficient requires a separate weighted identity. Averaging
the coordinate bound under the Bayesian reference prior gives
\[
    \E\Reg_T
    \le
    \gamma H(A^\star)+\E\sum_{t=1}^T\varepsilon_t.
\]
The ratio bound instead uses the sequential AIR ratio and the action-index
chain rule, followed by Cauchy--Schwarz. Its entropy term satisfies
$H(A^\star)\le\log n$.

\subsection{Numerical protocol for the hub--cloud experiment}
\label{app:numerics}
The AIR/AMS mean regret under the negative hub truth, not shown in
Figure~\ref{fig:kernel-air-simulation}, is $56.7$ at $T=16384$.
For each $T\in\{512,1024,2048,4096,8192,16384\}$, we take $\alpha=1/5$,
$a=T^{-1/5}$, $N=4T$, unit observation variance, and $\Delta=1/4$.
Since $N\ge T$, the maximal information gain for $0\le s\le T$ is
\begin{equation}
\label{eq:numerical-exact-gamma}
    \Gamma_s(T)
    =
    \frac12\max_{h\in\{0,\ldots,s\}}
    \left\{\log(1+h)+(s-h)\log(1+a^2)\right\}.
\end{equation}
The integer maximizer lies among the clipped integers adjacent to $1/\log(1+a^2)-1$.  The horizon-aware rule uses~\eqref{eq:max-info-beta} without approximation, while the finite-block anytime rule replaces $2\log T$ by $2\log(t\vee2)$.  For the original schedule with $\delta=0.1$, the exact distinct-set information gain is
\[
    \Gamma_s^{\rm SKKS}(T)
    =\frac12\{\log2+(s-1)\log(1+a^2)\},
    \qquad 1\le s\le T,
\]
because an optimal size-$s$ set contains the hub and $s-1$ distinct cloud points.  For a diagonal coordinate of prior variance $d$, after $n$ pulls with observation sum $S$, the implementation uses posterior mean $dS/(1+nd)$ and variance $d/(1+nd)$.  If a sampled cloud coordinate has posterior mean $m$ and standard deviation $s<a$, its score gap from an unused cloud coordinate is $m-\beta_t(a-s)$, which is nonincreasing because $\beta_t$ is nondecreasing.  Once this gap is nonpositive, the coordinate cannot be selected again.  Ties between sampled and unused cloud coordinates are resolved in favor of the unused one; equivalently, under Gaussian noise the pruning gives almost surely the same choices as direct maximization over all $N+2$ actions.

For the localized comparator, let $f^+$ and $f^-$ have hub means $+\Delta$
and $-\Delta$, respectively, and value zero at the anchor. Their optimal
retained actions are the hub and the anchor. Let $q$ be the reference mass
on $f^+$, $r$ the candidate belief's mass on $f^+$, and $p$ the probability
of selecting the hub. The solver uses the two-model Gaussian specialization
of the AIR/AMS saddle of \citet{xu2025bayesian}. Its relation to the EBO
principle of \citet{lattimore2021mirror} is through minimax and KL duality,
not through an implementation of their functional-estimator mirror descent.
The saddle is
\begin{equation}
\label{eq:numerical-air-saddle}
    \inf_{0\le p\le1}\sup_{0\le r\le1}
    \left[
       \Delta\{r(1-p)+(1-r)p\}
       -\gamma p I_r
       -\gamma\kl(r,q)
    \right],
    \qquad
    \gamma=\sqrt{\frac{(1+\Delta^2)T}{\log2}},
\end{equation}
where $I_r$ is the mutual information of the binary Gaussian channel $Z\mid f^\pm\sim N(\pm\Delta,1)$:
\begin{align}
\label{eq:numerical-binary-information}
I_r
&=r\E_{Z\sim N(\Delta,1)}
  \left[-\log\{r+(1-r)e^{-2\Delta Z}\}\right] \\
&\quad +(1-r)\E_{Z\sim N(-\Delta,1)}
  \left[-\log\{(1-r)+re^{2\Delta Z}\}\right].\nonumber
\end{align}
This instance of~\eqref{eq:amsebo-kernel-rule-new} covers only
$\{f^+,f^-\}$, not the full projected RKHS ball. Sion's minimax theorem
reduces it to
\[
    \sup_{0\le r\le1}
    \min\left\{
       \Delta r-\gamma\kl(r,q),
       \Delta(1-r)-\gamma I_r-\gamma\kl(r,q)
    \right\},
\]
which is a concave maximization in one variable. Let
$v_c=\operatorname{logit}(r_c)$, where $r_c$ solves
\[
    \Delta(2r_c-1)+\gamma I_{r_c}=0,
\]
and write $I'(r)=dI_r/dr$.  If $u=\operatorname{logit}(q)$ and $v=\operatorname{logit}(r)$, define
\[
    u_0=v_c-\Delta/\gamma,
    \qquad
    u_1=v_c+\Delta/\gamma+I'(r_c).
\]
The saddle has $p=0$ and $\operatorname{logit}(r)=u+\Delta/\gamma$ for $u\le u_0$; it has $r=r_c$ and $p=(u-u_0)/(u_1-u_0)$ for $u_0<u<u_1$; and it has $p=1$ for $u\ge u_1$, with $v=\operatorname{logit}(r)$ determined by
\[
    v-u+\Delta/\gamma+I'\!\left(\operatorname{logit}^{-1}(v)\right)=0.
\]
After an anchor pull, $u^+=v$.  After a hub observation $Z=z$, Gaussian Bayes updating gives $u^+=v+2\Delta z$.

The numerical policy evaluates~\eqref{eq:numerical-binary-information}
by $48$-node Gauss--Hermite quadrature and solves scalar roots to absolute
tolerance $2\times10^{-13}$. It tabulates the $p=1$ branch at $5001$
log-odds points, interpolates linearly, and uses the limiting shift
outside the tabulated range. As a diagnostic, for each
$\sigma\in\{+,-\}$ the code uses $96$-node quadrature to evaluate
\[
    \Delta_\sigma(p)
    -\gamma\E_\sigma
       \log\frac{q^+(Y_\sigma\mid A,Z)}{q(Y_\sigma)}
    -\frac{1+\Delta^2}{\gamma}.
\]
Here $\Delta_\sigma(p)$ and $Y_\sigma$ are the one-step regret and
optimal-action index under $f^\sigma$. The check covers more than ten
thousand reference log-odds and every reference state visited in the
Monte Carlo runs. The largest values on the grid and at visited states
are approximately $-3.584\times10^{-3}$ over the displayed horizons.
The largest tabulated first-order residual is $6.81\times10^{-8}$.
No bracket violation was detected at this precision; this check does not
prove a uniform bound.

Each algorithm--truth pair uses $200$ independent NumPy streams spawned from the master seed $20260717$; the streams of the original horizon-aware/AIR experiment are retained unchanged, and separate streams are used for the two added GP--UCB schedules.  Shaded bands (left) and error bars (right) show two Monte Carlo standard errors.  The right-panel truncation curve is the deterministic normalized value $(aT)/T=a$, not a Monte Carlo estimate.  The one-hot cloud truths are external to the binary saddle class; their role is only to evaluate the deterministic loss incurred by the predetermined truncation.

\section{Gaussian MAB: finite-index matching}
\label{app:mab}

Consider a multi-armed bandit with $K\ge3$ Gaussian arms and unit noise
variance. The model class $\Omega_{\rm MAB}\subseteq\mathbb R^K$ consists of
mean vectors $\omega=(\theta_\omega(1),\ldots,\theta_\omega(K))$.
Pulling arm $a$ gives an observation with law $N(\theta_\omega(a),1)$. With
a fixed tie-breaking rule, write
\[
    A^\star(\omega)\in\arg\max_{a\in[K]}\theta_\omega(a)
\]
and take the optimal arm as the information index:
\[
    \chi(\omega)=A^\star(\omega)\in[K].
\]
The random environment $\Omega$ is the full mean vector, but the lower bound
tracks only
\[
    Y=\chi(\Omega)=A^\star(\Omega).
\]
For $\Delta>0$, let $\mu_\Delta$ be uniform on the spike environments
\[
    \omega^j=\Delta e_j,\qquad j\in[K],
\]
where $e_j$ is the $j$th coordinate vector. If $J\sim\unif([K])$ and
$\Omega=\omega^J$, the unique optimal arm on the prior support is
\[
    Y=\chi(\Omega)=J.
\]
The one-step regret in environment $\omega^j$ is
\[
    \ell_{\omega^j}(a)
    =
    \max_b\theta_{\omega^j}(b)-\theta_{\omega^j}(a)
    =
    \Delta\1\{a\ne j\}.
\]
For brevity write $\ell_j=\ell_{\omega^j}$ on this subfamily.

\begin{lemma}[MAB ghost entropy]
\label{lem:mab-ghost}
Let $Y=\chi(\Omega)$ under the prior $\Omega\sim\mu_\Delta$.  For every algorithm and every reference history $H_T'$ independent of $Y\sim\unif([K])$,
\[
    \Pp\!\left(
        \bar L_Y(H_T')\le \frac\Delta2
    \right)
    \le
    \frac2K,
\]
where $\bar L_y(H_T')=T^{-1}\sum_{t=1}^T \ell_y(A_t')$ and $A_t'$ denotes the arm selected in the reference history at time $t$.
\end{lemma}

\proofinappendix{app:proof-lem:mab-ghost}

\begin{lemma}[MAB information bound]
\label{lem:mab-info}
There is a universal constant $c_0>0$ such that if $\Delta^2T/K\le c_0$, then every adaptive algorithm satisfies
\[
    I(Y;H_T)
    \le
    \frac14\log(K/2),
\]
where $Y=\chi(\Omega)$ and $\Omega\sim\mu_\Delta$.
\end{lemma}

\proofinappendix{app:proof-lem:mab-info}

\begin{theorem}[MAB lower bound]
\label{thm:mab}
For Gaussian $K$-armed bandits with unit noise variance, suppose the mean-vector
class $\Omega_{\rm MAB}$ contains the spike instances
$\{\Delta e_j:j\in[K]\}$ at the scale chosen in the proof. Then
\[
    \mathfrak R_T^\star\ge c\sqrt{KT}
\]
for a universal constant $c>0$ whenever $T\ge K\ge3$. In particular, the
spikes lie in $[0,1]^K$ for a sufficiently small constant in the choice of
$\Delta$.
\end{theorem}

\proofinappendix{app:proof-thm:mab}

\begin{proposition}[MAB upper bound and minimax-rate matching]
\label{prop:mab-upper-match}
For Gaussian $K$-armed bandits with unit noise variance and mean vectors in
$[0,1]^K$, a nonanticipating algorithm such as a sub-Gaussian version of MOSS
satisfies
\[
    \sup_{\omega}\E_\omega^\Alg L_\omega(H_T)\le C\sqrt{KT}
\]
for a universal constant $C$ and all $T\ge K\ge3$. For another fixed bounded
cube, the constant may depend on the range of its means. By
Theorem~\ref{thm:mab},
\[
    \mathfrak R_T^\star(\Omega_{\rm MAB})=\Theta(\sqrt{KT})
\]
on every class contained in $[0,1]^K$ that contains the spike family.
In the notation of Section~\ref{sec:info-sandwich}, the hard prior has $M=K$
index values and code length \(L_0=\log K\) on the spike family.
Lemma~\ref{lem:mab-ghost} gives the algorithm-uniform ghost entropy bound
\[
E_{\mu_\Delta,r}^{\star}=\log\frac1{p_{\mu_\Delta,r}^{\star}}\ge\log(K/2)
\]
at radius $r=\Delta/2\asymp\sqrt{K/T}$. Thus the entropy ratio
$L_0/E_{\mu_\Delta,r}^{\star}$ is bounded by a universal constant for \(K\ge3\).
The matching regret rates do not by themselves verify the calibration in
Definition~\ref{def:upper-at-lower-calibration} for MOSS.
\end{proposition}

\proofinappendix{app:proof-prop:mab-upper-match}

\section{Hypercube linear bandits: indexed matching}
\label{app:linear}

Consider stochastic linear bandits with action set
\[
    \mathcal A=\{a\in\R^d:\|a\|_2\le1\}.
\]
The parameter class is also the Euclidean unit ball:
\[
   \Omega_{\rm lin}=\{\omega\in\R^d:\|\omega\|_2\le1\}.
\]
At time $t$, after choosing $A_t\in\mathcal A$, the learner observes
\[
    O_t=\langle \omega,A_t\rangle+\xi_t,
    \qquad
    \xi_t\sim N(0,1),
\]
where the noises are independent over time and independent of the
learner's external randomization.  For \(\omega\ne0\), the unique optimal
action over \(\mathcal A\) is
\[
    A^\star(\omega)=\frac{\omega}{\|\omega\|_2}.
\]
At $\omega=0$, fix any deterministic tie-breaking action; this case is not
used by the lower-bound prior.  The information index is the optimal action
\[
    \chi(\omega)=A^\star(\omega),\qquad Y=\chi(\Omega),
\]
not the parameter vector $\Omega$. In particular, the index does not
distinguish parameters with the same optimal action.

We recover the standard \(\Omega(d\sqrt T)\) minimax lower bound by applying
the quantile theorem to a hypercube prior and bounding the information gained
under adaptive sampling. \citet[Section~3.3]{chen2024unified} give a related
interactive Fano proof using a Gaussian prior and a unit-sphere calculation.
Both arguments recover the sharp order by indexing the optimal direction.
Model-indexed DEC arguments use a different target and can recover the same
order when their localization reflects this geometry.

Let
\[
    V\sim\unif(\{\pm1\}^d),
    \qquad
    \Omega=\Delta V.
\]
Assume $0<\Delta\le1/\sqrt d$, so that this prior is supported on
$\Omega_{\rm lin}$.  On this prior support,
\begin{align}
    \label{eq:optimal-action}
    Y=\chi(\Omega)=A^\star(\Omega)=\frac{V}{\sqrt d}.
\end{align}
The map \(v\mapsto v/\sqrt d\) is one-to-one on \(\{\pm1\}^d\).  Hence, for
this prior only, \(V\) may be used as a code for the action index \(Y\), and
\[
    I(Y;H_T)=I(V;H_T).
\]
For an index value \(y=v/\sqrt d\), write \(\ell_v\) for the regret under
the corresponding environment \(\omega=\Delta v\).  The one-step regret of
action \(a\in\R^d\), \(\|a\|_2\le1\), is
\[
    \ell_v(a)
    =
    \Delta\sqrt d-\Delta\langle v,a\rangle.
\]

\begin{lemma}[Linear ghost entropy]
\label{lem:linear-ghost}
There is a universal constant \(c_{\rm gh}>0\) such that every deterministic
reference action sequence \(a_1',\ldots,a_T'\), with
\(\|a_t'\|_2\le1\), satisfies
\[
    \Pp_{V\sim\unif(\{\pm1\}^d)}
    \left(
        \frac1T\sum_{t=1}^T
        \bigl(\Delta\sqrt d-\Delta\langle V,a_t'\rangle\bigr)
        \le
        \frac{\Delta\sqrt d}{4}
    \right)
    \le
    \exp(-c_{\rm gh}d).
\]
Consequently, the same bound holds after averaging over any random
reference history \(H_T'\) whose induced action sequence is independent of
the action index \(Y\), equivalently independent of \(V\) under the
hypercube prior.
\end{lemma}

\proofinappendix{app:proof-lem:linear-ghost}

\begin{lemma}[Adaptive hypercube information capacity]
\label{lem:linear-info}
For the hypercube prior above, every possibly randomized adaptive
algorithm satisfies
\[
   I(Y;H_T)=I(V;H_T)
    \le
    \Delta^2 T ,
\]
where \(H_T=(A_1,O_1,\ldots,A_T,O_T)\).
\end{lemma}

\proofinappendix{app:proof-lem:linear-info}

\begin{remark}[Why the fixed-design log-determinant argument is not used]
For a Gaussian prior on the parameter vector, posterior conjugacy gives the
adaptive identity
\[
    I(\Theta;H_T)
    =
    \frac12
    \E
    \log\det\left(
        I+\Sigma_0^{1/2}
        \Bigl(\sum_{t=1}^T A_tA_t^\top\Bigr)
        \Sigma_0^{1/2}
    \right).
\]
One can bound the right-hand side using the trace constraint on the design.
For the hypercube prior, the posterior is not Gaussian. Nor can one treat
the realized adaptive design as fixed in advance: it depends on past
observations and can itself carry information about \(Y\).
Lemma~\ref{lem:linear-info} instead bounds \(I(Y;H_T)\) directly with the
coordinatewise KL calculation.
\end{remark}

\begin{theorem}[Linear bandit lower bound]
\label{thm:linear}
For stochastic linear bandits with action set \(\{a\in\R^d:\|a\|_2\le1\}\),
unit Gaussian noise, and unit-ball parameter class
\(\Omega_{\rm lin}=\{\omega\in\R^d:\|\omega\|_2\le1\}\),
\[
    \mathfrak R_T^\star\ge c\,d\sqrt T
\]
for a universal constant \(c>0\), whenever \(T\ge d^2\).
\end{theorem}

\proofinappendix{app:proof-thm:linear}

\begin{proposition}[Linear-bandit upper bound and near matching]
\label{prop:linear-upper-match}
For stochastic linear bandits on the Euclidean unit ball with unit Gaussian noise and unit-ball parameter class, there are algorithms based on confidence ellipsoids, such as OFUL, with regret
\[
    \sup_{\omega\in\Omega_{\rm lin}}\E_\omega^\Alg L_\omega(H_T)
    \le C d\sqrt{T}\,\polylog(T,d)
\]
for a universal constant $C$ and all $T\ge d$. In the regime $T\ge d^2$ of
Theorem~\ref{thm:linear}, this matches its lower bound up to logarithmic
factors. The hypercube prior has
$M=2^d$ index values and code length \(L_0=d\log2\) on its support.
At radius $r=\Delta\sqrt d/4\asymp d/\sqrt T$,
Lemma~\ref{lem:linear-ghost} gives the algorithm-uniform ghost entropy
$E_r^\star:=\log(1/p_r^\star)\ge c_{\rm gh}d$, so $L_0/E_r^\star$ is bounded
by a constant. The rate comparison does not by itself verify the
upper-at-lower-entropy calibration of the exact log-potential Bellman program
on the full unit-ball class.
\end{proposition}

\proofinappendix{app:proof-prop:linear-upper-match}

This Euclidean-ball minimax calculation concerns a different geometry and
criterion from the results discussed in Section~\ref{sec:kernel-separation}.
Corollary~\ref{cor:finite-linucb-separation} separates a particular LinUCB
rule from the minimax cumulative-regret rate on a horizon-dependent
small-cloud action set. \citet{maiti2026power} separate adaptive from
nonadaptive sampling for fixed-confidence best-arm identification on a union
of orthogonal balls.

\section{Proofs}
\label{app:proofs}

\subsection{State compression and the information-complexity sandwich}\label{app:sandwich-proof}

\subsubsection{Proof of Proposition~\ref{prop:state-compression-info}}
\label{app:proof-prop:state-compression-info}
\begin{proof}
Histories with the same \(S_t\) have the same feasible actions, comparison
set, one-step cost, transition law, and continuation value, so they give the
same Bellman optimization. The information inequality follows by data
processing under \(S_t=\phi_t(H_{t-1})\). When \(I(Y;H_{t-1})\) is finite,
the chain rule shows that equality holds precisely when
\(I(Y;H_{t-1}\mid S_t)=0\).
\end{proof}

\subsubsection{Proof of Lemma~\ref{lem:admissible-implies-nondegenerate}}
\label{app:proof-lem:admissible-implies-nondegenerate}
\begin{proof}
Since \(C_{\mu}\ge0\), \eqref{eq:bf-admissible-condition} gives
\(\underline E_{\mu,r}\ge 2\log2\). Thus
\(p_{\mu,r}^{\star}=e^{-E_{\mu,r}^{\star}}\le
e^{-\underline E_{\mu,r}}\le1/4\).
Conversely, \(p_{\mu,r}^{\star}\le\varrho\) is equivalent to
\(E_{\mu,r}^{\star}\ge\log(1/\varrho)\) and gives no bound on \(C_\mu\).
\end{proof}

\subsubsection{Proof of Proposition~\ref{prop:entropy-gap}}
\label{app:proof-prop:entropy-gap}
\begin{proof}
The uniform representative prior gives \(q_{1,\mu_h}(y)=1/M\), hence the code length under this common upper/lower marginal is \(\log M\).  If \(\mu_h\) is Bellman--Fano admissible, Lemma~\ref{lem:admissible-implies-nondegenerate} gives \(E_{\mu_h,r}^{\star}\ge2\log2\), proving \eqref{eq:log-gap-from-admissibility}.  Under the multiplicity condition, every algorithm satisfies
\[
    p_{\mu_h,r}^{\Alg}
    =
    \E_{H_T'\sim\bar\Pp_{\mu_h}^{\Alg}}\frac{|G_r(H_T')|}{M}
    \le
    \frac{m_r}{M}.
\]
Taking the supremum over \(\Alg\) gives \(p_{\mu_h,r}^{\star}\le m_r/M\), and hence \(E_{\mu_h,r}^{\star}\ge\log(M/m_r)\). This proves \eqref{eq:localized-entropy-gap}.  The final statement follows from \(E_{\mu_h,r}^{\star}\ge\log(1/\varrho)\).
\end{proof}

\subsubsection{Proof of Lemma~\ref{lem:growth-needed}}
\label{app:proof-lem:growth-needed}
\begin{proof}
Set \(\mathsf U(x)=1\) for \(x<L\) and \(\mathsf U(x)=\max\{1,B\}\) for
\(x\ge L\). Then \(\mathsf U\) is nondecreasing, \(\mathsf U(E)=1\), and
\(\mathsf U(L)\ge B\).
\end{proof}

\subsubsection{Proof of Proposition~\ref{prop:scale-closing-diagnostics}}
\label{app:proof-prop:scale-closing-diagnostics}
\begin{proof}
Condition (i) is \eqref{eq:upper-at-lower-entropy}. Condition (ii) bounds the
objective in \eqref{eq:upper-budget-value} at one \(\gamma>0\), so it also
bounds the infimum in (i). Dividing (i) by \(T\) gives (iii).
\end{proof}

\subsubsection{Proof of Theorem~\ref{thm:interactive-info-risk-sandwich}}
\label{app:proof-thm:interactive-info-risk-sandwich}
\begin{proof}
The upper inequality is Theorem~\ref{thm:frequentist-info-risk-upper}, optimized over \(\gamma\), with the uniform coordinate code length bounded by \(L_0\) and the remaining errors included in \(\Delta_T^\gamma\).  The lower inequality is Theorem~\ref{thm:bellman-lower} applied to the Bellman--Fano admissible prior, followed by Yao's principle because \(\mu\) is supported on \(\Omega_0\).  For the final comparison, set
\[
    a=\max\left\{1,\frac{L_0}{E_{\mu,r}^{\star}}\right\}.
\]
Since \(\mathsf U_T\) is nondecreasing and \(aE_{\mu,r}^{\star}\ge L_0\), Assumption~\ref{ass:upper-growth} and \eqref{eq:upper-at-lower-entropy} give
\[
    \mathsf U_T(L_0)
    \le
    \mathsf U_T(aE_{\mu,r}^{\star})
    \le
    C_{\rm gr}a^\beta\mathsf U_T(E_{\mu,r}^{\star})
    \le
    C_{\rm gr}C_{\rm base}a^\beta Tr.
\]
Combining this with \(\mathfrak R_T^\star(\Omega_0)\ge Tr/2\) gives the regret-ratio bound.  The finite-index statements are Proposition~\ref{prop:entropy-gap}.
\end{proof}

\subsection{Indexed AIR/MAIR identities and upper bounds}

\subsubsection{Proof of Proposition~\ref{prop:index-chain}}
\label{app:proof-prop:index-chain}
\begin{proof}
Under the stated assumption on $O_0$, the mutual-information chain rule gives
\[
    I_\mu(Y;H_T)=\sum_{t=1}^T I_\mu(Y;A_t,O_t\mid H_{t-1}).
\]
The algorithm samples $A_t$ from a decision kernel given $H_{t-1}$, so
$A_t$ is conditionally independent of $Y$. Therefore
\[
    I_\mu(Y;A_t,O_t\mid H_{t-1})=I_\mu(Y;O_t\mid H_{t-1},A_t).
\]
Conditioning further on $A_t=a$, the law of $O_t$ given $Y=y$ is $P_{S_t,p_t,a}^y$, while the posterior predictive law is $P_{S_t,p_t,a}$. Hence
\[
    I_\mu(Y;O_t\mid H_{t-1},A_t=a)
    =
	    \int \KL(P_{S_t,p_t,a}^y\|P_{S_t,p_t,a})q_{S_t}(dy).
\]
Averaging over the action and over the marginal history law under the
Bayesian mixture gives the reference-history expression.
\end{proof}

\subsubsection{Proof of Lemma~\ref{lem:air-gradient-bracket}}
\label{app:proof-lem:air-gradient-bracket}
\begin{proof}
For fixed $s,p,a$, write the joint mixture of $(Y,O)$ under $\nu$ as $w_{y,o}$ and the observation mixture as $w_o=\sum_y w_{y,o}$. In the dominated case, the sums below are interpreted as integrals with respect to the fixed dominating measures. The information and reference term can be written as
\[
    I_\nu(Y;O\mid s,p,a)+\KL(q_\nu\|q)
    =
    \sum_{y,o}w_{y,o}\log\frac{w_{y,o}}{w_o q(y)}.
\]
Differentiating with respect to the mass at $(\omega,y)$ gives
\[
    \E_{O\sim P_{\omega,t}(\cdot\mid s,p,a)}
    \log\frac{q_\nu^+(y\mid s,p,a,O)}{q(y)},
\]
because the $+1$ terms from the numerator and denominator cancel.  Averaging over $a\sim p$ and adding the derivative of the linear loss term proves \eqref{eq:air-gradient}.  Pairing the gradient with $\delta_{(\omega,y)}-\nu$ cancels the posterior-averaged terms and gives \eqref{eq:air-bracket}.
\end{proof}

\subsubsection{Proof of Theorem~\ref{thm:air-regret-identity}}
\label{app:proof-thm:air-regret-identity}
\begin{proof}
Condition on $H_{t-1}$.  By the update rule \eqref{eq:air-reference-update},
\[
    \E\!\left[
        \log\frac{q_{t+1}(y^\star)}{q_t(y^\star)}
        \middle|H_{t-1}
    \right]
    =
    \E_{a\sim p_t,\,O\sim P_{\omega^\star,t}(\cdot\mid S_t,p_t,a)}
    \log\frac{q_{\nu_t}^+(y^\star\mid S_t,p_t,a,O)}{q_t(y^\star)}.
\]
Combining this equality with \eqref{eq:air-bracket} gives the one-step identity between expected loss, bracket, and log-ratio increment.  Summing over $t$ telescopes the logarithm.
\end{proof}

\subsubsection{Proof of Lemma~\ref{lem:mair-gradient}}
\label{app:proof-lem:mair-gradient}
\begin{proof}
For the derivative and the first two claims, apply
Lemma~\ref{lem:air-gradient-bracket} with the environment as the index. For
the square-root claim, fix $\bar a$ and write
\[
Z_{\bar a}(\bar o):=\int\sqrt{p_{\omega',\bar a}(\bar o)}\,\rho(d\omega').
\]
Then
\[
\log\frac{d\mu_{\bar a,\bar o}^{1/2}}{d\rho}(\omega^\star)
=\frac12\log p_{\omega^\star,\bar a}(\bar o)-\log Z_{\bar a}(\bar o).
\]
Under $\bar o\sim P_{\omega^\star,s,p,\bar a}$, Jensen's inequality gives
\begin{align*}
\E\log\frac{d\mu_{\bar a,\bar o}^{1/2}}{d\rho}(\omega^\star)
&=-\E\log\frac{Z_{\bar a}(\bar o)}{\sqrt{p_{\omega^\star,\bar a}(\bar o)}}\\
&\ge -\log\int\sqrt{p_{\omega^\star,\bar a}(o)}Z_{\bar a}(o)d\nu_{s,p,\bar a}(o)\\
&=-\log\left(1-\E_{\omega\sim\rho}h^2(P_{\omega^\star,s,p,\bar a},P_{\omega,s,p,\bar a})\right)\\
&\ge \E_{\omega\sim\rho}h^2(P_{\omega^\star,s,p,\bar a},P_{\omega,s,p,\bar a}).
\end{align*}
Averaging over $\bar a\sim p$, substituting into \eqref{eq:mair-gradient-bracket}, and dropping the nonnegative KL term proves \eqref{eq:sqrt-posterior-bracket}.
\end{proof}

\subsubsection{Proof of Theorem~\ref{thm:mair-upper}}
\label{app:proof-thm:mair-upper}
\begin{proof}
The update \(q_{t+1}=q_\nu^+(\cdot\mid s,p,A,O)\) makes the conditional
change in the coordinate log loss equal to
\(-\gamma\mathsf G_\chi(s,p,\nu;\omega^\star)\).
Substitute the bracket inequality into \eqref{eq:potential-identity}.
\end{proof}

\subsubsection{Proof of Theorem~\ref{thm:frequentist-info-risk-upper}}
\label{app:proof-thm:frequentist-info-risk-upper}
\begin{proof}
By sure calibration, the fixed truth is one of the environments in the robust supremum at every reached state.  Therefore the displayed approximate minimization implies almost surely
\[
   \ell_{\omega^\star}(S_t,p_{u_t})
    +\E_{\omega^\star}[W_{t+1}^\gamma(S^+)\mid S_t,u_t]
    -W_t^\gamma(S_t)
    -\gamma\mathsf G_\chi(S_t,u_t;\omega^\star)
    \le \varepsilon_t .
\]
Apply the log-potential telescope \eqref{eq:potential-identity} with $V_t=W_t^\gamma$.  Since $W_{T+1}^\gamma=0$ and the terminal index log loss is nonnegative for a finite or countable index, the result follows.
\end{proof}

\subsection{Bellman controls and quantile-index lower bounds}

\subsubsection{Proof of Theorem~\ref{thm:info-potential-dp-upper}}
\label{app:proof-thm:info-potential-dp-upper}
\begin{proof}
Rearrange the displayed inequality and telescope $U_t^\gamma$.  If $O_0$ is deterministic or independent of $Y$, Proposition~\ref{prop:index-chain} identifies $\E\sum_t\calI_\chi(S_t,p_t)$ with $I_\mu(Y;H_T)$.  If $O_0$ is informative, the sum is instead $I_\mu(Y;H_T\mid O_0)\le I_\mu(Y;H_T)$.  In either case $I_\mu(Y;H_T)\le H_\mu(Y)$ gives the displayed bound.  The equivalence with \eqref{eq:phi-bellman-relaxation} follows from the posterior entropy-drop identity.
\end{proof}

\subsubsection{Proof of Lemma~\ref{lem:ucb-bracket}}
\label{app:proof-lem:ucb-bracket}
\begin{proof}
The optimism/calibration hypothesis \eqref{eq:generic-optimism-width} and Young's inequality give
\[
    c_{\rm opt}\beta_t\sigma_t(a_t)
    \le
    \frac{c_{\rm opt}^2\beta_t^2}{4\gamma}
    +
    \gamma\sigma_t^2(a_t).
\]
This proves \eqref{eq:generic-ucb-offset}. Subtracting
$\gamma J_\chi(S_t,\delta_{a_t};\omega^\star)$ from both sides gives
\eqref{eq:generic-ucb-gradient-bracket}; summing it over \(t\) gives the
claimed comparison of the two information sums.
\end{proof}

\subsubsection{Proof of Lemma~\ref{lem:amsebo-fixed-bracket}}
\label{app:proof-lem:amsebo-fixed-bracket}
\begin{proof}
By Lemma~\ref{lem:air-gradient-bracket}, the fixed-truth bracket is the AIR
value at \(\bar\nu_t\) plus the directional derivative toward
\(\delta_{(\omega^\star,y^\star)}\). Since \(\bar\nu_t\) maximizes the AIR
objective and that direction is admissible, the derivative is nonpositive.
The bracket is therefore at most the robust value, which is bounded by
\(\varepsilon_t\). Apply the log-potential telescope.
\end{proof}

\subsubsection{Proof of Theorem~\ref{thm:quantile-index}}
\label{app:proof-thm:quantile-index}
\begin{proof}
Consider the true joint law of $(Y,H_T)$ and the ghost law of $(Y,H_T')$.  Both have the same marginal law of $Y$, while $Y$ and $H_T'$ are independent under the ghost law.  Hence
\[
    \KL\bigl(\mathcal L(Y,H_T)\|\mathcal L(Y,H_T')\bigr)
    =
    I_\mu(Y;H_T).
\]
Data processing applied to the indicator of the event $\{\bar L_\chi(Y,\cdot)\le r\}$ gives
\[
    \kl(q_r^\Alg,p_r^\Alg)
    \le
    I_\mu(Y;H_T).
\]
Proposition~\ref{prop:index-chain} identifies the right-hand side with $C_\chi^\Alg(\mu)=T\bar C_\chi^\Alg(\mu)$, proving \eqref{eq:binary-dp}.  The definition of $\psi$ gives $q_r^\Alg\le\psi(p_r^\Alg,T\bar C_\chi^\Alg)$. By nonnegativity of the cumulative indexed loss and by \eqref{eq:index-loss-equality},
\[
\begin{aligned}
    \E_{\Omega,H_T}L_\Omega(H_T)
    &=\E_{Y,H_T}L_\chi(Y,H_T) \\
    &\ge Tr\,\Pp\{\bar L_\chi(Y,H_T)>r\} \\
    &=Tr(1-q_r^\Alg),
\end{aligned}
\]
which proves \eqref{eq:quantile-exact}.  If \eqref{eq:quantile-kl-condition} holds, then $q_r^\Alg\le1/2$, giving \eqref{eq:average-lower}.  Finally,
\[
    \kl\!\left(\frac12,p\right)
    =
    \frac12\log\frac1{4p(1-p)}
    \ge
    \frac12\log\frac1p-\log2,
\]
so \eqref{eq:average-readable} implies \eqref{eq:quantile-kl-condition}.
\end{proof}

\subsubsection{Proof of Theorem~\ref{thm:bellman-lower}}
\label{app:proof-thm:bellman-lower}
\begin{proof}
For every algorithm, \eqref{eq:capacity-bound} gives $T\bar C_\chi^\Alg(\mu)=C_\chi^\Alg(\mu)\le\overline C_1(s_1)$, and \eqref{eq:ghost-mass-bound} gives $p_r^\Alg(\mu,\chi)\le e^{-\underline{\calE}_1^r(\mu,\chi)}$.  Substitute these two inequalities into Theorem~\ref{thm:quantile-index}, using the monotonicity of \(\psi\) in both arguments.  The minimax statement follows from Yao's principle.
\end{proof}

\subsection{Finite-marginal kernel-bandit upper bounds}

\subsubsection{Proof of Theorem~\ref{thm:gpucb-realized}}
\label{app:proof-thm:gpucb-realized}
\begin{proof}
Let $x^\star\in\argmax_{x\in\calX}f^\star(x)$.  On $\mathcal E_{\rm GP}$, calibration at $x^\star$ and the optimistic choice of $X_t$ give
\[
    f^\star(x^\star)
    \le m_t(x^\star)+w_t(x^\star)
    \le m_t(X_t)+w_t(X_t).
\]
Calibration at $X_t$ gives $m_t(X_t)\le f^\star(X_t)+w_t(X_t)$.  Hence
\[
    f^\star(x^\star)-f^\star(X_t)
    \le 2w_t(X_t)
    =2\beta_t s_t(X_t)
    \le 2\beta_T s_t(X_t).
\]
Summing and applying Cauchy--Schwarz yields
\[
    \Reg_T(f^\star)
    \le2\beta_T\sqrt{T\sum_{t=1}^T s_t^2(X_t)}.
\]
For $u\in[0,\kappa^2/\lambda]$, monotonicity of $u/\log(1+u)$ on a bounded interval gives
\[
    \lambda u
    \le
    \frac{\kappa^2}{\log(1+\kappa^2/\lambda)}\log(1+u).
\]
Applying this with $u=s_t^2(X_t)/\lambda$ and using \eqref{eq:realized-gp-info} proves \eqref{eq:gpucb-pathwise-bound-new}.  The expectation bound adds the trivial regret bound $T$ on $\mathcal E_{\rm GP}^c$ and uses Jensen's inequality.
\end{proof}

\subsubsection{Proof of Theorem~\ref{thm:finite-marginal-air-upper}}
\label{app:proof-thm:finite-marginal-air-upper}
\begin{proof}
Let $\Theta=\Theta_{\mathcal R}\subset[-L,L]^{\mathcal R}$ be the compact class of retained reward vectors.  To avoid a discontinuous tie-breaking graph in the saddle argument, use the closed optimal relation
\[
    \widetilde\Theta
    :=
    \{(\theta,y):\theta\in\Theta,\ y\in\argmax_{x\in\mathcal R}\theta_x\}.
\]
This is a compact subset of $\Theta\times\mathcal R$.  Let $\calY_\Theta$ be the set of actions that occur as the second coordinate of some pair in $\widetilde\Theta$, and put $K_0=|\calY_\Theta|\le K$.  The fixed tie-breaking rule used for the true index selects an element of this relation.  Candidate pair beliefs range over the weakly compact convex set $\Delta(\widetilde\Theta)$.

Fix $q\in\operatorname{int}\Delta(\calY_\Theta)$ and $\gamma>0$.  For $\nu\in\Delta(\widetilde\Theta)$ and $p\in\Delta(\mathcal R)$, write
\[
    \mathfrak A_{q,\gamma}(p,\nu)
    =
    \Delta_\nu(p)
    -\gamma\calI_\nu(Y;Z_A\mid A)
    -\gamma\KL(\nu_Y\|q),
\]
where $A\sim p$ is independent of $(\theta,Y)\sim\nu$, $Z_A=\theta_A+\xi$ with $\xi\sim N(0,\lambda)$, and
\[
    \Delta_\nu(p)
    :=
    \E_{\nu,p}[\theta_Y-\theta_A].
\]
The loss is affine in each of $p$ and $\nu$ separately.  Moreover,
\[
    \calI_\nu(Y;Z_A\mid A)+\KL(\nu_Y\|q)
    =
    \sum_{a\in\mathcal R}p(a)
    \KL\!\left(P_{Y,Z_a}^\nu\,\middle\|\,q\otimes P_{Z_a}^\nu\right).
\]
The map $\theta\mapsto N(\theta_a,\lambda)$ is weakly continuous, and all its means range over a compact interval.  Joint convexity and lower semicontinuity of relative entropy therefore make $\mathfrak A_{q,\gamma}$ concave and upper semicontinuous in $\nu$ and affine and continuous in $p$.  Both feasible sets are compact, so Sion's theorem gives
\begin{equation}
\label{eq:finite-air-sion}
    \inf_p\sup_\nu\mathfrak A_{q,\gamma}(p,\nu)
    =
    \sup_\nu\inf_p\mathfrak A_{q,\gamma}(p,\nu).
\end{equation}

We bound the right-hand side.  Fix $\nu$, let $r=\nu_Y$, take $p=r$ on $\calY_\Theta$ and zero on the remaining actions, and define
\[
    d_i
    :=
    \E_\nu[\theta_i\mid Y=i]-\E_\nu\theta_i,
    \qquad i\in\calY_\Theta,
\]
with $d_i=0$ when $r(i)=0$.  The marginal law of $Z_i$ is $(\lambda+L^2)$-sub-Gaussian about its mean: the Gaussian noise contributes variance proxy $\lambda$, while Hoeffding's lemma contributes $L^2$ for the random mean in $[-L,L]$.  The entropy variational formula consequently gives
\[
    \KL(P_{Z_i\mid Y=i}^\nu\|P_{Z_i}^\nu)
    \ge
    \frac{d_i^2}{2(\lambda+L^2)}.
\]
It follows that
\begin{align*}
    \calI_\nu(Y;Z_A\mid A,p=r)
    &\ge
    \frac{1}{2(\lambda+L^2)}
    \sum_{i\in\calY_\Theta}r(i)^2d_i^2,\\
    \Delta_\nu(r)
    &=
    \sum_{i\in\calY_\Theta}r(i)d_i.
\end{align*}
By Cauchy--Schwarz,
\begin{equation}
\label{eq:finite-air-information-ratio}
    \Delta_\nu(r)^2
    \le
    2(\lambda+L^2)K_0\,
    \calI_\nu(Y;Z_A\mid A,p=r).
\end{equation}
Dropping the nonpositive term $-\gamma\KL(r\|q)$ and maximizing
$\sqrt{2(\lambda+L^2)K_0 z}-\gamma z$ over $z\ge0$ now yields
\begin{equation}
\label{eq:finite-air-robust-value}
    \inf_p\sup_\nu\mathfrak A_{q,\gamma}(p,\nu)
    \le
    \frac{(\lambda+L^2)K_0}{2\gamma}
    \le
    \frac{(\lambda+L^2)K}{2\gamma}
    =:c_\gamma.
\end{equation}

Take a saddle pair $(p_q,\bar\nu_q)$ in \eqref{eq:finite-air-sion} and update $q$ by the posterior optimal-action marginal generated by $\bar\nu_q$.  The identity
\[
    -\calI_\nu(Y;Z_A\mid A)-\KL(\nu_Y\|q)
    =
    H_\nu(Y\mid A,Z_A)+\E_\nu\log q(Y)
\]
shows that a maximizing belief assigns positive mass to every attainable index.  Indeed, suppose $\bar\nu_{q,Y}(y)=0$, choose $(\theta^y,y)\in\widetilde\Theta$, and put
\[
    \nu_\epsilon
    =(1-\epsilon)\bar\nu_q+\epsilon\delta_{(\theta^y,y)}.
\]
Let $B$ indicate the new mixture component and let $h(\epsilon)$ be binary entropy.  Since the new label did not occur under $\bar\nu_q$,
\begin{align*}
    H_{\nu_\epsilon}(Y\mid A,Z_A)
    &=(1-\epsilon)H_{\bar\nu_q}(Y\mid A,Z_A)
      +H(B\mid A,Z_A),\\
    H(B\mid A,Z_A)
    &\ge
      h(\epsilon)
      -\epsilon\sup_{a\in\mathcal R}
       \KL\!\left(N(\theta_a^y,\lambda)\,\middle\|\,P_{Z_a}^{\bar\nu_q}\right)\\
    &\ge h(\epsilon)-\frac{2L^2}{\lambda}\epsilon.
\end{align*}
The last step uses convexity in the second argument of KL and the uniform pairwise Gaussian bound.  Thus the conditional-entropy contribution increases by $\epsilon\log(1/\epsilon)-O(\epsilon)$, while the loss and $\E\log q(Y)$ change by only $O(\epsilon)$.  This contradicts maximality.  Hence the generated posterior coordinate remains positive for every feasible fixed-truth index after every observation.

Because $\lambda>0$, the predictive mixture densities are strictly positive.  Bounded Gaussian means and the preceding KL bound also make the relevant log likelihood ratios integrable.  The one-sided Gateaux derivative toward every $\delta_{(\theta,y)}-\bar\nu_q$ therefore exists, so Lemma~\ref{lem:air-gradient-bracket} applies to this dominated continuous-belief problem.

At the maximizing belief, the directional derivative of the concave AIR functional toward any $(\theta,y)\in\widetilde\Theta$ is nonpositive.  Lemma~\ref{lem:air-gradient-bracket} and \eqref{eq:finite-air-robust-value} therefore give
\begin{equation}
\label{eq:finite-air-fixed-truth-bracket}
    \Delta_\theta(p_q)
    -\gamma\E_\theta
       \log\frac{q^+(y\mid A,Z_A)}{q(y)}
    \le c_\gamma
    \qquad
    \text{for every }(\theta,y)\in\widetilde\Theta.
\end{equation}
Take the full class $\Theta$ as the comparison set and admit only AIR
updates that stay positive on every attainable index. By
\eqref{eq:finite-air-fixed-truth-bracket}, the saddle control makes the
Bellman bracket nonpositive for $V_t=\overline W_t=(T-t+1)c_\gamma$.
Thus $\overline W_t$ is a supersolution and backward induction gives
$W_t^\gamma\le\overline W_t$. Theorem~\ref{thm:frequentist-info-risk-upper}
applies to an approximate selector for the recursion, while
Theorem~\ref{thm:mair-upper} applies directly to the saddle policy. In
either case, the uniform initial law on $\calY_\Theta$ gives the bound
below, up to arbitrarily small optimization error for approximate selectors:
\[
    \sup_{\theta\in\Theta}\E_\theta\Reg_T^{\mathcal R}
    \le
    Tc_\gamma+\gamma\log K_0
    \le
    \frac{(\lambda+L^2)KT}{2\gamma}+\gamma\log K.
\]
Standard measurable-selection results apply to the compact saddle problem; equivalently, measurable approximate saddles may be used and their arbitrarily small optimization errors added to the display.

Finally, for a policy confined to $\mathcal R$,
\[
    \sup_{x\in\calX_{\rm all}}f(x)-f(X_t)
    \le
    \varepsilon_{\mathcal R}
    +\max_{x\in\mathcal R}f(x)-f(X_t).
\]
Summing proves \eqref{eq:finite-marginal-air-upper}.  Optimizing in $\gamma$ proves \eqref{eq:finite-marginal-air-optimized}.
\end{proof}

\subsection{Fixed-kernel separations}

\subsubsection{Proof of Theorem~\ref{thm:fixed-kernel-gpucb}}
\label{app:proof-thm:fixed-kernel-gpucb}
\begin{proof}
Choose increasing integers $T_m$ recursively, starting with $T_1$ sufficiently large.  Put
\[
    a_m=T_m^{-\alpha},\qquad
    N_m=\lceil4T_m\rceil,\qquad
    M_{m-1}=\sum_{\ell<m}N_\ell,
\]
and choose $T_m$ sufficiently large that
\begin{equation}
\label{eq:fixed-kernel-scale-separation}
    (M_{m-1}+2)\log(1+T_m)
    \le a_m^2T_m=T_m^{1-2\alpha}.
\end{equation}
This is possible because $1-2\alpha>0$.  Let
\[
    \calX=\{x_0,x_h\}\cup
       \bigcup_{m\ge1}B_m,
    \qquad
    B_m=\{c_{m,1},\ldots,c_{m,N_m}\},
\]
and define a diagonal kernel by
\[
    k(x_h,x_h)=1,\qquad
    k(c_{m,j},c_{m,j})=a_m^2,
\]
with every other entry zero.  Let $D=\bigcup_{m\ge1}B_m$ have the discrete topology and endow $D\cup\{x_0\}$ with its one-point compactification: a neighborhood base at $x_0$ is
\[
    U_F=\{x_0\}\cup(D\setminus F),
    \qquad F\subset D\ \text{finite}.
\]
Adjoin $x_h$ as an isolated point.  The resulting $\calX$ is compact, metrizable, and countable, with $x_0$ as its only accumulation point.  Since $T_m$ is strictly increasing, $a_m\downarrow0$.  Every point other than $x_0$ is isolated.  If $z_r\to x_0$, then the block index of $z_r$ tends to infinity and hence $k(z_r,z_r)\to0$.  All off-diagonal values are zero, so the same observation verifies continuity at $(x_0,x_0)$ and at every $(x_0,x)$ and $(x,x_0)$.  Thus $k$ is jointly continuous on $\calX\times\calX$ and bounded by one.  Its unit RKHS ball is
\begin{equation}
\label{eq:fixed-kernel-rkhs-ball}
    \calF=\left\{f:f(x_0)=0,\quad
       f(x_h)^2+
       \sum_{m\ge1}a_m^{-2}\sum_{j=1}^{N_m}f(c_{m,j})^2
       \le1\right\}.
\end{equation}

For the linear-bandit interpretation, take
\[
    \mathcal H
    =\ell_2\bigl(\{h\}\cup\{(m,j):m\ge1,\ j\in[N_m]\}\bigr)
\]
with its canonical orthonormal basis and define
\[
    \varphi(x_0)=0,
    \qquad
    \varphi(x_h)=e_h,
    \qquad
    \varphi(c_{m,j})=a_m e_{m,j}.
\]
Then $k(x,x')=\langle\varphi(x),\varphi(x')\rangle_{\mathcal H}$, and
\eqref{eq:fixed-kernel-rkhs-ball} is the unit ball of parameters
$\theta\in\mathcal H$ under the identification
$f(x)=\langle\theta,\varphi(x)\rangle_{\mathcal H}$. This makes the glued
theorem one fixed infinite-dimensional Hilbert-space linear-bandit instance.

Fix the horizon $T_m$ and put
\[
    \mathcal R_m
    :=
    \{x_0,x_h\}\cup\bigcup_{\ell<m}B_\ell,
    \qquad
    K_m:=|\mathcal R_m|=M_{m-1}+2.
\]
The one-hot alternatives supported on $B_m$ and the coupling argument from Theorem~\ref{thm:small-cloud-gpucb} give the lower bound $3a_mT_m/4$.  For the upper bound, run a minimax finite-armed sub-Gaussian policy on $\mathcal R_m$.  Its regret relative to the best retained action is at most $C\sqrt{K_mT_m}$.  Every point in the ignored blocks $B_\ell$, $\ell\ge m$, has $|f(x)|\le a_\ell\le a_m$, while the retained anchor has value zero.  By \eqref{eq:fixed-kernel-scale-separation}, the full regret is therefore at most
\[
    C\sqrt{K_mT_m}+a_mT_m
    \le C'a_mT_m.
\]
This proves \eqref{eq:fixed-kernel-minimax}.

To define one AIR Bellman policy for all horizons, set $T_0=0$ and
$L_m=T_m-T_{m-1}$. At the start of epoch $(T_{m-1},T_m]$, discard earlier
data and initialize the optimal-action reference uniformly. On
$\mathcal R_m$, run the $L_m$-round action-index AIR Bellman selector from
Theorem~\ref{thm:finite-marginal-air-upper} with
\[
    \gamma_m
    =
    \sqrt{\frac{K_mL_m}{\log K_m}}.
\]
The AIR/AMS saddle of \citet{xu2025bayesian} supplies the control on each
marginal. The epoch endpoints and the rule for each epoch are set in
advance; the policy does not take the evaluation horizon as input. The
projection of the unit RKHS ball onto $\mathcal R_m$ is a compact ellipsoid
in $[-1,1]^{K_m}$, and retaining $\mathcal R_m$ costs at most $a_m$ per round.
Thus the regret in epoch $m$ is at most
\[
    a_mL_m+2\sqrt{K_mL_m\log K_m}
    \le 3a_mT_m,
\]
where the last inequality uses
\[
    K_m\log K_m
    \le K_m\log(1+T_m)
    \le a_m^2T_m.
\]
Here the first inequality follows because the scale condition itself implies $K_m\le T_m$ for all sufficiently large $m$.
The RKHS envelope bounds per-round regret by two.  For $m\ge2$, moreover, $K_m\ge N_{m-1}\ge4T_{m-1}$, so the regret before the current epoch is at most
\[
    2T_{m-1}
    \le\frac{K_m}{2}
    \le\frac{a_m^2T_m}{2\log(1+T_m)}
    \le a_mT_m.
\]
Consequently, the single epochwise AIR Bellman policy has worst-case regret at most $4a_mT_m=O(T_m^{1-\alpha})$ at time $T_m$.  This proves \eqref{eq:fixed-kernel-air}.

Let $\overline\Gamma_s$ denote the maximal information gain of this fixed kernel.  For $s\le T_m$, selecting distinct points of $B_m$ and using $\log(1+u)\ge u/2$ gives
\begin{equation}
\label{eq:fixed-kernel-gamma-bounds}
    \frac{s a_m^2}{4}
    \le \overline\Gamma_s
    \le
    \frac{M_{m-1}+1}{2}\log(1+s)
      +\frac{a_m^2s}{2}
    \le a_m^2T_m.
\end{equation}
Indeed, the hub and each coordinate in the earlier blocks contribute at most $\frac12\log(1+s)$, while all coordinates in $B_\ell$, $\ell\ge m$, contribute in total at most $a_m^2s/2$.  Use the anytime multiplier
\[
    \overline\beta_t
    =1+\sqrt{2\{\overline\Gamma_{t-1}+2\log(t\vee2)\}}.
\]
At round $t$, the fixed-kernel GP--UCB rule selects a maximizer of $m_t(x)+\overline\beta_ts_t(x)$.  This is one policy, independent of the evaluation horizon.  After every finite history, $m_t(x)$ and $s_t^2(x)$ are finite algebraic combinations of the continuous kernel sections $k(x,X_r)$ and $k(x,x)$.  Hence $m_t+\overline\beta_ts_t$ is continuous on compact $\calX$ and attains its maximum.
By \eqref{eq:fixed-kernel-gamma-bounds}, an unqueried point of $B_m$ scores
at least $a_m^2\sqrt{(t-1)/2}$ throughout the horizon. This lower bound
diverges uniformly on $T_m/2\le t\le T_m$. The upper bound in
\eqref{eq:fixed-kernel-gamma-bounds} and the scale condition
\eqref{eq:fixed-kernel-scale-separation} give
$\overline\beta_t\le Ca_m\sqrt{T_m}$ for $t\le T_m$.

Finally take the fixed hub truth $f^\dagger(x_h)=1/4$ and zero elsewhere.  Repeat the Gaussian-noise-stack event and last-half argument in the proof of Theorem~\ref{thm:small-cloud-gpucb}.  For every $t\le T_m$, an unqueried point of $B_m$ remains available because $N_m\ge4T_m$.  Its score is $a_m\overline\beta_t>0$, whereas the anchor $x_0$ has score zero; consequently, $x_0$ is never selected.  Hence, if fewer than $T_m/4$ nonanchor, nonhub pulls occur, the hub has been sampled at least $T_m/4-1$ times before every round in the last half.  On an event of probability at least $1-T_m^{-1}$ its acquisition score is at most $1/4+O(a_m)+O(\sqrt{\log T_m/T_m})<1/2$.  An unqueried point of $B_m$ has score larger than one for all sufficiently large $m$.  Thus every round in the last half must select a cloud action, a contradiction.  At least $T_m/4$ pulls therefore incur regret $1/4$, proving \eqref{eq:fixed-kernel-gpucb} and the ratio claim for \eqref{eq:fixed-kernel-anytime-rule}.

For the original schedule \eqref{eq:original-rkhs-gpucb-rule}, the same argument applies with only two scale estimates changed. The distinct-set information gain obeys the same bounds needed here:
\[
    \frac{s a_m^2}{4}
    \le \Gamma_s^{\rm SKKS}(k)
    \le
    \frac{M_{m-1}+1}{2}\log 2+\frac{a_m^2s}{2}
    \le a_m^2T_m,
    \qquad s\le T_m.
\]
The lower bound selects \(s\) distinct points of \(B_m\); the upper bound follows because a set uses every earlier coordinate at most once and all later diagonal values are at most \(a_m^2\). Hence, uniformly for \(T_m/2\le t\le T_m\),
\[
    a_m b_t^{\rm SKKS}
    \ge
    c\,a_m^2\sqrt{T_m}\,\log^{3/2}T_m
    \longrightarrow\infty,
\]
because \(\alpha<1/4\). Uniformly for \(t\le T_m\), it also gives
\[
    b_t^{\rm SKKS}
    \le
    C\{1+a_m\sqrt{T_m}\log^{3/2}T_m\}.
\]
After at least \(T_m/4-1\) hub pulls, the corresponding hub bonus is
\[
    O\!\left(T_m^{-1/2}+a_m\log^{3/2}T_m\right)=o(1),
\]
while an unqueried point of \(B_m\) has score larger than one. The same noise-stack event and contradiction force at least \(T_m/4\) cloud pulls. This proves \eqref{eq:fixed-kernel-gpucb} and the ratio claim for \eqref{eq:original-rkhs-gpucb-rule} as well.
\end{proof}

\subsection{Finite-armed and linear bandits}

\subsubsection{Proof of Lemma~\ref{lem:mab-ghost}}
\label{app:proof-lem:mab-ghost}
\begin{proof}
For a fixed reference history, let \(N_y'\) count pulls of arm \(y\).
The event \(\bar L_Y(H_T')\le\Delta/2\) requires \(N_Y'\ge T/2\).
Since \(\sum_y N_y'=T\), at most two indices satisfy this condition.
The conditional probability is therefore at most \(2/K\).
\end{proof}

\subsubsection{Proof of Lemma~\ref{lem:mab-info}}
\label{app:proof-lem:mab-info}
\begin{proof}
Let \(\Pp_j\) and \(\Pp_0\) be the history laws under the spike \(\omega^j\)
and the zero model, respectively, and put
\(\bar\Pp=K^{-1}\sum_j\Pp_j\).
Since \(\bar\Pp\ge K^{-1}\Pp_j\), we can use
\(\KL(\Pp_j\|\bar\Pp)\le(2+\log K)H^2(\Pp_j,\bar\Pp)\), where
\(H^2(P,Q)=\int(\sqrt{dP}-\sqrt{dQ})^2\).
Indeed, Taylor's inequality for \(g(x)=x^2\log x^2-x^2+1\) gives
\(g(x)\le(2+\log K)(x-1)^2\) for \(0\le x\le\sqrt K\); integrate with
\(x=\sqrt{d\Pp_j/d\bar\Pp}\).
The triangle inequality for \(H\), convexity of \(H^2\), and the Gaussian
KL chain rule then give
\begin{align*}
    I(Y;H_T)
    &\le \frac{4(2+\log K)}{K}\sum_{j=1}^K H^2(\Pp_j,\Pp_0)\\
    &\le \frac{4(2+\log K)}{K}\sum_{j=1}^K \KL(\Pp_0\|\Pp_j)
    = 2(2+\log K)\frac{\Delta^2T}{K}.
\end{align*}
The equality uses \(\KL(\Pp_0\|\Pp_j)
=(\Delta^2/2)\E_0\sum_t\1\{A_t=j\}\).
Since \(\log(K/2)/(2+\log K)\) is bounded away from zero for \(K\ge3\),
taking \(c_0\) sufficiently small gives the claim.
\end{proof}

\subsubsection{Proof of Theorem~\ref{thm:mab}}
\label{app:proof-thm:mab}
\begin{proof}
Take \(\Delta=c_1\sqrt{K/T}\), \(r=\Delta/2\), and write \(p_r\) for the
algorithm's ghost success probability under \(\mu_\Delta\).
Lemma~\ref{lem:mab-ghost} gives \(p_r\le2/K\), while the calculation in
Lemma~\ref{lem:mab-info} gives \(I(Y;H_T)\le2(2+\log K)c_1^2\).
The ratio \(\kl(3/4,2/K)/(2+\log K)\) is positive for every \(K\ge3\)
and tends to \(3/4\) as \(K\to\infty\).
Thus a universal \(c_1\in(0,1]\) can be chosen so that
\(I(Y;H_T)\le\kl(3/4,2/K)\). Since \(p_r\le2/K<3/4\),
\eqref{eq:quantile-exact} bounds the true success probability by \(3/4\).
The resulting Bayes regret is at least
\(Tr/4=(c_1/8)\sqrt{KT}\) for every algorithm.
\end{proof}

\subsubsection{Proof of Proposition~\ref{prop:mab-upper-match}}
\label{app:proof-prop:mab-upper-match}
\begin{proof}
The standard distribution-free MOSS bound
\citep{audibert2009minimax,lattimore2020bandit} and Theorem~\ref{thm:mab}
give the regret comparison. Under the uniform index law the code length is
\(\log K\), and Lemma~\ref{lem:mab-ghost} gives
\(L_0/E_{\mu_\Delta,r}^{\star}\le\log K/\log(K/2)\), bounded for \(K\ge3\).
\end{proof}

\subsubsection{Proof of Lemma~\ref{lem:linear-ghost}}
\label{app:proof-lem:linear-ghost}
\begin{proof}
Set \(\bar a=T^{-1}\sum_t a_t'\), so \(\|\bar a\|_2\le1\).
The low-regret event requires \(\langle V,\bar a\rangle\ge3\sqrt d/4\).
This centered Rademacher sum has variance proxy \(\|\bar a\|_2^2\);
Hoeffding's inequality gives the bound \(e^{-9d/32}\).
For a random reference, condition on the action sequence, which is
independent of \(V\) because \(Y=V/\sqrt d\) is one-to-one.
\end{proof}

\subsubsection{Proof of Lemma~\ref{lem:linear-info}}
\label{app:proof-lem:linear-info}
\begin{proof}
Since \(Y\) and \(V\) determine one another, it suffices to bound
\(I(V;H_T)\). For an independent algorithm seed \(U\),
\(I(V;H_T)\le I(V;H_T\mid U)\), so we may fix the seed and then average.

Let \(\mathbb P_v\) be the trajectory law at \(\Omega=\Delta v\), and let
\(v^{(i)}\) be \(v\) with bit \(i\) flipped. Given a prefix \(V_{<i}=u\),
the conditional mutual information is the Jensen--Shannon divergence of
the trajectory mixtures for \(V_i=+1\) and \(V_i=-1\).
It is bounded by one quarter of the sum of the two directed KL divergences
between these mixtures. Pairing the same suffix in both mixtures and using
joint convexity gives
\[
    I(V_i;H_T\mid V_{<i}=u)
    \le \tfrac12\E\!\left[
       \KL(\mathbb P_V\|\mathbb P_{V^{(i)}})\mid V_{<i}=u
    \right].
\]
Only the observation kernel contributes to the KL chain rule, and flipping
bit \(i\) changes the conditional Gaussian mean by \(2\Delta V_i A_{t,i}\).
Thus \(\KL(\mathbb P_v\|\mathbb P_{v^{(i)}})
=2\Delta^2\E_v\sum_t A_{t,i}^2\).
Averaging over prefixes and using the mutual-information chain rule yields
\[
    I(V;H_T)
    \le \Delta^2\E\sum_{t=1}^T\sum_{i=1}^d A_{t,i}^2
    \le \Delta^2 T.
\]
\end{proof}

\subsubsection{Proof of Theorem~\ref{thm:linear}}
\label{app:proof-thm:linear}
\begin{proof}
For large enough \(d\), set \(\Delta=c_1\sqrt{d/T}\) and
\(r=\Delta\sqrt d/4\). The prior is supported on the unit ball since
\(T\ge d^2\) and \(c_1\le1\). By Lemmas~\ref{lem:linear-ghost}
and~\ref{lem:linear-info}, \(p_r\le e^{-c_{\rm gh}d}\) and
\(I(Y;H_T)\le c_1^2d\). Choose \(c_1^2\le c_{\rm gh}/4\) and a universal
\(d_0\) such that \(\log 2\le c_{\rm gh}d/4\) for \(d\ge d_0\).
Then \(I(Y;H_T)+\log2\le\tfrac12\log(1/p_r)\), so
Theorem~\ref{thm:quantile-index} gives Bayes regret at least
\(Tr/2=(c_1/8)d\sqrt T\).

For \(d<d_0\), use the uniform prior on \(\{\pm\delta e_1\}\) with
\(\delta=(2\sqrt T)^{-1}\). The Gaussian KL chain rule bounds the
divergence between the two trajectory laws by \(2\delta^2T\), so their
total variation distance is at most \(\delta\sqrt T=1/2\) by Pinsker's
inequality. The same bound holds for each action marginal; for a
randomized rule it also follows by fixing the seed and averaging. Since
\(|A_{t,1}|\le1\), the Bayes regret at round \(t\) is at least
\[
    \delta\left(1-\tfrac12(\E_+ A_{t,1}-\E_- A_{t,1})\right)
    \ge \delta/2 .
\]
The total is at least \(\sqrt T/4\ge d\sqrt T/(4d_0)\).
Both Bayes lower bounds apply to every algorithm.
\end{proof}

\subsubsection{Proof of Proposition~\ref{prop:linear-upper-match}}
\label{app:proof-prop:linear-upper-match}
\begin{proof}
The standard self-normalized analysis gives the upper bound
\citep{abbasi2011improved,lattimore2020bandit}.
The index \(Y=V/\sqrt d\) is uniform on \(2^d\) values, and
Lemma~\ref{lem:linear-ghost} gives \(E_r^\star\ge c_{\rm gh}d\), as claimed.
\end{proof}
\newpage
\section*{Statement on the use of generative AI \\and mathematical provenance}
\addcontentsline{toc}{section}{Statement on the use of generative AI and mathematical provenance}

The author used OpenAI's ChatGPT and Codex throughout the development and
revision of this manuscript, including for mathematical work. The systems
suggested candidate definitions, derivations, intermediate lemmas, and proof
repairs. They helped test theorem statements and parameter choices, identify
missing hypotheses and inconsistent notation, edit the exposition, and
write and check simulation code. The author checked suggested references
against the cited sources. The research questions, central constructions,
final theorem statements and scope, and decisions about which claims to
retain are the author's. The author reviewed and revised the arguments,
citations, code, and numerical claims and takes responsibility for the
paper. ChatGPT and Codex also helped draft and revise this disclosure.

\paragraph{Formal model and Bellman-sufficient state.}
The history-based controlled-process model, nonanticipating decision kernels,
posterior states, and Bellman recursion are standard ideas in sequential
decision making and reinforcement learning
\citep{barto1989learning,littman1996algorithms,sutton2018reinforcement,lattimore2020bandit}.
The stepwise formulation here is compatible with the interactive statistical
decision making framework of \citet{chen2024unified}, and the DMSO models of
\citet{foster2021statistical,foster2023tight} are special cases. Low Bellman
rank \citep{jiang2017bellman} is a related way to study representations, but
does not define the state used here. The author proposed distinguishing the
fixed truth, retained Bellman state, information index, and posterior or
algorithmic reference law. The author also required that one representation
support both the controlled recursion and the upper and lower information
calculations.
ChatGPT and Codex assisted in turning that viewpoint into the four closure
clauses in Definition~\ref{def:index-bellman} and in separating fixed-truth
Bellman sufficiency from the stronger posterior-indexed lift needed for the
chain rule and Bellman--Fano argument. They also helped audit the treatment of
the initial observation packet, random initial states, announced mixed actions,
conditional index losses, standard-Borel and measurable-selection assumptions,
and the equality condition in the data-processing statement. These
definitions and qualifications are specific to this paper.

\paragraph{Technical information-complexity sandwich.}
The sandwich combines several lines of prior work. Matching an information
upper construction with a Fano or local-entropy lower construction comes from
classical information-theoretic minimax theory
\citep{yang1999information,zhang2006information,yu1997assouad}. The
action-index AIR objective and robust AIR/AMS control build on
\citet{xu2025bayesian}; the information-ratio and EBO viewpoints are due to
\citet{russo2014learning,lattimore2021mirror}. The comparison between the
true adaptive experiment and an algorithm-dependent ghost experiment follows
the interactive Fano method of \citet{chen2024unified}. Using growth conditions and fixed points to relate localized complexity to statistical rates is classical in noninteractive learning. Examples include sub-root bounds for local Rademacher averages \citep[Section~3]{bartlett2005local}, critical radii in the small-ball method \citep{mendelson2014learning}, and local covering fixed points that yield matching upper and lower bounds in regression \citep{mendelson2017local}. The calibration and growth conditions were developed independently for the Bellman comparison in this paper. Interpreting them in terms of coefficients, fixed points, and localization highlights parallels with DEC and constrained DEC \citep{foster2021statistical,foster2023tight}.

The formulation in Theorem~\ref{thm:interactive-info-risk-sandwich} was
developed for this paper with substantial assistance from ChatGPT and Codex.
They helped distinguish the hypotheses that establish the upper and lower bounds from two additional conditions used to compare them: calibration at the lower entropy scale (Definition~\ref{def:upper-at-lower-calibration}) and a growth bound on the upper value (Assumption~\ref{ass:upper-growth}). Together, these conditions give an elementary mechanism to control the gap without directly assuming that the regret bounds match at a common radius. They also helped
construct the counterexample for a generic monotone budget proxy and check
the definitions and proofs. The Bellman recursions for information capacity
and ghost mass, the use of calibration and growth to compare the bounds,
and the common-marginal specialization for a finite index set are this
paper's contributions in this part. Appendix~\ref{app:sandwich-details}
gives the technical details. 

\paragraph{GP--UCB separation.}
The GP--UCB acquisition rule and original exploration schedule are from
\citet{srinivas2010gaussian}, and the modern anytime calibration is related to
\citet{chowdhury2017kernel}. The starting weighted-basis hub--cloud geometry
comes from the author's earlier work \citep{xu2026separation}. The author
supplied the central new mechanism in this paper: reversing the amplitude
ordering so that the cloud can be ignored at the sublinear minimax cost
\(a_mT_m\), while its multiplicity still controls the global optimism
multiplier. The author also proposed the comparison with the unrestricted
minimax value, the fixed-kernel and fixed-truth formulation, the epochwise
localized comparator, and the finite-dimensional LinUCB consequence.
ChatGPT and Codex helped formulate and check several parts of the proof:
the finite-block information-gain bounds; the one-hot coupling and
truncation arguments for the minimax rate; the gluing argument, including
compactness, continuity, and the RKHS norm; and the Gaussian noise-stack
argument for each GP--UCB schedule. They also helped derive the pathwise
reduction to a linear bandit, formulate the LinUCB corollary, and write and
audit the numerical implementation. They ran the reported reproducibility
checks and verified the reported values and figure labels. The author
reviewed the protocol, code, and outputs.

The finite-marginal comparator uses the AIR/AMS control of
\citet{xu2025bayesian}; its connection to \citet{lattimore2021mirror} is
through the EBO optimization principle and associated duality, not through
that paper's functional-estimator implementation. The Mat\'ern result of
\citet{wang2026suboptimality} is used only in the separate comparison in
Appendix~\ref{app:matern-scope}, not in the fixed-kernel proof. The author
checked and revised the proof and simulation claims and is responsible for
the theorem, its priority and scope claims, and any remaining errors.

\clearpage
\bibliographystyle{plainnat}
\bibliography{references}

\end{document}